\documentclass{article}

\usepackage{microtype}
\usepackage{graphicx}
\usepackage{subfigure}
\usepackage{verbatim}
\usepackage{booktabs} 
\usepackage{bbm}
\usepackage{bbold}
\usepackage{dsfont}
\usepackage{hyperref}

\usepackage{amsthm,mathtools}
\usepackage{amssymb}
\usepackage{bm}

\usepackage{makecell}
\usepackage{multirow}

\theoremstyle{definition}
\newtheorem{definition}{Definition}[section]
\newtheorem{theorem}{Theorem}[section]
\theoremstyle{remark}

\theoremstyle{lemma}
\newtheorem{lemma}[theorem]{Lemma}

\usepackage{mdwlist}
\newcommand{\pra}{\mathcal{P}^\mathcal{R}_{\mathcal{A},\mathcal{H}}}

\newcommand{\pha}{\mathcal{P}_{\mathcal{A},\mathcal{H}}}
\newcommand{\oh}[1]{\delta^{#1}}

\DeclarePairedDelimiter{\normone}{\lVert}{\rVert^1_1}
\DeclarePairedDelimiter{\normtwo}{\lVert}{\rVert^2_2}
\DeclarePairedDelimiter{\norminf}{\lVert}{\rVert_\infty}
\DeclareMathOperator*{\argmax}{\arg\max}
\DeclareMathOperator*{\argmin}{\arg\min}

\usepackage[accepted]{icml2020}

\icmltitlerunning{Minimax Pareto Fairness: A Multi Objective Perspective}

\begin{document}

\twocolumn[
\icmltitle{Minimax Pareto Fairness: A Multi Objective Perspective}



\icmlsetsymbol{equal}{*}

\begin{icmlauthorlist}
\icmlauthor{Natalia Martinez}{equal,to}
\icmlauthor{Martin Bertran}{equal,to}
\icmlauthor{Guillermo Sapiro}{to}
\end{icmlauthorlist}

\icmlaffiliation{to}{Department of Electircal and Computer Engineering, Duke University}

\icmlcorrespondingauthor{Natalia Martinez}{natalia.martinez@duke.edu}

\icmlkeywords{Machine Learning, ICML}

\vskip 0.3in
]



\printAffiliationsAndNotice{\icmlEqualContribution} 


\begin{abstract}
In this work we formulate and formally characterize group fairness as a multi-objective optimization problem, where each sensitive group risk is a separate objective. We propose a fairness criterion where a classifier achieves minimax risk and is Pareto-efficient w.r.t. all groups, avoiding unnecessary harm, and can lead to the best zero-gap model if policy dictates so. We provide a simple optimization algorithm compatible with deep neural networks to satisfy these constraints. Since our method does not require test-time access to sensitive attributes, it can be applied to reduce worst-case classification errors between outcomes in unbalanced classification problems. We test the proposed methodology on real case-studies of predicting income, ICU patient mortality, skin lesions classification, and assessing credit risk, demonstrating how our framework compares favorably to other approaches.
\end{abstract}

\section{Introduction}
\label{sec:Introduction}

Machine learning algorithms play an important role in decision making in society. When these are used to make high-impact decisions such as hiring, credit-lending, predicting mortality for intensive care unit patients, or classifying skin lesions, it is paramount to guarantee that the prediction is both accurate and unbiased with respect to sensitive attributes such as gender or ethnicity. A model that is trained naively may not have these properties by default; see, for example \cite{barocas2016big}.

In these critical applications, it is desirable to impose some fairness criteria. Some well-known definitions of group fairness in the machine learning literature attempt to make algorithms whose predictions are independent of the sensitive populations (e.g., Demographic Parity, \cite{louizos2015variational,zemel2013learning,feldman2015certifying}); or algorithms whose outputs are independent of the sensitive attribute given the objective's ground truth (e.g., Equality of Odds, Equality of Opportunity, \cite{hardt2016equality, woodworth2017learning}). Notions of Individual Fairness have also been proposed \cite{dwork2012fairness,joseph2016rawlsian, zemel2013learning}. These can be appropriate in many scenarios, but in domains where quality of service is paramount, such as healthcare, we argue that it is necessary to strive for models that are as close to fair as possible without introducing unnecessary harm \cite{ustun2019fairness}. Additionally, a model satisfying these characteristics can be post-processed to introduce a controlled performance degradation that results in a perfectly fair, albeit harmful classifier. This is a decision beyond algorithmic design and is left to the policymaker, but machine-learning should inform fairness policy and provide the necessary tools to implement it.

Here we focus on group fairness in terms of predictive risk disparities, a metric that has been explored in recent works such as \cite{calders2010three,dwork2012fairness,feldman2015certifying,chen2018my,ustun2019fairness}. We formulate fairness as a Multi-Objective Optimization Problem and use Pareto optimality \cite{mas1995microeconomic} to define the set of all efficient classifiers, meaning that the increase in predictive risk on one group is due to a decrease in the risk of another (no unnecessary harm). We consider problems where target labels available for training are trustworthy (not affected by discrimination), and tackle fairness as a minimax problem where the goal is to find the classifier with the smallest maximum group risk among all efficient models. As a design choice, this implies that a system's risk is as good as its worst group performance, but we do not enforce zero risk disparity if the disadvantaged groups do not benefit directly. When perfect fairness is achievable, this reduces to finding the efficient classifier in our hypothesis class that has the same risks among all groups. Our approach differs from post-hoc correction methods like the ones proposed in \cite{hardt2016equality,woodworth2017learning}, where zero-disparity is enforced by design, and test-time access to sensitive attributes is needed. Since our proposed methodology does not require the latter, and is not restricted to binary sensitive and target variables, it can also be used to reduce worst-case classification error between outcomes in imbalanced classification scenarions.
\clearpage

\paragraph{Main Contributions.}  We formulate group fairness as a Multi-Objective Optimization Problem (MOOP), where each objective function is the sensitive group conditional risk of the model. We formalize \textit{no unnecessary harm} fairness using Pareto optimality \cite{mas1995microeconomic}; and characterize the space of Pareto-efficient classifiers for convex models and risk functions, which include deep neural networks (DNNs) and standard classifier losses. We show that all efficient classifiers under these conditions can be recovered with a simple modification of the overall risk function. We introduce and discuss minimax Pareto fairness (MMPF), where we select the efficient classifier with the smallest worst group conditional risk, and provide a simple and efficient algorithm to recover this classifier using standard (Stochastic) Gradient Descent on top of an adaptive loss. Critical to numerous applications in fairness and privacy, the proposed methodology does not require test-time access to the sensitive attributes. We also show that if the policy mandate is to obtain a zero-gap classifier, we can add harmful post-hoc corrections to the MMPF model, which ensures the lowest risk levels across all groups under certain conditions. In addition to this, we demonstrate how our methodology performs on real tasks such as inferring income status in the Adult dataset \cite{Dua:2019}, predicting ICU mortality rates in the MIMIC-III dataset from hospital notes \cite{johnson2016mimic}, classifying skin lesions in the HAM10000 dataset \cite{tschandl2018ham10000}, and assessing credit risk on the German Credit dataset \cite{Dua:20192}. Finally, since our methodology does not require access test-time sensitive attributes, it can be used to reduce worst-case classification error between outcomes in unbalanced classification problems. Code is available at \href{https://github.com/natalialmg/MMPF}{github.com/natalialmg/MMPF}.

\section{Related Work}
\label{sec:RelatedWork}
There is a growing body of work on group fairness in machine learning. Following \cite{friedler2019comparative}, we empirically compare our methodology against the works of \cite{feldman2015certifying,kamishima2012fairness,zafar2015fairness}. Our method shares conceptual similarities with \cite{zafar2017parity, woodworth2017learning,agarwal2018reductions, oneto2019taking}, but differs on the fairness objective and how it is adapted to work with standard neural networks. Although optimality is often discussed in the fairness literature, it is usually in the context of error-unfairness tradeoffs \cite{kearns2019ethical, kearns2017preventing}, and not between sensitive groups as studied here. The conflict between perfect fairness and optimality has been previously studied in \cite{kaplow1999conflict}, we acknowledge this impossibility and formally characterize what is achievable in the context of machine learning and classification.

The work presented in \cite{hashimoto2018fairness} discusses decoupled classifiers (one per sensitive group) as a way of minimizing group-risk disparity, but simultaneously cautions against this methodology when presented with insufficiently large datasets. The works of \cite{chen2018my,ustun2019fairness} also empirically report the disadvantages of decoupled classifiers as a way to mitigate risk disparity. Here we argue for the use of a single classifier since it does not require access to sensitive group membership during test time, and might allow transfer learning between diverse groups when possible. If access to group membership during test time is available, this can be naturally incorporated as part of our observation features; with a sufficiently rich hypothesis set, this is equivalent to training separate classifiers, with the added benefit of positive transfer on samples were groups share optimal decision boundaries \cite{ustun2019fairness,wang2020split}.

The work of \cite{chen2018my} uses the unified bias-variance decomposition advanced in \cite{domingos2000unified} to identify that noise levels across different sub-populations may differ, making perfect fairness parity impossible without explicitly degrading performance on one group. Their methodology attempts to bridge the disparity gap by collecting additional samples from high-risk sub-populations. Here we modify the classifier loss to improve worst-case group performance without inducing unnecessary harm, which could be considered synergistic with their methodology.

\section{Minimax Pareto Fairness: Formulation and Basic Properties}
\label{sec:ProblemStatement}
Consider we have access to a dataset $\mathcal{D} = \{(x_i,y_i,a_i)\}_{i=1}^n$ containing $n$ independent triplet samples drawn from a joint distribution $(x_i,y_i,a_i) \sim P(X,Y,A)$, where $x_i \in \mathcal{X}$ are our input features (e.g., images, tabular data), $y_i \in \mathcal{Y}$ is our target variable, and $a_i \in \mathcal{A}$ indicates group membership or sensitive status (e.g., ethnicity, gender); our input features $X$ may or may not explicitly contain $A$, meaning sensitive attributes need not be available at deployment.

Let $h \in \mathcal{H}$ be a classifier from a hypothesis class $\mathcal{H}$ trained to infer $y$ from $x$, $h:\mathcal{X} \rightarrow [0,1]^{|\mathcal{Y}|}$. We use $\oh{Y} \in \{0,1\}^{|\mathcal{Y}|}: \oh{Y}_i = \mathds{1}(Y=y_i), i=1,...,|\mathcal{Y}|,$ to denote the one-hot representation of $Y$. Given a loss function $\ell: [0,1]^{|\mathcal{Y}|}\times [0,1]^{|\mathcal{Y}|} \rightarrow \mathbbm{R}^{+}$ the group-specific risk of classifier $h$ on group $a$ is $r_a(h) = E_{X,Y|A=a}[\ell(h(X),\delta^Y)]$. We approach fairness as a Multi-Objective Optimization Problem (MOOP), where the classifier $h$ is our decision variable, the group-specific risks $\{r_a(h)\}_{a=1}^{|\mathcal{A}|}$ are our objective functions, and they conform a risk vector $\bm{r}({h})=\{r_a(h)\}_{a=1}^{|\mathcal{A}|}$. The MOOP can be stated as
\begin{equation}
\min\limits_{h \in \mathcal{H}} \;(r_1(h),r_2(h), \dots, r_{|\mathcal{A}|}(h)).
    \label{eq:MOOP}
\end{equation}
We use dominance \cite{miettinen2008introduction} to define optimality for a MOOP, namely, Pareto optimality, the definitions are given below. We later formally characterize the space of optimal multi-objective classifiers $h$ for well-known losses, and argue a fairness criteria where a fair classifier is both Pareto optimal and has the smallest maximum group risk. Lemmas and theorems are stated without proof throughout the main text, proofs are provided in Section \ref{sec:AppendixProofs}.

\begin{definition}{Dominant vector:}
A vector $\bm{r}' \in \mathbbm{R}^k$ is said to dominate $\bm{r} \in \mathbbm{R}^k$, noted as $\bm{r}' \prec \bm{r}$, if $ r'_i \le r_i, \forall i = 1,\dots,k$ and $\exists j: r'_{j} < r_{j}$ (i.e., strict inequality on at least one component). Likewise, we denote $\bm{r}' \preceq \bm{r}$ if $\bm{r} \not\prec \bm{r'}$
\label{def:dominanceVectorMain}
\end{definition}
\begin{definition}{Dominant classifier:}
Given a set of group-specific risk functions $\bm{r}(h)$, a classifier $h'$ is said to dominate $h^{''}$, noted as $h' \prec h^{''}$, if $\bm{r}({h'})\prec \bm{r}({h^{''}})$. Similarily, we denote $h' \preceq h^{''}$ if $\bm{r}({h'})\preceq \bm{r}({h^{''}})$.
\label{def:dominance}
\end{definition}
\begin{definition}{Pareto front and Pareto optimality:} Given a family of classifiers $\mathcal{H}$, and a set of group-specific risk functions $\bm{r}(h)$, the set of Pareto front classifiers is $\pha = \{ h \in \mathcal{H}: \nexists h' \in \mathcal{H}| h' \prec h\}  =\{ h \in \mathcal{H}: h\preceq h'\, \forall h' \in \mathcal{H}\}$. The corresponding achievable risks are denoted as $\pra = \{\bm{r} \in \mathbbm{R}^{+|\mathcal{A}|}:  \exists h \in \pha, \bm{r}=\bm{r}({h}) \}$. A classifier $h$ is a Pareto optimal solution to the MOOP in Eq.(\ref{eq:MOOP}) iff $h \in \pha$.
\label{def:paretofront2}
\end{definition}

\vspace{-1em}
{\bf No unnecessary harm fairness.} The Pareto front defines the best achievable trade-offs between population risks $r_a(h)$. This is already suited for classification and regression tasks where the sensitive attributes are categorical. Constraining the classifier to be in the Pareto front disallows laziness, there exists no other classifier in the hypothesis class $\mathcal{H}$ that is at least as good on all group-specific risks and strictly better in one of them. In this sense, we say that a classifier in the Pareto front does \textit{no unnecessary harm}.

Literature on fairness has focused on putting constraints on the norm of discrimination gaps \cite{zafar2017parity,zafar2015fairness,creager2019flexibly, woodworth2017learning}. Here we focus on minimizing the risk on the worst performing group (Definition \ref{def:noharm}), these two criteria often yield similar results, and can be shown to be identical for Pareto optimal classifiers when $|\mathcal{A}|=2$. For more than $2$ sensitive groups, there may be situations where minimum risk discrepancy leads to higher minimax risk (e.g., a classifier that increases both the minimum and maximum risk but decreases their gap is still optimal if a third group sees their risk diminished). Constraining solutions to be Pareto optimal and minimizing the maximum risk preserves the overall idea of reducing risk disparities while avoiding some potentially undesirable tradeoffs. We formalize this next:

\begin{definition}{Minimax Pareto fair classifier and Minimax Pareto fair vector:}
A classifier $h^*$ is a Minimax Pareto fair classifier if it minimizes the worst group-specific risk among all Pareto front classifiers, $h^* \in \argmin\limits_{h \in \pha} \max\limits_{a\in \mathcal{A}}r_{a}(h) = \argmin\limits_{h \in \pha} \norminf{\bm{r}(h)}$, with corresponding Minimax Pareto-fair risk vector $\bm{r}^* =\bm{r}({h^*})$. 
\label{def:noharm}
\end{definition}
\vspace{-.8em}
An important consequence of this formulation is that when the hypothesis class $\mathcal{H}$ contains a classifier that is both Pareto optimal and has zero risk disparity, this classifier is also Minimax Pareto fair.

\begin{lemma} If $\exists h^* \in \pha:  r_a(h^*)=r_{a'}(h^*),\forall a,a' \in \mathcal{A}$ then $ \bm{r}(h^*) = \argmin_{\bm{r} \in \pra} \norminf{\bm{r}}$.
\label{lemma:perfectfairnessminmax}
\end{lemma}

Even when perfect equality of risk is desirable, Pareto classifiers still serve as useful intermediaries. To this end, Lemma \ref{lemma:eopareto}, shows that any classifier in $\mathcal{H}$ that attains equality of risk has worse performance on all groups than the Minimax Pareto fair classifier. Furthermore, we can post-process the Pareto fair classifier to be perfectly fair by increasing risk on over-performing groups, this procedure is still no worse than obtaining an equal risk classifier in our hypothesis class.

\begin{lemma} Let $h_{ER} \in \mathcal{H}$ be an equal risk classifier such that $r_a(h_{ER})=r_{a'}(h_{ER}) \forall\, a, a'$, and let $h^*$ be the Pareto fair classifier. Additionally, define the Pareto fair post-processed equal risk classifier $h^*_{ER} : r_{a}(h^{*}_{ER}) = ||\bm{r}(h^*)||_{\infty} \forall \,a \in \mathcal{A}$, then we have
\vspace{-.2em}
\begin{equation*}
\begin{array}{rl}
    r_a(h_{ER}) &\ge r_a(h_{ER}^*) \ge  r_a(h^*)\; \forall a \in \mathcal{A}.
\end{array}
\end{equation*}
\label{lemma:eopareto}
\end{lemma}
\vspace{-2em}
We provide a fair optimal classifier that improves the worst group risk. It also serves as an intermediate step to get perfect fairness; the decision between the two is left to the policymaker.
To exemplify these notions graphically, Figure \ref{fig:ParetoCurve2d} shows a scenario with binary sensitive attributes $a$ where none of the Pareto classifiers achieve equality of risk. Here the noise level differs between groups, and the Pareto fair risk $\bm{r}^*$ is not achieved by either a Naive classifier (minimizes expected global risk), or a classifier where groups are re-sampled to appear with equal probability (Balanced classifier). We observe how the discrimination gap along the Pareto front is closed by trading off performance from one group to another. The gap can be further closed by moving outside the Pareto front, but this discrimination reduction is a result of performance degradation on the privileged group, with no tangible upside to the underprivileged one.

\begin{figure}[ht]
 \centering
 \includegraphics[width=1\linewidth]{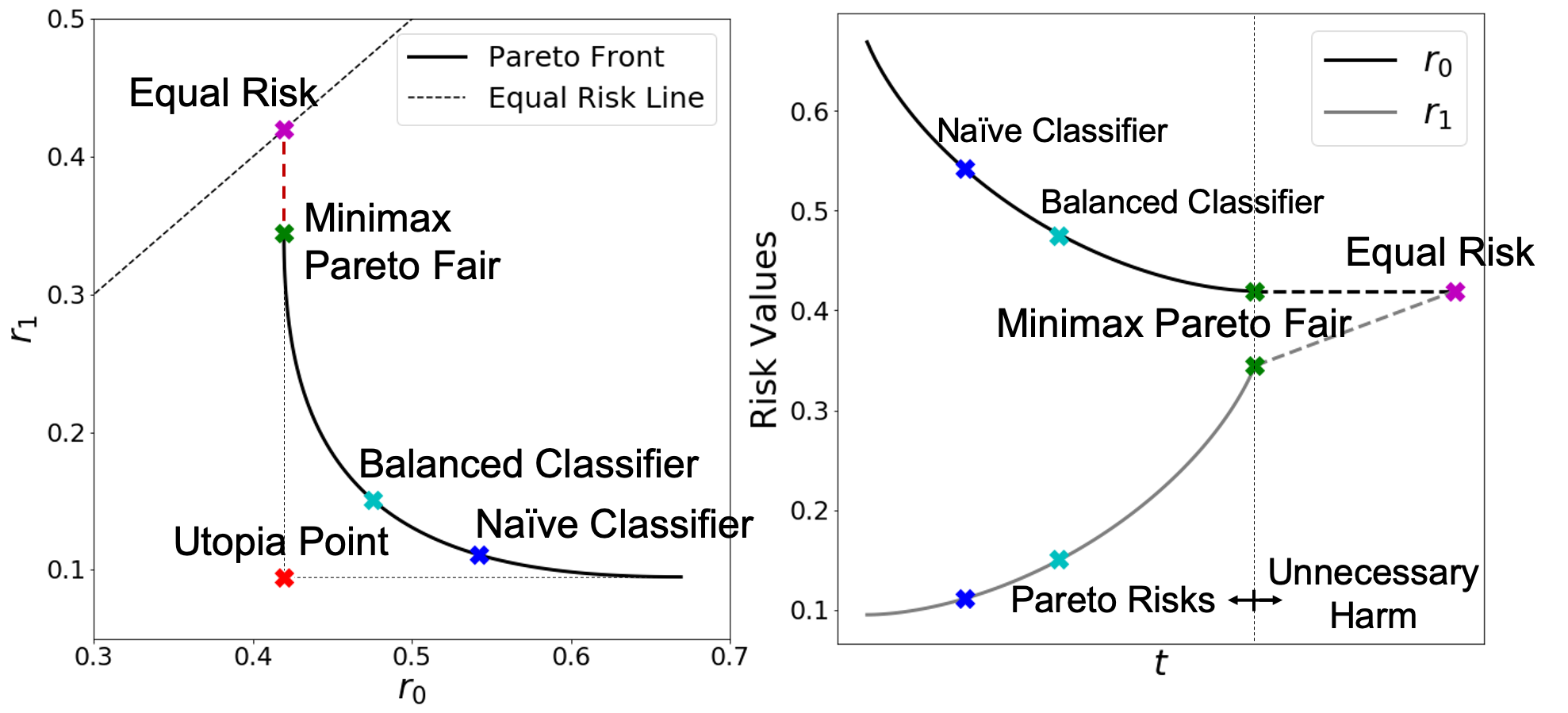}
 \vspace{-2em}
  \caption{Achievable risk trade-offs for a binary classification problem $Y\in\{0,1\}$ with two unbalanced groups $A\in\{0,1\}$, and covariates $X|A\sim N(\mu_A,1)$ (parameters provided in Supplementary Material). \it{Left:} Pareto front risks and the equal risk line; the minimax Pareto fair point (green) does not achieve equality of risk; the trade-offs attained by standard (Naive, blue) classifier and a class-rebalanced (Balanced, cyan) classifier are also shown. The Utopia point (red) corresponds to the minimum achievable risk for each group. \it{Right:} Parametrization of group risks along the Pareto front line, and on the minimax Pareto fair to Equal Risk line (Unnecessary harm). All points in the Pareto front efficiently trade-off performance between groups; the trajectory outside of the Pareto front, however, does not improve performance on the worst performing group $r_0$, it only degrades performance on $r_1$.}
 \label{fig:ParetoCurve2d}
\end{figure}

Section \ref{sec:OptimizationMethods} provides a method to recover the minimax Pareto fair classifier (MMPF) from training samples. Before that, Section \ref{sec:AnalysisPOSols} shows important properties of Pareto-efficient classifiers for convex hypothesis classes, including DNNs, and risk functions.

\section{Analysis of Pareto Optimal Solutions}
\label{sec:AnalysisPOSols}
In this section we characterize the Pareto front for convex hypothesis classes and convex risk functions. under these conditions the models have attractive regularity properties. 
\begin{definition}{Convex hypothesis class and risk function:}
A hypothesis class $ \mathcal{H} \subseteq \{ h: \mathcal{X} \rightarrow [0,1]^{|\mathcal{Y}|}\}$ is convex iff $\forall\, h',h'' \in \mathcal{H},\; \lambda \in [0,1], \rightarrow \lambda h' + (1-\lambda)h'' \in \mathcal{H}$. 

A risk function $r:\mathcal{H} \rightarrow \mathbbm{R}_+$ is convex iff $\forall \lambda \in [0,1] \; r(\lambda h' + (1-\lambda)h'')  \le \lambda r(h') + (1-\lambda)r(h'')$
\label{def:ConvexHypothesisClass}
\end{definition}
\begin{definition}{Convex Pareto front:}
A Pareto front $\pra \subseteq \mathbbm{R}^{|\mathcal{A}|}$ is convex if $\forall \bm{r}, \bm{r}' \in \pra, \lambda \in [0,1], \exists \bm{r}^\lambda \in \pra$ such that $\bm{r}^\lambda\preceq \lambda \bm{r} + (1-\lambda) \bm{r}'$
\label{def:ConvexParetoFront}
\end{definition}
Convexity of the hypothesis class for DNNs can be seen as a natural consequence of the Universal Function Approximation Theorem, shown for fully connected neural networks in \cite{hornik1989multilayer}, and more recently for convolutional NNs (CNNs) in \cite{zhou2020universality}. Note that this convexity is w.r.t. its function output space (i.e., for any two classifiers in the hypothesis class $h_1, h_2$, there exists a classifier in this same family $h_\lambda : h_\lambda(x) = (1-\lambda)h_1(x) + \lambda h_2(x) \forall x, \lambda \in [0,1]$), the parameter space itself may be highly non-convex. As for risk functions, many standard classification losses such as Brier Score ($r_a^{BS}(h) = E_{X,Y\mid a} [\normtwo{\oh{Y} - h(X)}]$) and Cross Entropy ($r_a^{CE}(h) = E_{X,Y\mid a} [\langle\oh{Y},  \ln (h(X))\rangle]$) are convex w.r.t. the classifier output. 

The following theorem (Theorem \ref{theo:linearweightconvex}) shows that under these conditions, the Pareto front can be fully characterized by solving the linear weighting problem: 
\begin{equation}
\begin{array}{l}
    \hat{h}= \argmin\limits_{h \in \mathcal{H}} \sum_{a=1}^{|\mathcal{A}|}\mu_a r_a(h); \\
    \normone{\bm{\mu}}=1, \hspace{0.05in}  \mu_a>0,\hspace{0.05in} a=1,...,|\mathcal{A}|.
\end{array}
\label{eq:linearweightproblem}
\end{equation}
We use the shorthand notation $h^{\bm{\mu}}$ to describe a classifier that solves Problem \ref{eq:linearweightproblem}; likewise, we denote $\bm{r}(\bm{\mu}) = \bm{r}_{h^{\bm{\mu}}}$. We utilize the results derived in \cite{geoffrion1968proper} to show that when both the hypothesis class $\mathcal{H}$ and the risk functions are convex, any optimal classifier $h \in \pha$ is a solution to Problem \ref{eq:linearweightproblem} for some choice of weights $\bm{\mu}$. Note that even when the risk functions and hypothesis class are non-convex, solutions to Problem \ref{eq:linearweightproblem} still belong to the Pareto front.

\begin{theorem}
Given $\mathcal{H}$ a convex hypothesis class and $\{r_a(h)\}_{a \in \mathcal{A}}$ convex risk functions then:
\begin{enumerate}
    \item The Pareto front is convex: $\forall \bm{r}, \bm{r}' \in \pra,\, \lambda \in [0,1], \; \exists \bm{r}'' \in \pra: \bm{r}'' \preceq \lambda \bm{r} + (1-\lambda) \bm{r}' $.
    \item Every Pareto solution is a solution to Problem \ref{eq:linearweightproblem}: $\forall \bm{\hat{r}} \in \pra, \exists \bm{\mu}: \bm{\hat{r}} = \bm{r}(\bm{\mu})$.
\end{enumerate}
\label{theo:linearweightconvex}
\end{theorem}
\vspace{-1em}
We can then characterize the Pareto optimal classifiers for Brier score (BS) and Cross-Entropy (CE) in the infinite samples and unbounded hypothesis class regime. The following provides an expression for the optimal classifiers and risks in terms of the probability densities and weights.

\begin{theorem} 
Given input features $X \in \mathcal{X}$, categorical target $Y \in \mathcal{Y}$ and sensitive group  $A \in \mathcal{A}$, with joint distribution $p(X,Y,A)$, and weights $\bm{\mu} = \{\mu_a\}_{a \in \mathcal{A}}$, the optimal predictor to the linear weighting problem $h(\bm{\mu})$ for both Brier score and Cross-Entropy is
\vspace{-.51em}
\begin{equation*}
 \begin{array}{cc}
h^{\bm{\mu}}(x)  = \frac{\sum_{a\in\mathcal{A}}\mu_a p(x|a) p(y|x,a)}{\sum_{a\in\mathcal{A}}\mu_a p(x|a)},
\end{array}
\end{equation*}
\noindent with corresponding risks
\begin{equation*}
 \begin{array}{ll}
 r^{BS}_a(\bm{\mu}) = E_{X,Y|a}[||\oh{Y}-p(y|X,a)||^2_2] +\\
        \hspace{0.6in} E_{X|a}\Big[||p(y|X,a) - h^{\bm{\mu}}(X)||^2_2\Big], \\
r^{CE}_a(\bm{\mu}) =H(Y|X,a)+\\ \hspace{0.6in} E_{X|a}\Big[D_{KL}\Big(p(y|X,a)\big|\big|h^{\bm{\mu}}(X)\Big)\Big], 
        
\end{array}
\end{equation*}
where $p(y|X,a)=\{p(Y=y_i|X,A=a)\}_{i=1}^{|\mathcal{Y}|}$ is the probability mass vector of $Y$ given $X$ and $A=a$. $H(Y|X,a)$ is the conditional entropy $H(Y|X,A=a) = E_{X|A=a}[H(Y|X=X,A=a)]$.
\label{theo:briercrosspareto}
\end{theorem}

The optimal risks for BS and CE are decomposed as the sum of two non-negative terms. The first term corresponds to the minimum achievable group risk, attained with the group-specific optimal classifier $p(y|X,a)$, this is independent of $h^{\bm{\mu}}$. The second term measures the discrepancy between $p(y|X,a)$ and the optimal predictor $h^{\bm{\mu}}$. Since both risk functions are convex, this is a full characterization of all asymptotically optimal multi-objective classifiers; this includes our proposed minimax Pareto fair classifier.

From the expressions in Theorem \ref{theo:briercrosspareto} we observe that if the separability condition $Y \perp A|X$ is satisfied, the minimum risk for each subgroup is attained. Here the Pareto front only contains the Utopia point (see Figure \ref{fig:ParetoCurve2d}). Additionally, if the entropy of the sensitive attribute given our features is small ($A$ is well predicted from $X$), the Pareto front also tends to the Utopia point. This is formalized in the following lemma.

\begin{lemma} In the conditions of Theorem \ref{theo:briercrosspareto} we observe that if $Y \perp A|X$ then
\begin{equation*}
 \begin{array}{ll}
 r^{BS}_a(\bm{\mu}) = E_{X,Y|a}[||\oh{Y}-p(y|X)||^2_2] \;\forall \bm{\mu}, \\
r^{CE}_a(\bm{\mu}) =H(Y|X) \; \forall \bm{\mu}. 
\end{array}
\end{equation*}
\vspace{-0.2em}
Likewise, if $H(A|X) \rightarrow 0$ then 
\vspace{-0.5em}
\begin{equation*}
 \begin{array}{ll}
 r^{BS}_a(\bm{\mu}) \rightarrow E_{X,Y|a}[||\oh{Y}-p(y|X,a)||^2_2] \;\forall \bm{\mu}, \\
r^{CE}_a(\bm{\mu})  \rightarrow H(Y|X,a) \; \forall \bm{\mu}. 
\end{array}
\end{equation*}
\label{lemma:limitingCases}
\end{lemma}
\vspace{-1.5em}
Note that even on these ideal cases ($Y\perp A\mid X$ or $H(A|X) \rightarrow 0$), the baseline risks between groups might differ. Therefore, perfect equality of risk may only be achieved by selecting sub-optimal classifiers (unnecessary harm), or by improving the input features $X$; this observation concurs with previous analysis on classifiers' bias-variance tradeoffs done in \cite{chen2018my, domingos2000unified}.

\section{Minimax Pareto Fair Optimization}
\label{sec:OptimizationMethods}
Our goal is to find the Pareto classifier $h^*$ that minimizes the risk of the worst performing sensitive groups (i.e., $h^* \in \argmin\limits_{h \in \pha} \norminf{\bm{r}(h)}$). It is important to note that the classifier $h^*$ is not necessarily unique, nor is its corresponding risk vector $\bm{r}^* =\bm{r}(h^*)$. Throughout this section, we assume that the hypothesis for Theorem \ref{theo:linearweightconvex} are satisfied, therefore, for every minimax vector  $\bm{r}^*$ there is a set of weights $\bm{\mu}^*$ such that $\bm{r^*}$ is a unique solution for Problem \ref{eq:linearweightproblem} ($\bm{r}^* =\bm{r}(\bm{\mu}^*)$). 

Computing $\bm{\mu}^*$ directly can be challenging, even when closed form solutions for the classifiers and risks are available, as shown in Theorem \ref{theo:briercrosspareto} for Brier Score and Cross Entropy. A potential approach to estimate $\bm{\mu}^*$ would be to perform sub-gradient descent on $\norminf{\bm{r}(\bm{\mu})}$. This approach suffers from two main setbacks. First, closed form formulas for the Jacobian $\nabla_{\bm{\mu}} \bm{r}(\bm{\mu})$ require accurate estimates of the conditional distributions $p(x|a)$ and $p(y|x,a)$; or of $p(a),p(a|x)$ and $p(y|x,a)$. Secondly, $\norminf{\bm{r}(\bm{\mu})}$ can potentially have local minima on $\bm{\mu}$.

We propose a simple optimization method to recover $\bm{\mu}^*$ that only requires access to function evaluations of $\bm{r}(\bm{\mu})$. Note that given $\bm{\mu}$, the risk vector $\bm{r}(\bm{\mu})$ can be obtained by estimating the optimal classifier with the expression derived in Theorem \ref{theo:briercrosspareto} (plug-in estimation), which requires the conditional densities $p(a),p(a|x)$ and $p(y|x,a)$. Another option is to directly minimize Problem, \ref{eq:linearweightproblem} (joint estimation). Both approaches can be implemented using DNNs for estimation. Note that the second approach makes use of all samples to estimate a single classifier, while estimating the densities $p(a|x), p(y|x,a)$ individually may suffer from data fragmentation and cannot benefit from transfer learning, which could be harmful for a sensitive group with limited data \cite{ustun2019fairness,wang2020split}.

To avoid blind sampling of the weighting vectors $\bm{\mu}$, Theorem \ref{theo:Starshaped} summarizes important properties that any weighting vector $\bm{\mu}'$ must satisfy to improve the minimax risk at any given iteration ($\bm{\mu}': \norminf{\bm{r}(\bm{\mu'})}<\norminf{\bm{r}(\bm{\mu})}$).

\begin{theorem}
Let $\pra$ be a Pareto front, and $\bm{r}(\bm{\mu}) \in \pra$ denote the solution to the linear weighting Problem \ref{eq:linearweightproblem}. For any $\bm{\mu'} \not\in \argmin\limits_{\bm{\mu} \in \Delta^{|\mathcal{A}|-1}}||\bm{r}(\bm{\mu})||_{\infty}$, and $\bm{\mu}^*\in \argmin\limits_{\bm{\mu} \in \Delta^{|\mathcal{A}|-1}}||\bm{r}(\bm{\mu})||_{\infty}$, the sets $ N_{i} = \{\bm{\mu}:r_{i}(\bm{\mu})<|| \bm{r}(\bm{\mu'})||_{\infty}\}$
satisfy:
\vspace{-.5em}
\begin{enumerate}
    \item $\bm{\mu}^* \in \bigcap\limits_{i \in \mathcal{A}}N_i$;
    \vspace{-.5em}
    \item If $\bm{\mu} \in N_i \rightarrow \lambda\bm{\mu}+(1-\lambda)\bm{e^i} \in N_i$, $\forall\, \lambda \in [0,1], i=1,\dots,|\mathcal{A}|$, where $\bm{e^i}$ denotes the standard basis vector;
    \vspace{-.5em}
    \item $\forall\, \mathcal{I} \subseteq \mathcal{A}$, $ \bm{\mu}: \mu_{\mathcal{A}\backslash\mathcal{I}} = 0 \rightarrow \bm{\mu} \in \bigcup\limits_{i \in \mathcal{I}}N_i$;
    \vspace{-.5em}
    \item If $\bm{r}(\bm{\mu})$ is also continuous in $\bm{\mu}$, then $\forall\, \mathcal{I} \subseteq \mathcal{A}$ such that $\bm{\mu} \in \bigcap\limits_{i \in \mathcal{I}}N_i \rightarrow \exists \epsilon>0: B_{\epsilon}(\bm{\mu}) \subset  \bigcap\limits_{i \in \mathcal{I}}N_i$;
    \vspace{-.5em}
    \item If $\pra$ is also convex, then $\bm{r}(\bm{\mu}^*) \in \argmin\limits_{\bm{r}\in\pra}||\bm{r}||_{\infty}$.
\end{enumerate}
\label{theo:Starshaped}
\end{theorem}

\vspace{-1.5em}

Therefore, finding an updated weighting vector $\bm{\mu}$ that diminishes the minimax risk is equivalent to finding an element in the intersection of $|\mathcal{A}|$ subsets defined on the $\Delta^{|\mathcal{A}|-1}$ simplex. These subsets $N_i$ are themselves star-shaped sets w.r.t. the basis element $\bm{e^i}$, whose intersections $\cap_{i \in \mathcal{I}}N_i$ form open sets. Property $3$ in the theorem provides a straightforward rule to find elements belonging to any arbitrary union of the coordinate descent regions $\cup_{i \in \mathcal{I}}N_i$.

\begin{figure}[ht!]
 \centering
 \includegraphics[width=0.9\linewidth]{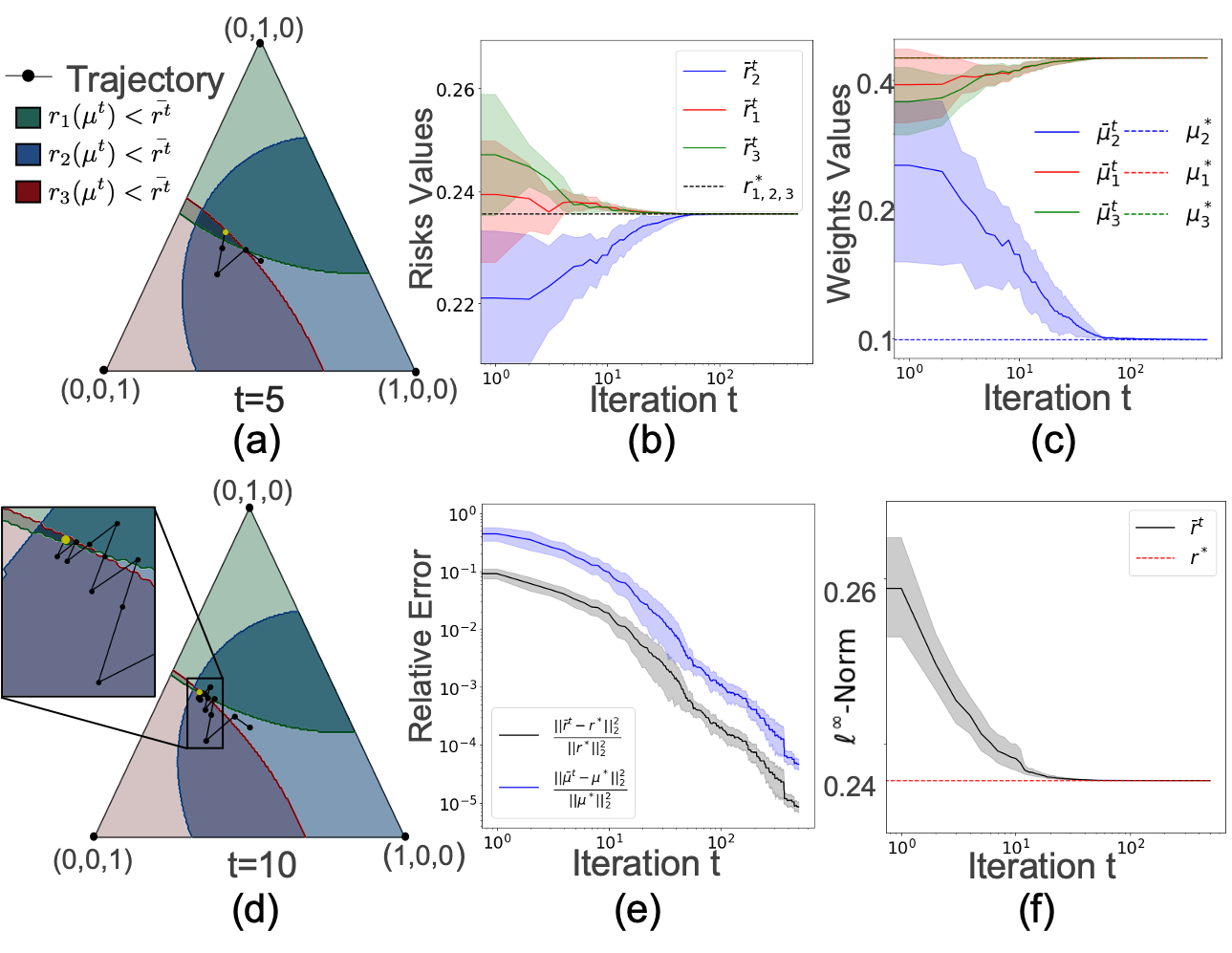}
\caption{\footnotesize Synthetic data experiment with $3$ sensitive groups. (a) and (d) show a simplex diagram of the linear weights $\bm{\mu}$ on the fifth and tenth iteration of the APStar algorithm; blue, green and red shaded areas correspond to the $N_i$ areas at iterations, $t=5, t=10$, the optimal linear weight lies in their intersection. (b) and (c) show the risk values and linear weights as a function of the iteration counter; shaded regions represent standard deviations across $5$ randomized runs. (e) shows relative error as a function of iterations for both risks and weights; (f) shows similar information comparing the maximum risk against the theoretical optimal. Risks for all groups converge to the minimax value, while the weights converge to $\bm{\mu}^*$. Simulation details are provided in Section \ref{sec:SyntheticAppendix}.}
\label{fig:AlgorithmSynthetic}
\end{figure}

Using these observations, we propose the Approximate Projection onto Star Sets (APStar) Algorithm (Algorithm \ref{alg:main_algo}) to iteratively refine the minimax risk by updating the linear weighting vector. The main intuition behind this algorithm is that whenever we observe a weighting vector that reduces the risk in a subset of groups $\mathcal{I}$ ($\bm{\mu}^t \in \cap_{i\in \mathcal{I}}N_i$), we can generate a new vector $\bm{\mu}^{t+1}$ that linearly interpolates between $\bm{\mu}^t$ and a vector $\bm{\mu}^{\mathcal{A}\setminus\mathcal{I}}\in \cup_{i\in \mathcal{A}\setminus\mathcal{I}}N_i$ that belongs to the union of the unsatisfied group risks ($\bm{\mu}^{t+1} \rightarrow \alpha \bm{\mu}^{t} +(1-\alpha)\bm{\mu}^{\mathcal{A}\setminus\mathcal{I}}$). An analysis of the convergence properties of APStar is provided in Section \ref{sec:AppendixOptimization}.

Figure \ref{fig:AlgorithmSynthetic} illustrates how the linear weights $\bm{\mu}$ are updated on a synthetic example, reducing the minimax risk; the example shown is of a classifier with $3$ sensitive groups where perfect fairness is attainable, we observe how risks converge to their common final value $\bm{r}^*$ and that the weights $\bm{\mu}^*$ required to recover this are not equally-weighted vector.

\begin{algorithm}[tb!]
   \caption{APStar}
   \label{alg:main_algo}
  \small
\begin{algorithmic}
   \STATE {\bfseries Input:} hypothesis class: $\mathcal{H}$, initial weights: $\bm{\mu}$, risk functions: $r_a(\cdot)$, optimizer: $\{\arg\}\min\limits_{h \in \mathcal{H}} \sum\limits_{i=1}^{|\mathcal{A}|} \mu_i r_i(h)$, $\alpha\in(0,1)$, $K_{min}$
   \STATE {\bfseries Initialize:}
   \STATE $h, \bm{r}(\mu) \leftarrow \{\arg\}\min\limits_{h \in \mathcal{H}} \sum \mu_i r_i(h)$
   \STATE $\bar{\bm{r}}\leftarrow \norminf{\bm{r}(\mu)};\, K\leftarrow 1.$
   \REPEAT 
   \STATE $\bm{1}_\mu \leftarrow \{\bm{1}(r_i(\mu)\ge \bar{\bm{r}})\}_{i=1}^{|\mathcal{A}|}$
   \STATE $\bm{\mu} \leftarrow (\alpha\bm{\mu} +\frac{1-\alpha}{K\normone{\bm{1}_\mu}}\bm{1}_\mu)\frac{K}{(K-1)\alpha+1}$
   \STATE $h, \bm{r}(\mu) \leftarrow \{\arg\}\min\limits_{h \in \mathcal{H}} \sum \mu_i r_i(h);\; K \leftarrow K+1$
   \IF{$\norminf{\bm{r}(\mu)} < \bar{\bm{r}}$}
   \STATE $\bar{\bm{r}} \leftarrow \norminf{\bm{r}(\mu)}, \; K \leftarrow \min(K,K_{min})$
   \STATE $h^*, \bm{\mu}^*, \bm{r^*} \leftarrow h, \bm{\mu}, \bm{r}(\mu)  $
   \ENDIF
   \UNTIL{$Convergence$}
   \STATE {\bfseries Return:} $h^*, \bm{\mu}^*, \bm{r^*}$
\end{algorithmic}
 
\end{algorithm}
In Section \ref{sec:ExperimentsAndResults} we apply the APStar algorithm to synthetic and real datasets using DNNs and SGD to solve Problem \ref{eq:linearweightproblem}. Note that APStar can be applied to non-convex risk functions and hypothesis classes, in which case the solutions to Problem \ref{eq:linearweightproblem} form a subset of the Pareto front and may not offer a full characterization of all optimal solutions. In Section \ref{sec:ImplementationDetails} we provide implementation details; Pytorch code for this algorithm will be made available.

\section{Experiments and Results}
\label{sec:ExperimentsAndResults}
We applied the proposed APStar algorithm to learn a minimax Pareto fair classifier (MMPF) and show how our approach produces well calibrated models that improve minimax performance across several metrics beyond the risk measure itself. While details are provided in Section \ref{sec:StarshapedSampling}, Figure \ref{fig:simplex2x2convergence} illustrates how our algorithm is empirically convergent and significantly faster than random sampling and the multiplicative weight update (MWU) algorithm proposed in \cite{chen2017robust} for minimax optimization in the context of robustness. In sections \ref{sec:SyntheticAppendix} and \ref{sec:PlugvJointAppendix} we validate the optimization algorithm on synthetic data and compare joint versus plug-in estimation. Here we compare the performance of our method against other state of the art approaches on a variety of public fairness datasets.

\begin{figure}[ht!]
 \centering
 \includegraphics[width=1\linewidth]{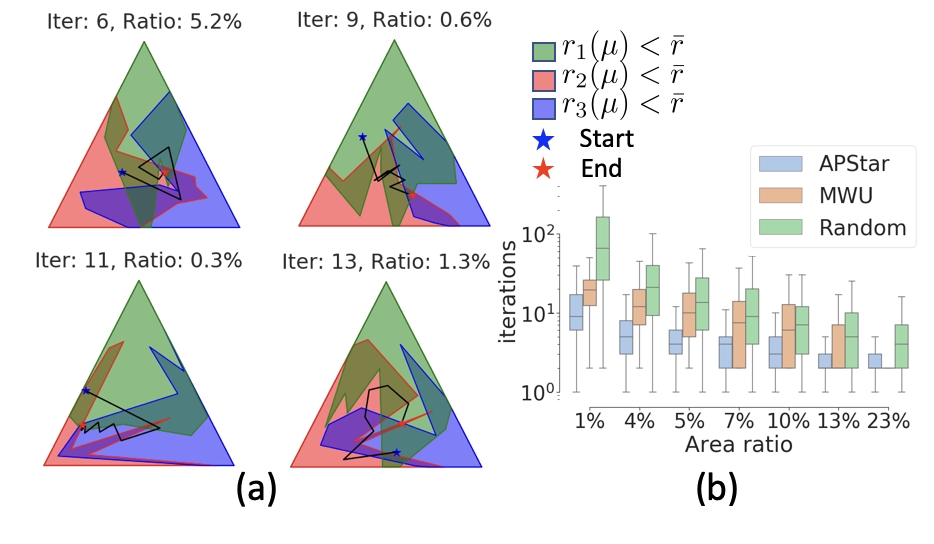}
 \vspace{-2em}
 \caption{\footnotesize Synthetic data experiment on star-shaped sets. (a) Randomly sampled star sets satisfying the conditions of Theorem \ref{theo:Starshaped}; a starting point is sampled (Blue), trajectories recovered by the APStar algorithm are recorded until convergence (Red); number of iterations and intersection area are shown for all examples. (b) Empirical distribution of number of iterations required to converge versus percentage of linear weights that lie in the triple intersection; values are shown for APStar, random sampling, and the multiplicative weight update (MWU) algorithm proposed in \cite{chen2017robust} for minimax optimization. The number of iterations required by the algorithm is well below the random sampler, this is especially apparent for low area ratio scenarios.}
\label{fig:simplex2x2convergence}
\end{figure}
 \vspace{-1em}
\subsection{Real Datasets: Methods and Metrics}
We evaluate our method on mortality prediction (MIMIC-III), skin lesion classification (HAM10000), income prediction (Adult), and credit lending (German). The latter two are common benchmarks in the fairness literature. Results are reported for joint estimation (MMPF) and plug-in estimation (MMPF P), presented in Section \ref{sec:OptimizationMethods}. We evaluate a model trained to minimize the average risk (Naive) and one that samples all sensitive groups equally (Balanced). When applicable, we compare our results against the methodologies proposed in \cite{hardt2016equality,zafar2015fairness,kamishima2012fairness, feldman2015certifying}; for implementations on all methods except \cite{hardt2016equality}, we used the unified test-bed provided in \cite{friedler2019comparative}. A description of these methods is provided in Section \ref{sec:MethodsAppendix}.

Metrics used for evaluation include accuracy (Acc), Brier Score (BS) and Cross-Entropy (CE), values are reported per sensitive attribute, standard deviations computed across $5$ splits are shown when available. For datasets containing more than two sensitive groups (MIMIC-III, HAM10000) we report dataset average (sample mean), group-normalized average (group mean), worst performing group (worst group), and largest difference between groups (disparity). Note that the latter two are especially important in our setting, since our explicit focus is to minimize worst-case performance, and disparity is a common measure of interest in fairness literature. For an in-depth description of these metrics and datasets, refer to Sections \ref{sec:Metrics} and \ref{sec:realDatasets}. Additional tables showing these and other metrics on a per-group basis are provided in Section \ref{sec:SupplementaryResults}. All tables bold the best result for disparity and worst group risk.

Similarly to \cite{kamishima2012fairness,zafar2015fairness,hardt2016equality,woodworth2017learning}, we omit the sensitive attribute from our observation features, which broadens the potential application of the framework. Classifiers are implemented using neural networks and/or linear logistic regression; for details on architectures and hyper-parameters, refer to Section \ref{sec:AppendixArchitectures}.  

\subsection{Predicting Mortality in Intensive Care Patients}

We used clinical notes collected from adult ICU patients at the Beth Israel Deaconess Medical Center (MIMIC-III dataset) \cite{johnson2016mimic} to predict patient mortality. We study fairness with respect to age (adult/senior), ethnicity (white/nonwhite), and outcome (alive/deceased) simultaneously (8 sensitive groups). Outcome is included as a sensitive attribute because, in our experiments, patients who ultimately passed away on ICU were under-served by a Naive classifier (high classifier loss). Using the target label as sensitive attribute addresses class imbalance by requiring similar risk performance on each target label. It also demonstrates a use-case where group membership would not be available at test-time.

We use a fully connected NN with BS loss as the base hypothesis class (results with CE provided in Section \ref{sec:SupplementaryResults}), input features are Tf-idf statistics on the $10k$ most frequent words from clinical notes, mirroring \cite{chen2018my}. For an even comparison, we provided the feature embeddings of the Naive classifier as input to the baselines, since the implementation in \cite{friedler2019comparative} only includes linear classifiers, this was also done since some of the available implementations failed to converge in a reasonable time with the original inputs. Table \ref{table:MIMICOverallTable} reports Acc and BS of tested methodologies. Note that, while the Balanced classifier has significantly better worst-case BS performance than the Naive classifier, MMPF is better still; these performance gains are also reflected on Acc, showing that this improvement goes beyond the training metric. Plug-in performance (MMPF P) is not an improvement over joint estimation. Accuracy comparison table incorporates the post-processing methodology described in \cite{hardt2016equality} to achieve zero accuracy disparity (``+H'' suffix). Hardt post-processing decreases the  accuracy disparity gap w.r.t. the baseline methods but requires test-time access to group membership. The best result is obtained on MMPF H, but we note that the best minimax risk is still attained on the original MMPF model.

\begin{table}[h!]
\caption{MIMIC dataset. Group priors range from 0.4\% to 57\%. In this and all tables we bold the best result for disparity and worst group risk.}
\label{table:MIMICOverallTable}
\centering
\footnotesize
(a) Acc comparison\\
\begin{tabular}{lllll}
\toprule
{} & \makecell{Sample\\mean} & \makecell{Group\\mean} & \makecell{Worst\\group} & Disparity \\
\midrule 
Naive & 89.5 $\pm$0.2 & 61.9 $\pm$1.7 & 19.0 $\pm$2.0 & 80.5 $\pm$1.3 \\
Balanced & 79.4$\pm$0.6 & 77.5$\pm$1.4 & 66.8$\pm$2.2 & 22.6$\pm$2.3 \\
Zafar         &  86.2$\pm$0.3 &  65.8$\pm$1.8 &  32.0$\pm$2.4 &  62.9$\pm$3.6 \\
Feldman       &  88.6$\pm$2.4 &  64.4$\pm$2.9 &   28.7$\pm$2.4 &  72.1$\pm$5.5\\
Kamishima     &  89.3$\pm$0.2 &  63.6$\pm$2.0 &    25.1$\pm$5.1 &  76.4$\pm$5.2 \\

MMPF & 76.2$\pm$0.2 & 78.3$\pm$1.5 & \textbf{72.6$\pm$1.7} & \textbf{17.1$\pm$3.5} \\
MMPF P & 75.5$\pm$1.0 & 76.8$\pm$1.3 & 70.7$\pm$2.1 & 17.8$\pm$3.8 \\
\midrule
Balanced+H & 75.6$\pm$1.1 &  71.7$\pm$1.6 & 65.6$\pm$2.8& 19.1$\pm$1.8  \\
Zafar+H & 62.8$\pm$1.6 &  58.3$\pm$2.1 & 51.5$\pm$2.8& 17.8$\pm$3.1  \\
MMPF+H& 72.4$\pm$1.1 &  72.3$\pm$1.5 &\textbf{72.0$\pm$3.7} & \textbf{11.4$\pm$3.5}\\

\bottomrule
\end{tabular}
\vspace{.5em}
(b) BS comparison\\
\begin{tabular}{lllll}
\toprule
{} & \makecell{Sample\\mean} & \makecell{Group\\mean} & \makecell{Worst\\group} & Disparity \\
\midrule 
Naive & .16 $\pm$.01 & .51 $\pm$.02 & 1.05 $\pm$.01 & 1.03 $\pm$.02 \\
Balanced & .28$\pm$.01 & .31$\pm$.01 & .42$\pm$.01 & .25$\pm$.04 \\
Zafar         &   .27$\pm$.01 &  .67$\pm$.04 &   1.34$\pm$.05 &  1.25$\pm$.07 \\
Feldman       &  .19$\pm$.04 &  .62$\pm$.04 &  1.26$\pm$.08 &   1.29$\pm$.08 \\
Kamishima     &  .16$\pm$.01 &  .53$\pm$.03 &  1.06$\pm$.04 &  1.11$\pm$.08 \\

MMPF & .32$\pm$.01 & .3$\pm$.01 & \textbf{.35$\pm$.02} & \textbf{.17$\pm$.03} \\
MMPF P & .33$\pm$.01 & .32$\pm$.01 & .37$\pm$.01 & \textbf{.17$\pm$.04} \\
\bottomrule
\end{tabular}
 \vspace{-1.2em}
\end{table}

\subsection{Skin Lesion Classification}

The HAM10000 dataset \cite{tschandl2018ham10000} contains over $10k$ dermatoscopic images of $7$ types of skin lesions; with ratios between $67\%$ and $1.1\%$. A Naive classifier exhibited no significant discrimination based on age or race. We instead chose to use the diagnosis class as both the target and sensitive variable. It was not possible to compare against \cite{hardt2016equality} since the sensitive attribute is perfectly predictive of the outcome; likewise, \cite{zafar2015fairness} and \cite{kamishima2012fairness} cannot handle non-binary target attributes in their provided implementations. Table \ref{table:HAMOverallTable} shows results for MMFP P, since plug-in estimation is equivalent to joint estimation when $A=Y$, but enables cheaper APStar iterations (see Section \ref{sec:PlugvJointAppendix}). The MMPF classifier improves minimax BS and Acc when compared to both Naive and Balanced classifiers.

\begin{table}[h!]
\caption{HAM10000 dataset. Group priors range from 1\% to 67\%.}
\label{table:HAMOverallTable}
\centering
\footnotesize

(a) Acc comparison\\
\begin{tabular}{lllll}
\toprule
{} & \makecell{Sample\\mean} & \makecell{Group\\mean} & \makecell{Worst\\group} & Disparity \\
\midrule 
Naive & 78.5 $\pm$ 0.5 & 50.8 $\pm$ 1.9 & 2.6 $\pm$3.5 & 93.7 $\pm$1.1 \\
Balanced & 70.1 $\pm$2.1 & 70.1 $\pm$2.2 & 52.6 $\pm$5.3 & 32.5 $\pm$ 4.6\\
MMPF P & 64.7 $\pm$1.2 & 66.7 $\pm$ 3.5& \textbf{56.9 $\pm$3.1} & \textbf{19.8 $\pm$6.6} \\

\bottomrule
\end{tabular}

(b) BS comparison\\
\begin{tabular}{lllll}
\toprule
{} & \makecell{Sample\\mean} & \makecell{Group\\mean} & \makecell{Worst\\group} & Disparity \\
\midrule 
Naive & .31$\pm$.01 & .69$\pm$.3 & 1.38$\pm$.04 & 1.29$\pm$.04 \\
Balanced & .41$\pm$.02 & .42$\pm$.03 & 0.64$\pm$.05 & 0.45$\pm$.07 \\
MMPF P & .49$\pm$.02 & .46$\pm$.04 & \textbf{0.56$\pm$0.4} & \textbf{0.23$\pm$.06} \\

\bottomrule
\end{tabular}
 \vspace{-1.2em}
\end{table}

 \vspace{-0.5em}
\subsection{Income Prediction and Credit Risk Assesment}
We predict income in the Adult UCI dataset \cite{Dua:2019} and assess credit risk in the German Credit dataset \cite{Dua:20192}. We select gender (Male/Female) as our sensitive attribute, additional results for gender and ethnicity (Male/Female and White/Other) are also shown for the Adult dataset.
Tables \ref{table:AdultGender}, \ref{table:AdultMaleWhite},  and \ref{table:GermanGender} show results for linear logistic regression (LR suffix) and a fully connected NN. Note that linear logistic regression is not a convex hypothesis class, but is included to compare evenly against the baselines.

Our approach leads to the best worst-case performance in Acc and CE on the Adult dataset, although all methods perform similarly; this is especially true for the gender case, where Kamishima has a slight advantage in terms of standard deviation in CE. On the German dataset, our method produces the best worst-case CE results, and smallest disparities in both Acc and CE; Feldman does, however, have test time access to sensitive attributes, which may explain the difference in Acc. We show results for Hardt post-processing in Section \ref{sec:SupplementaryResults}, with similar conclusions to the ones made on the MIMIC dataset.


\vspace{-1.3em}
\begin{table} [h!]
\centering
\caption{Adult gender dataset. Females represent 32\% of samples.}
\label{table:AdultGender}
\footnotesize

(a) Acc comparison\\
\begin{tabular}{llll}
\toprule
{} & \makecell{Female} & \makecell{Male} & Disparity \\
\midrule 
Naive LR & 92.3$\pm$0.4 & 80.5$\pm$0.4 & 11.9$\pm$0.7 \\
Balanced LR & 92.3$\pm$0.3 & 80.3$\pm$0.7 & 12.0$\pm$0.7 \\

Zafar & 92.5$\pm$0.3 & 80.9$\pm$0.3 & 11.6$\pm$0.4 \\
Feldman & 92.3$\pm$0.3 & 80.7$\pm$0.2 & 11.6$\pm$0.1 \\
Kamishima & 92.6$\pm$0.4 & 80.9$\pm$0.4 & 11.7$\pm$0.7 \\

MMPF LR & 91.9$\pm$0.4 & 81.0$\pm$0.4 & 10.9$\pm$0.7\\
MMPF & 92.1$\pm$0.3 & 81.3$\pm$0.3 & 10.8$\pm$0.5 \\
MMPF LR P & 92.0$\pm$0.4 & 81.0$\pm$0.5 & 11.0$\pm$0.6 \\
MMPF P & 91.7$\pm$0.3 & \textbf{81.5$\pm$0.5} & \textbf{10.1$\pm$0.5} \\

\bottomrule
\end{tabular}

(b) CE comparison \\
\begin{tabular}{llll}
\toprule
{} & \makecell{Female} & \makecell{Male} & Disparity \\
\midrule 
Naive LR & .204$\pm$.009 & .411$\pm$.006 & .207$\pm$.007\\
Balanced LR & .204$\pm$.011 & .416$\pm$.011 & .211$\pm$.005 \\

Zafar & .202$\pm$.018 & .398$\pm$.006 & .195$\pm$.023 \\
Feldman & .201$\pm$.004 & .403$\pm$.004 & .203$\pm$.006 \\
Kamishima & .189$\pm$.006 & \textbf{.395$\pm$.004} & .206$\pm$.007 \\

MMPF LR & .204$\pm$.008 & \textbf{.395$\pm$.006} & .19$\pm$.011 \\
MMPF  & .21$\pm$.019 & .403$\pm$.025 & .193$\pm$.013 \\
MMPF LR P & .208$\pm$.008 & \textbf{.395$\pm$.005} & .187$\pm$.01 \\
MMPF P & .227$\pm$.019 & .403$\pm$.023 & \textbf{.176$\pm$.014} \\
\bottomrule
\end{tabular}
\end{table}

\vspace{-1.9em}
\begin{table} [h!]
\centering
\caption{Adult ethnicity and gender dataset. Group priors range from 6\% to 60\%.}
\label{table:AdultMaleWhite}
\small
\footnotesize

(a) Acc comparison\\

\begin{tabular}{lllll}
\toprule
{} & \makecell{Sample\\mean} & \makecell{Group\\mean} & \makecell{Worst\\group} & Disparity \\
\midrule 
Naive LR & 84.7$\pm$0.3 &  87.8$\pm$0.1 &  80.6$\pm$0.5 &  14.1$\pm$1.0 \\
Balanced LR & 84.7$\pm$0.3 &  88.0$\pm$0.3 &  80.5$\pm$0.5 &  14.5$\pm$1.0 \\

Zafar & 84.7$\pm$0.2 &  87.9$\pm$0.2 &  80.6$\pm$0.5 &  14.5$\pm$0.9 \\
Feldman & 84.5$\pm$0.2 &  87.7$\pm$0.3 &  80.4$\pm$0.3 &  14.7$\pm$0.9 \\
Kamishima & 84.3$\pm$0.8 &  87.8$\pm$0.3 &  80.0$\pm$1.2 &  15.2$\pm$1.8 \\

MMPF LR& 84.5$\pm$0.2 &  87.8$\pm$0.3 &  80.6$\pm$0.5 &  14.0$\pm$1.0 \\
MMPF & 84.8$\pm$0.3 &  87.8$\pm$0.4 &  80.9$\pm$0.6 &  13.6$\pm$1.5 \\
MMPF LR P & 84.6$\pm$0.3 &  87.7$\pm$0.2 &  80.7$\pm$0.5 &  13.9$\pm$1.0 \\
MMPF P & 84.6$\pm$0.4 &  87.6$\pm$0.5 &  \textbf{81.0$\pm$0.8} &  \textbf{13.4$\pm$1.5} \\
\bottomrule
\end{tabular}

(b) CE comparison\\
\begin{tabular}{lllll}
\toprule
{} & \makecell{Sample\\mean} & \makecell{Group\\mean} & \makecell{Worst\\group} & Disparity \\
\midrule 
Naive LR & .332$\pm$.004 &  .268$\pm$.004 &  .408$\pm$.008 &  .268$\pm$.016 \\
Balanced LR & .333$\pm$.004 &  .268$\pm$.005 &  .411$\pm$.008 &  .273$\pm$.015 \\

Zafar & .334$\pm$.005 &  .273$\pm$.005 &   .409$\pm$.01 &   .266$\pm$.03 \\
Feldman & .337$\pm$.003 &  .276$\pm$.006 &  .412$\pm$.006 &  .262$\pm$.016 \\
Kamishima &.337$\pm$.015 &  .275$\pm$.009 &  .414$\pm$.023 &  .269$\pm$.026 \\

MMPF LR& .334$\pm$.004 &  .274$\pm$.005 & \textbf{.404$\pm$.007} &  \textbf{.251$\pm$.015} \\
MMPF& .334$\pm$.005 &  .272$\pm$.003 &  .404$\pm$.009 &  .263$\pm$.022 \\
MMPF LR P & .335$\pm$.005 &  .275$\pm$.006 &  .405$\pm$.006 &   \textbf{.251$\pm$.01} \\
MMPF P & .345$\pm$.009 &   .284$\pm$.01 &   .41$\pm$.014 &   .258$\pm$.03 \\
\bottomrule
\end{tabular}
\end{table}

\begin{table} [h!]
\centering
\caption{German dataset. Females represent 30\% of samples.}
\label{table:GermanGender}
\footnotesize

(a) Acc comparison\\

\begin{tabular}{llll}

\toprule
{} & \makecell{Female} & \makecell{Male} & Disparity \\
\midrule 
Naive LR & 70.7$\pm$7.3 & 71.2$\pm$4.5 & 8.8$\pm$4.7 \\
Balanced LR & 71.6$\pm$5.9 & 70.9$\pm$4.1 & 5.8$\pm$3.6 \\

Zafar & 73.0$\pm$5.6 & 71.0$\pm$3.5 & 5.8$\pm$3.5 \\
Feldman & 73.5$\pm$8.6 & \textbf{71.9$\pm$4.3} & 7.9$\pm$4.4 \\
Kamishima & 68.8$\pm$6.8 & 72.7$\pm$2.6 & 6.0$\pm$4.4 \\

MMPF LR & 72.5$\pm$5.5 & 71.6$\pm$2.8 & 5.0$\pm$2.6 \\
MMPF LR P& 70.7$\pm$4.5 & 71.5$\pm$3.6 & \textbf{4.4$\pm$0.5} \\

\bottomrule
\end{tabular}

(b) CE comparison\\
\begin{tabular}{llll}
\toprule
{} & \makecell{Female} & \makecell{Male} & Disparity \\
\midrule 
Naive LR & .607$\pm$.1 & .559$\pm$.069 & .127$\pm$.064 \\
Balanced LR & .594$\pm$.082 & .568$\pm$.068 & .096$\pm$.05 \\

Zafar & .567$\pm$.09 & .735$\pm$.205 & .273$\pm$.151 \\
Feldman & .564$\pm$.096 & .551$\pm$.063 & .091$\pm$.068 \\
Kamishima & .62$\pm$.064 & .545$\pm$.062 & .075$\pm$.067 \\

MMPF LR& .565$\pm$.04 & .544$\pm$.046 & \textbf{.048$\pm$.041} \\
MMPF LR P & \textbf{.563$\pm$.043} & .537$\pm$.051 & .057$\pm$.034 \\

\bottomrule
\end{tabular}
\end{table}

\section{Discussion}
\label{sec:Discussion}
Here we formulate group fairness as a multi objective optimization problem where each group-specific risk is an objective function. Our goal is to recover an efficient classifier that reduces worst-case group risks ethically (i.e., avoiding unnecessary harm). We consider problems where target labels available for training are trustworthy (not affected by discrimination). We formally characterized Pareto optimal solutions for a family of models and risk functions, yielding insight on the fundamental sources of risk trade-offs. We proposed a simple algorithm to recover a model that improves minimax group risk (MMPF), and does not require test-time access to sensitive attributes.

We demonstrated the proposed framework and optimization algorithm on several real-world case studies, achieving state-of-the-art performance. The algorithm is straightforward to implement, and is agnostic to the hypothesis class, risk function and optimization method, which allows integration with a variety of classification pipelines, including neural networks. If the hypothesis class or risk functions are not convex, the algorithm can still be deployed to recover Pareto-efficient models that reduce minimax risks, though minimax optimality might not be achievable by optimizing a linear weighting of the risk functions. While this paper addresses minimal harm, other considerations like marginal risk tradeoffs between groups may be of interest. Controlling these in the proposed framework can be achieved by adding a constraint on the ratio between linear weights.

As an avenue of future research, we would like to automatically identify high-risk sub-populations as part of the learning process and attack risk disparities as they arise, rather than relying on preexisting notions of disadvantaged groups. The APStar algorithm is empirically convergent, but a formal proof or counterexample is desireable. We strongly believe that Pareto-efficient notions of fairness are of great interest for several applications, especially so on domains such as healthcare, where quality of service is paramount.

\section*{Acknowledgments}

Work partially supported by ONR, ARO, NGA, NSF, NIH, Simons Foundation, and gifts from AWS, Microsoft, and Google.


\bibliography{icml_fairness}
\bibliographystyle{icml2020}

\appendix
\clearpage

\section{Supplementary Material}
\label{sec:Appendix}

\subsection{Proofs}
\label{sec:AppendixProofs}
Here we provide proof outlines for all lemmas and theorems in sections \ref{sec:ProblemStatement},\ref{sec:AnalysisPOSols}, and \ref{sec:OptimizationMethods}.

\paragraph{Lemma \ref{lemma:perfectfairnessminmax}.} If $\exists h^* \in \pha:  r_a(h^*)=r_{a'}(h^*),\forall a,a' \in \mathcal{A}$ then $ \bm{r}(h^*) = \argmin\limits_{\bm{r} \in \pra} \norminf{\bm{r}}$.

\begin{proof}
By contradiction, assume $\exists\, h' \in \pha: \norminf{\bm{r}(h')} <  \norminf{\bm{r}(h^*)}$. Then we have 
\[r_a(h') \le \norminf{\bm{r}(h')} < \norminf{\bm{r}(h^*)} = r_a(h^*),\; \forall a \in \mathcal{A}\]
And therefore $\bm{r}_{h'} \prec \bm{r}_{h*} \rightarrow h* \not\in \pha$, which contradicts the hypothesis.
\end{proof}

\paragraph{Lemma \ref{lemma:eopareto}.} Let $h_{ER} \in \mathcal{H}$ be an equal risk classifier such that $r_a(h_{ER})=r_{a'}(h_{ER}) \forall\, a, a'$, and let $h^*$ be the Pareto fair classifier. Additionally, define the Pareto fair post-processed equal risk classifier $h^*_{ER} : r_{a}(h^{*}_{ER}) = ||\bm{r}(h^*)||_{\infty} \forall \,a \in \mathcal{A}$, then we have
\begin{proof}
Note that $\norminf{\bm{r}_{h_{ER}}} \ge \norminf{\bm{r}_{h^*}}$, otherwise $\norminf{\bm{r}_{h_{ER}}} < \norminf{\bm{r}_{h^*}}$ and also $h_{ER}\preceq h^*$, which contradicts the definition of $h^*$.

The statement therefore follows from $r_a(h_{ER}) = \norminf{\bm{r}_{h^{ER}}}  \ge \norminf{\bm{r}_{h^*}} = r_a (h^*_{ER})\ge r_a(h^*)\; \forall a \in \mathcal{A}$. 

\end{proof}

\paragraph{Theorem \ref{theo:linearweightconvex}.}
Given $\mathcal{H}$ a convex hypothesis class and $\{r_a(h)\}_{a \in \mathcal{A}}$ convex risk functions then:
\begin{enumerate}
    \item The Pareto front is convex: $\forall \bm{r}, \bm{r}' \in \pra,\, \lambda \in [0,1], \; \exists \bm{r}'' \in \pra: \bm{r}'' \preceq \lambda \bm{r} + (1-\lambda) \bm{r}' $.
    \item Every Pareto solution is a solution to Problem \ref{eq:linearweightproblem}: $\forall \bm{\hat{r}} \in \pra, \exists \bm{\mu}: \bm{\hat{r}} = \bm{r}(\bm{\mu})$.
\end{enumerate}

\begin{proof}
We prove the first item of the theorem statement, the second item is a direct application of the results in \cite{geoffrion1968proper}.

Let $h',h'' \in \pha$, with corresponding risk vectors $\bm{r}_{h'}, \bm{r}_{h''}$. Using the convexity of $r_a(h)$, $\forall \lambda \in [0,1]$ we have
\[\lambda r_a(h') + (1-\lambda) r_a(h')\ge r_a(\lambda h + (1-\lambda) h'').\]

Since $\mathcal{H}$ is convex, $h^\lambda = \lambda h' + (1-\lambda) h'' \in \mathcal{H}$. We have two possibilities;
\begin{itemize}
    \item $h^\lambda \in \pha$, therefore, by definition $\bm{r}_{h^\lambda} \in \pra $ and  $\bm{r}_{h^\lambda}\preceq \lambda \bm{r}' + (1-\lambda) \bm{r}''$;;
    \item $h^\lambda \notin \pha$, therefore $\exists\, \bm{\hat{r}} \in \pra, \bm{\hat{r}} \prec \bm{r}^\lambda$.
\end{itemize}

In both cases, for all risk vectors  $\bm{r}', \bm{r}'' \in \pra, \lambda \in [0,1]\; \exists\, \hat{\bm{r}} \in \pra: \hat{\bm{r}}\preceq \lambda \bm{r}' + (1-\lambda) \bm{r}''$

\end{proof}

\paragraph{Theorem \ref{theo:briercrosspareto}.} 
Given input features $X \in \mathcal{X}$ and categorical target and sensitive group variables $Y \in \mathcal{Y}$ and $A \in \mathcal{A}$ respectively, with joint distribution $p(X,Y,A)$, and linear weights $\bm{\mu} = \{\mu_a\}_{a \in \mathcal{A}}$, the optimal predictor to the linear weighting problem $h(\bm{\mu})$ for both Brier score and Cross-Entropy is
\begin{equation*}
 \begin{array}{cc}
h^{\bm{\mu}}(x)  = \frac{\sum_{a\in\mathcal{A}}\mu_a p(x|a) p(y|x,a)}{\sum_{a\in\mathcal{A}}\mu_a p(x|a)},
\end{array}
\end{equation*}
\noindent with corresponding risks
\begin{equation*}
 \begin{array}{ll}
 r^{BS}_a(\bm{\mu}) = E_{X,Y|a}[||\oh{Y}-p(y|X,a)||^2_2] +\\
        \hspace{0.6in} E_{X|a}\Big[||p(y|X,a) - h^{\bm{\mu}}(X)||^2_2\Big], \\
r^{CE}_a(\bm{\mu}) =H(Y|X,a)+\\ \hspace{0.6in} E_{X|a}\Big[D_{KL}\Big(p(y|X,a)\big|\big|h^{\bm{\mu}}(X)\Big)\Big], 
        
\end{array}
\end{equation*}
where $p(y|X,a) = \{p(Y=y_i|X,A=a)\}_{i=1}^{|\mathcal{Y}|}$ is the probability mass vector of $Y$ given $X$ and $A=a$. $H(Y|X,a)$ is the conditional entropy $H(Y|X,A=a) = E_{X|A=a}[H(Y|X=X,A=a)]$.

\begin{proof}
We observe that the linear loss function $\sum\limits_{a\in \mathcal{A}} \mu_a r_a(h) $ can be decomposed as
\begin{equation*}
    \begin{array}{l}
        \sum\limits_{a\in \mathcal{A}} \mu_a r_a(h)\\
        = \sum\limits_{a\in \mathcal{A}} \mu_a E_{X|a}\big[E_{Y|X,a} [\ell(h(X),Y)]\big]\\
          = \sum\limits_{a\in \mathcal{A}} \mu_a E_{X}\big[\frac{p(X|a)}{p(X)}E_{Y|X,a} [\ell(h(X),Y)]\big] \\
         = \sum\limits_{a\in \mathcal{A}} \mu_a E_{X}\big[\frac{p(X|a)}{p(X)}E_{Y|X}[\frac{p(Y|X,a)}{p(Y|X)}\ell(h(X),Y)]\big] \\
          =E_{X}\big[\frac{1}{p(X)}E_{Y|X}[\frac{\sum\limits_{a\in \mathcal{A}}\mu_a p(X|a)p(Y|X,a)}{p(Y|X)}\ell(h(X),Y)]\big]\\
          =E_{X}\big[\frac{\sum\limits_{a\in \mathcal{A}}\mu_a p(X|a)}{p(X)}E_{Y\sim P^\mu(Y|X)}[\ell(h(X),Y)]\big],
    \end{array}
\end{equation*}

with $ P^\mu(Y|X) = \frac{\sum\limits_{a\in \mathcal{A}}\mu_a p(X|a)p(y|X,a)}{\sum\limits_{a\in \mathcal{A}}\mu_a p(X|a)}$, and denoting $p(y|X,a) = \{p(Y=y_i|X,a)\}_{i=1}^{|\mathcal{Y}|}$ the conditional probability mass vector of $Y|X,a$.

For both the Cross-Entropy loss $\ell^{CE}(h(X),Y) = \langle \oh{Y}, \ln(h(X))\rangle$ and the Brier score loss  $\ell^{BS}(h(X),Y) = \normtwo{ \oh{Y} - h(X)}$, the minimizer of $E_{Y\sim P(Y|X)}[\ell(h(X),Y)]$ is attained at $h(X) = P(Y|X)$. Therefore, we can plug in this optimal estimator to recover

\begin{equation*}
 \begin{array}{cc}
h^{\bm{\mu}}(x)  = \frac{\sum_{a\in\mathcal{A}}\mu_a p(x|a) p(y|x,a)}{\sum_{a\in\mathcal{A}}\mu_a p(x|a)}.
\end{array}
\end{equation*}

Plugging in the optimal classifier $h^{\bm{\mu}}(x)$ on the risk formulations we get the expressions for both scores as analyzed next.

\textbf{Brier Score:}
\begin{equation*}
    \begin{array}{l}
        r^{BS}_a(\bm{\mu}) = \\ =E_{X,Y|a}[||\oh{Y}-h^{\bm{\mu}}(X)||^2_2]\\
        =E_{X,Y|a}[||\oh{Y}-p(y|X,a)+ p(y|X,a)-h^{\bm{\mu}}(X)||^2_2]\\
        =E_{X,Y|a}[||\oh{Y}-p(y|X,a)||^2_2]\\
        \quad + E_{X,Y|a}[|| p(y|X,a)-h^{\bm{\mu}}(X)||^2_2] +\\
        \quad + 2 E_{X,Y|a}\Big[[\oh{Y}-p(y|X,a)]^T[p(y|X,a)- h^{\bm{\mu}}(X)]\Big]\\
        =r_{a}^{BSmin} + E_{X|a}\Big[ ||p(y|X,a) - h^{\bm{\mu}}(X)||^2_{2} \Big] \\
        \quad +2E_{X|a}\Big[[p(y|X,a)-p(y|X,a)]^T\Big]\big[p(y|X,a) - h^{\bm{\mu}}(X)\big]\\
        =r_{a}^{BSmin} + E_{X|a}\Big[ ||p(y|X,a) - h^{\bm{\mu}}(X)||^2_{2} \Big],
    \end{array}
\end{equation*}
where the equalities are attained by expanding the l2-norm, observing that $p(y|X,a)$ and $h(X)$ do not depend on random variable $Y$, and also that $E_{Y|X,a}[\oh{Y}] = p(y|X,a)$. We also denoted $r_a^{BSmin} = E_{X,Y|a}[||\oh{Y}-p(y|X,a)||^2_2]$, which is the risk attained by the Bayes optimal estimator for group $a$.

\textbf{Cross Entropy:}
\begin{equation*}
    \begin{array}{l}
        r^{CE}_a(\bm{\mu}) =\\
        = -E_{X,Y|a}\Big[\sum^{|\mathcal{Y}|}_{i=1}\oh{Y}_i log[h^{\bm{\mu}}_i(X)]\Big]\\
        = -E_{X,Y|a}\Big[\sum^{|\mathcal{Y}|}_{i=1}\oh{Y}_i log[\frac{h^{\bm{\mu}}_i(X)p(y_i|X,a)}{p(y_i|X,a)}]\Big]\\
        = -E_{X,Y|a}\Big[\sum^{|\mathcal{Y}|}_{i=1}\oh{Y}_i log[p(y_i|X,a)]\Big]+\\
        \qquad +E_{X,Y|a}\Big[\sum^{|\mathcal{Y}|}_{i=1}\oh{Y}_i log[\frac{p(y_i|X,a)}{h^{\bm{\mu}}_i(X)}]\Big] =  \\
        = r_{a}^{CEmin} + E_{X|a}\Big[\sum^{|\mathcal{Y}|}_{i=1}p(y_i|X,a)log[\frac{p(y_i|X,a)}{h^{\bm{\mu}}_i(X)} \Big]\\
        = r_{a}^{CEmin} + E_{X|a}\Big[D_{KL}\Big(p(y|X,a)\big|\big|h^{\bm{\mu}}(X)\Big)\Big], \\
    \end{array}
\end{equation*}

where we again use the linearity of the expectation and the equality $E_{Y|X,a}[\oh{Y}] = p(y|X,a)$. We also denote $r_{a}^{CEmin} =  H(Y|X,A=a)$.
\end{proof}

\paragraph{Lemma \ref{lemma:limitingCases}.} In the conditions of Theorem \ref{theo:briercrosspareto} we observe that if $Y \perp A|X$ then

\begin{equation*}
 \begin{array}{ll}
 r^{BS}_a(\bm{\mu}) = E_{X,Y|a}[||\oh{Y}-p(y|X)||^2_2] \;\forall \bm{\mu} \\
r^{CE}_a(\bm{\mu}) =H(Y|X) \; \forall \bm{\mu},
\end{array}
\end{equation*}

Likewise, if $H(A|X) \rightarrow 0$ then 

\begin{equation*}
 \begin{array}{ll}
 r^{BS}_a(\bm{\mu}) \rightarrow E_{X,Y|a}[||\oh{Y}-p(y|X,a)||^2_2] \;\forall \bm{\mu} \\
r^{CE}_a(\bm{\mu})  \rightarrow H(Y|X,a) \; \forall \bm{\mu}. 
\end{array}
\end{equation*}
\begin{proof}
$ $\newline
\textbf{If} $Y \perp A|X$  then 
\begin{equation*}
    \begin{array}{rl}
    h^{\bm{\mu}}(x)  &= \frac{\sum_{a\in\mathcal{A}}\mu_a p(x|a) p(y|x,a)}{\sum_{a\in\mathcal{A}}\mu_a p(x|a)},\\
         & = p(y|x),
    \end{array}
\end{equation*}

in which case the resulting expressions for $r^{BS}_a(\bm{\mu})$, $r^{CE}_a(\bm{\mu})$ are immediate and do not depend on $\bm{\mu}$

\textbf{If} $H(A|X) \rightarrow 0$ then we have $p(a|x) =1, p(a'|x)=0 \,\forall\, a\neq a', x : p(x|a)>0$. Therefore we can write
\begin{equation*}
    \begin{array}{rl}
    h^{\bm{\mu}}(x)  &= \frac{\sum_{a\in\mathcal{A}}\mu_a p(x|a) p(y|x,a)}{\sum_{a\in\mathcal{A}}\mu_a p(x|a)}\\
    &= \frac{\sum_{a\in\mathcal{A}}\mu_a p(a|x)p(a) p(y|x,a)}{\sum_{a\in\mathcal{A}}\mu_a p(a|x)p(a)}\\
         & = p(y|x,a),\, \forall a, x : p(x|a)>0,
    \end{array}
\end{equation*}
and again the resulting expressions for $r^{BS}_a(\bm{\mu})$, $r^{CE}_a(\bm{\mu})$ are immediate from direct substitution.
\end{proof}

We now present two auxiliary lemmas that will help us prove Theorem \ref{theo:Starshaped}.

\begin{lemma}
Let $\pra$ be a Pareto front, and let $\bm{r}(\bm{\mu}) \in \pra$ denote the solution the linear weighting Problem given by Eq.(\ref{eq:linearweightproblem}). Then $\forall \bm{\mu}, \in \mathbb{R}_+^\mathcal{A}$, $\mathcal{I} \subseteq \mathcal{A}, \bm{\eta}: \eta_i >0 \forall i \in \mathcal{I},  \eta_i =0\, \forall i \in \mathcal{A}\setminus \mathcal{I}$.

Then at least one risk in the $\mathcal{I}$ coordinates is reduced, or both risk vectors are the same in $\mathcal{I}$, i.e., \[\exists j \in \mathcal{I}: r_j(\bm{\mu}+\bm{\eta})<r_j(\bm{\mu}) \vee \bm{r}_{\mathcal{I}}(\bm{\mu}+\bm{\eta})=\bm{r}_{\mathcal{I}}(\bm{\mu}). \]
\label{lemma:Monotonicity}
\end{lemma}
\begin{proof}
Denote $\Phi(\bm{\mu}) = \sum_{a\in \mathcal{A}} \mu_a r_a(\bm{\mu})$,  $\Phi_{\mathcal{A}\setminus\mathcal{I}}(\bm{\mu}) = \sum_{a\in \mathcal{A}\setminus\mathcal{I}} \mu_a r_a(\bm{\mu})$, and  $\bm{r}_\mathcal{I}(\bm{\mu})=\{r_a(\bm{\mu})\}_{a \in \mathcal{I}} $.

By contradiction, we negate the thesis to get \[(\bm{r}_\mathcal{I}(\bm{\mu}+\bm{\eta})\ge \bm{r}_\mathcal{I}(\bm{\mu})) \wedge\,(\exists j \in \mathcal{I}: \bm{r}_j(\bm{\mu}+\bm{\eta})> \bm{r}_j(\bm{\mu})).\]

We discuss $2$ cases.

\textbf{Case} $\Phi_{\mathcal{A}\setminus\mathcal{I}}(\bm{\mu}) > \Phi_{\mathcal{A}\setminus\mathcal{I}}(\bm{\mu}+\bm{\eta})$:

By definition of $\bm{r}(\bm{\mu}), \bm{r}(\bm{\mu}+\bm{\eta})$ we observe
\begin{equation*}
    \begin{array}{l}
         \sum_{a\in \mathcal{A}} (\mu_a + \eta_a) r_a(\bm{\mu})\ge\sum_{a\in \mathcal{A}} (\mu_a + \eta_a) r_a(\bm{\mu}+\bm{\eta}),\\
         \sum_{a\in\mathcal{I}} (\mu_a + \eta_a) r_a(\bm{\mu}) +\Phi_{\mathcal{A}\setminus\mathcal{I}}(\bm{\mu}) \ge\\
         \qquad \qquad \ge\sum_{a\in \mathcal{I}} (\mu_a + \eta_a) r_a(\bm{\mu}+\bm{\eta}) +\Phi_{\mathcal{A}\setminus\mathcal{I}}(\bm{\mu}+\bm{\eta}),\\
         
         \underbrace{\sum_{a\in \mathcal{I}} (\mu_a + \eta_a) (r_a(\bm{\mu})-r_a(\bm{\mu}+\bm{\eta}))}_{\le0} \ge\\
         \qquad \qquad \ge \underbrace{\Phi_{\mathcal{A}\setminus\mathcal{I}}(\bm{\mu}+\bm{\eta})-\Phi_{\mathcal{A}\setminus\mathcal{I}}(\bm{\mu})}_{>0},
    \end{array}
\end{equation*}
where the inequalities in the underbrackets are a direct result of the case assumptions, these inequalities contradict the hypothesis.

\textbf{Case} $\Phi_{\mathcal{A}\setminus\mathcal{I}}(\bm{\mu}) \le \Phi_{\mathcal{A}\setminus\mathcal{I}}(\bm{\mu}+\bm{\eta})$:

We can directly observe that
\begin{equation*}
    \begin{array}{l}
         \sum_{a\in \mathcal{A}} (\mu_a + \eta_a) r_a(\bm{\mu}) \\
         \qquad\qquad =\sum_{a\in \mathcal{I}} (\mu_a + \eta_a) r_a(\bm{\mu}) + \Phi_{\mathcal{A}\setminus\mathcal{I}}(\bm{\mu})\\
         \qquad \qquad \le  \sum_{a\in \mathcal{I}} (\mu_a + \eta_a) r_a(\bm{\mu}) + \Phi_{\mathcal{A}\setminus\mathcal{I}}(\bm{\mu}+\bm{\eta})\\
         \qquad \qquad <  \sum_{a\in \mathcal{I}} (\mu_a + \eta_a) r_a(\bm{\mu}+\bm{\eta}) + \Phi_{\mathcal{A}\setminus\mathcal{I}}(\bm{\mu}+\bm{\eta}),\\
         
    \end{array}
\end{equation*}
which again contradicts the hypothesis.
\end{proof}

\begin{lemma}
Let $\pra$ be a Pareto front, and let $\bm{r}(\bm{\mu}) \in \pra$ denote the solution the linear weighting Problem given by Eq.(\ref{eq:linearweightproblem}). Let $\mathcal{I} \subseteq \mathcal{A}$ and let $\bm{\mu}^{\mathcal{I}}: \bm{\mu}^{\mathcal{I}}_\mathcal{I}>0, \bm{\mu}^{\mathcal{I}}_{\mathcal{A}\setminus\mathcal{I}}=0, \normone{\bm{\mu}^{\mathcal{I}}}=1$. Then

\[(\exists i \in \mathcal{I}: r_i(\bm{\mu}^{\mathcal{I}}) < \norminf{\bm{r}(\bm{\mu})}) \vee (\bm{\mu} \in \bm{\mu}^*).\]
\label{lemma:Edges}
\end{lemma}

\begin{proof}
By contradiction, assume
\[(\bm{r}_{i}(\bm{\mu}^{\mathcal{I}})\ge \norminf{\bm{r}(\bm{\mu})},\,\forall i \in \mathcal{I}) \wedge (\bm{\mu} \not\in \bm{\mu}^*).\]
We can then write 
\begin{equation*}
    \begin{array}{rl}
   \sum_{a\in \mathcal{A}} \mu^\mathcal{I}_a r_a(\bm{\mu}^\mathcal{I}) &\ge \sum_{a\in \mathcal{A}} \mu^\mathcal{I}_a \norminf{\bm{r}(\bm{\mu})}\\
   &> \sum_{a\in \mathcal{A}} \mu^\mathcal{I}_a \norminf{\bm{r}(\bm{\mu}^*)}\\
   &\ge \sum_{a\in \mathcal{A}} \mu^\mathcal{I}_a r_a(\bm{\mu}^*),
    \end{array}
\end{equation*}

which contradicts the definition of $\bm{r}(\bm{\mu}^\mathcal{I})$.
\end{proof}

\paragraph{Theorem \ref{theo:Starshaped}.}
Let $\pra$ be a Pareto front, and $\bm{r}(\bm{\mu}) \in \pra$ denote the solution to the linear weighting Problem \ref{eq:linearweightproblem}. For any $\bm{\mu'} \not\in \argmin\limits_{\bm{\mu} \in \Delta^{|\mathcal{A}|-1}}||\bm{r}(\bm{\mu})||_{\infty}$, and $\bm{\mu}^*\in \argmin\limits_{\bm{\mu} \in \Delta^{|\mathcal{A}|-1}}||\bm{r}(\bm{\mu})||_{\infty}$, the sets \[N_{i} = \{\bm{\mu}:r_{i}(\bm{\mu})<|| \bm{r}(\bm{\mu'})||_{\infty}\}\] satisfy:
\vspace{-0.2em}
\begin{enumerate}
    \item $\bm{\mu}^* \in \bigcap\limits_{i \in \mathcal{A}}N_i$;
    \item If $\bm{\mu} \in N_i \rightarrow \lambda\bm{\mu}+(1-\lambda)\bm{e^i} \in N_i$, $\forall\, \lambda \in [0,1], i=1,\dots,|\mathcal{A}|$, where $\bm{e^i}$ denotes the standard basis vector;
    \item $\forall\, \mathcal{I} \subseteq \mathcal{A}$, $ \bm{\mu}: \mu_{\mathcal{A}\backslash\mathcal{I}} = 0 \rightarrow \bm{\mu} \in \bigcup\limits_{i \in \mathcal{I}}N_i$;
    \item If $\bm{r}(\bm{\mu})$ is also continuous in $\bm{\mu}$, then $\forall\, \mathcal{I} \subseteq \mathcal{A}$ such that $\bm{\mu} \in \bigcap\limits_{i \in \mathcal{I}}N_i \rightarrow \exists \epsilon>0: B_{\epsilon}(\bm{\mu}) \subset  \bigcap\limits_{i \in \mathcal{I}}N_i$;
    \item If $\pra$ is also convex, then $\bm{r}(\bm{\mu}^*) \in \argmin\limits_{\bm{r}\in\pra}||\bm{r}||_{\infty}$.
\end{enumerate}

\begin{proof}$ $

\textbf{Property 1:}
We can directly observe that 
\[r_a(\bm{\mu}^*)\le \norminf{\bm{r}(\bm{\mu}^*)} <\norminf{\bm{r}(\bm{\mu}')},\, \forall a \in \mathcal{A},\] 
and therefore $\bm{\mu}^* \in \bigcap\limits_{i \in \mathcal{A}}N_i$.

\textbf{Property 2:}
Direct application of Lemma \ref{lemma:Monotonicity}.

\textbf{Property 3:}
Direct application of Lemma \ref{lemma:Edges}.

\textbf{Property 4:}
For all $i \in \mathcal{I}$ we have $r_i(\bm{\mu})<\norminf{\bm{r}(\bm{\mu}')}$, since $r_i(\bm{\mu})$ is also continuous in $\bm{\mu}$, $\exists \,\epsilon_i >0 : \,\forall \bm{\mu}'' \in B_{\epsilon_i}(\bm{\mu}) \rightarrow r_i(\bm{\mu})<\norminf{\bm{r}(\bm{\mu}')} $.
Taking $\epsilon =\min\limits_{i\in \mathcal{I}} \epsilon_i$ we have $B_{\epsilon}(\bm{\mu}) \subset  \bigcap\limits_{i \in \mathcal{I}}N_i$.

\textbf{Property 5:} 
Immediate since every point in the Pareto front in this condition can be expressed as a solution to the linear weighting Problem \ref{eq:linearweightproblem}. Proven in Theorem \ref{theo:linearweightconvex} and \cite{geoffrion1968proper}.

\end{proof}

\subsection{Analysis of Proposed Optimization Method}
\label{sec:AppendixOptimization}
Here we discuss some of the properties of the APStar Algorithm (Algorithm \ref{alg:main_algo}). Key observations regarding the update sequence can be summarized as follows:

\begin{itemize}
    \item Updates $\bm{\mu}^{t+1} =  (\bm{\mu}^{t} +\frac{1}{K\normone{\bm{1}_{\mu^t}}}\bm{1}_{\mu^t})\frac{K}{K+1}$ always satisfy $\bm{\mu}^{t} \ge 0, \normone{\bm{\mu}^{t}}=1$,
    \item Consecutive updates that do not decrease the minimax risk have a step size that converges to $0$: $\normtwo{\bm{\mu}^{t+1}-\bm{\mu}^{t}} = \frac{1}{K+1} \normtwo{\bm{\mu}^{t} - \bm{1}_{\mu^t}} \le \frac{1}{K+1} \rightarrow 0$,
    \item The update sequence $\bm{\mu}^{t+1} =  (\bm{\mu}^{t} +\bm{\eta})\frac{K}{K+1}; \, K \leftarrow K+1$ converges to $\bm{\eta}$. That is $\bm{\mu}^{t+1}\rightarrow \bm{\eta}$.
\end{itemize}

So far, we showed that the choice of update sequence always proposes feasible weighting vectors, with progressively smaller step sizes, but that nonetheless can reach any point in the feasible region given sufficient updates.

We can also state that the update directions $\frac{\bm{1}_\mu}{\normone{\bm{1}_\mu}}$ are not fixed points of the algorithm unless they themselves are a viable update vector.
\begin{lemma}
Let $\bm{\mu}'$ with corresponding $\norminf{r(\bm{\mu}')}$ in the conditions of Theorem \ref{theo:Starshaped}. Denote the possible update directions of Algorithm \ref{alg:main_algo} as 
\begin{equation*}
    \bm{\eta}^{\mathcal{I}}: \eta_i =
        \begin{cases} \frac{1}{|\mathcal{I}|}, i\in \mathcal{I},\\ 0 \; \text{o.w.}\end{cases}
\end{equation*}

Similarly, denote 
\begin{equation*}
    \bm{1}^{\bm{\mu}}: \bm{1}^{\bm{\mu}}_i =
        \begin{cases} 1 \; \text{if}\; r_i(\bm{\mu})> \norminf{r(\bm{\mu}')}, \\ 0 \; \text{o.w.}\end{cases}
\end{equation*}

All update directions that are non-viable descent updates are repellors. That is, if $\mathcal{I} \subseteq \mathcal{A} : \norminf{r(\bm{\eta}^{\mathcal{I}})}> \norminf{r(\bm{\mu}')}$ then $\exists \epsilon: \forall \bm{\mu} \in B_{\epsilon}(\bm{\eta}^{\mathcal{I}}); \frac{\bm{1}_\mu}{\normone{\bm{1}_\mu}}\neq \bm{\eta}^{\mathcal{I}}$.
\label{lemma:NoFixedPoints}
\end{lemma}
\begin{proof}
This is a direct corollary of properties $3$ and $4$ of Theorem \ref{theo:Starshaped}
\end{proof}

A full convergence proof of the algorithm would need to show that this algorithm has no cycles. This can be motivated to a point by observing that the update directions are equivalent to performing gradient descent on the following function:

\begin{equation}
    \begin{array}{cc}
        F(\bm{\mu}) &= \sum\limits_{i=1}^{|\mathcal{A}|} (1-\mu_a) \bm{1}(r_i(\bm{\mu}) \ge \bar{\bm{r}}),  \\
         \nabla F(\bm{\mu}) &= -\{(r_i(\bm{\mu}) \ge \bar{\bm{r}})\}_{i=1}^{|\mathcal{A}|}= -\bm{1}_{\bm{\mu}},
    \end{array}
    \label{eq:AlgoGradient}
\end{equation}

where we abuse the notation $\nabla F(\bm{\mu})$ since the function is not differentiable on the boundaries where $r_i(\bm{\mu})= \bar{\bm{r}}$. We do note however that even on those points, $\bm{1}_{\bm{\mu}}$ is a valid descent direction. Function $F(\bm{\mu})$ can be shown to have no non-global minima in the set $\bm{\mu}:\bm{\mu}>0, \normone{\bm{\mu}}=1$, and global minina on all $\bm{\mu} \in \bigcap\limits_{i \in \mathcal{A}}N_i$. We do note, however, that gradient descent on discontinous functions can still produce cycles on pathological cases.

Figure \ref{fig:simplex2x2convergenceAppendix} (repeating Figure \ref{fig:simplex2x2convergence} for convenience) shows examples of iterations of the APStar algorithm across randomly generated star-convex sets satisfying the conditions of Theorem \ref{theo:Starshaped}. We observe that the algorithm converges to a viable update direction in all instances. Convergence for simple cases happens in few iterations, but challenging scenarios can still be appropriately solved with the proposed algorithm; on average, our algorithm is significantly faster than random sampling, especially in challenging scenarios. An explanation on how these distributions are sampled is presented next.

\begin{figure}[ht]
 \centering
 \includegraphics[width=1\linewidth]{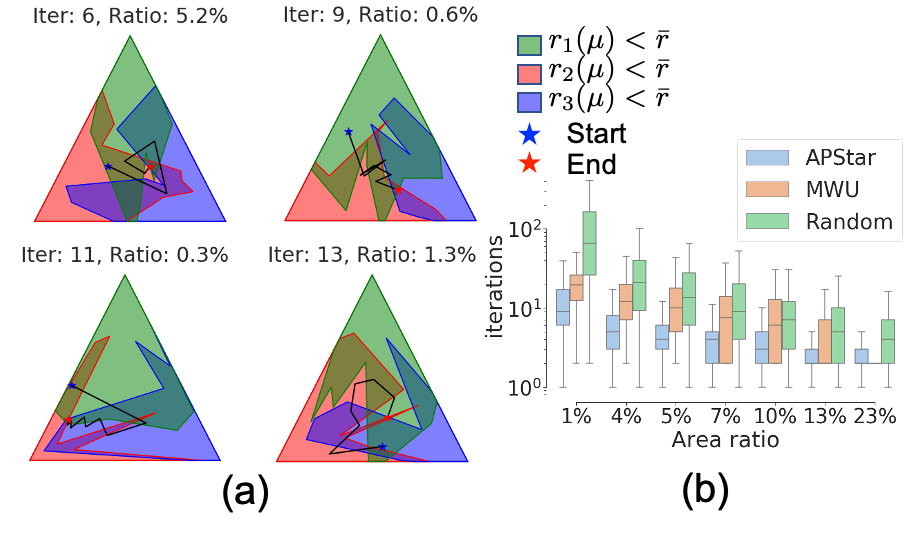}
 \caption{\footnotesize Synthetic data experiment on star-shaped sets. (a) Randomly sampled star sets satisfying the conditions of Theorem \ref{theo:Starshaped}; a random starting point is sampled (Blue), and the trajectories recovered by the APStar algorithm are recorded until convergence (Red); number of iterations and intersection area are shown for all examples. (b) Empirical distribution of the number of iterations required to converge as a function of the percentage of linear weights that lie in the triple intersection; values are shown for the APStar algorithm, random sampling, and the multiplicative weight update (MWU) algorithm proposed in \cite{chen2017robust} for minimax optimization. The number of iterations required by the algorithm is well below both samplers, this is especially apparent for low area ratio scenarios. APStar finds a viable weight in all scenarios, with simpler sets and initial conditions requiring a smaller number of iterations on average.}
\label{fig:simplex2x2convergenceAppendix}
\end{figure}
\subsubsection*{Sampling Star-convex Distributions}
\label{sec:StarshapedSampling}
Here we describe a simple procedure to sample Star-convex distributions in $\Delta^2$ that satisfy the properties of Theorem \ref{theo:Starshaped}.

We note that any point in the $\Delta^2$ simplex can be embedded into $\mathbbm{R}^2$ by using the following function,

\begin{equation*}
    \begin{array}{rl}
         f(\bm{\mu}):&\Delta^2 \rightarrow \mathbbm{R}^2,   \\
         f(\bm{\mu})=& (\frac{2\mu_1 +\mu_2}{2},\frac{\sqrt{3}\mu_2}{2}).
    \end{array}
\end{equation*}

In this restricted $\mathbbm{R}^2$ space, we note parametric curves of the form $\mathcal{C}_i: [0,\frac{\pi}{3}]\rightarrow \mathbbm{R}^{+}$ can be used to parametrize the Star-shaped sets we require for Theorem \ref{theo:Starshaped}. Namely, for any Star-shaped set $N_i \in \Delta^2$ centered on $\bm{e}^i$, we can find a function $\mathcal{C}_i: [0,\frac{\pi}{3}]\rightarrow \mathbbm{R}^{+}$ such that 

\begin{equation*}
\begin{array}{c}
     N_i=\{\bm{\mu} \in \Delta^2: d_i = \langle f(\bm{\mu})-f(\bm{e}^i) , f(\bm{e}^{i+1})-f(\bm{e}^i)\rangle,\\
     \quad \qquad \qquad \theta_i =\angle (f(\bm{\mu})-f(\bm{e}^i), f(\bm{e}^{i+1})-f(\bm{e}^i),\\
     \qquad d_i < \mathcal{C}_i(\theta_i)\}\bigcup \{f(\bm{e^i})\}.
\end{array}
\end{equation*}

Figure \ref{fig:simplexSampler} illustrates the relationship between curves $\mathcal{C}_i$ and their corresponding Star-shaped sets $N_i$.

\begin{figure}[ht!]
 \centering
 \includegraphics[width=.61\linewidth]{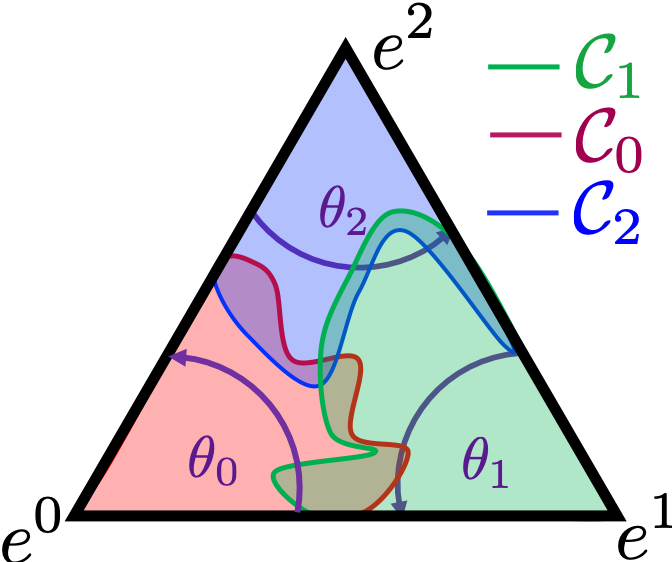}
 \caption{\footnotesize Illustration of Star-shaped sets $N_i \in \Delta^2$. The sets $N_i$ can be parametrized by the boundary curves $\mathcal{C}_i$, all points that connect the boundary $\mathcal{C}_i$ with $f(\bm{e}^i)$ conform the Star-shaped set $N_i$.}
\label{fig:simplexSampler}
\end{figure}

To create a curve $\mathcal{C}_i$, we construct a piecewise-linear function by sampling $K$ tuples $(\theta_i^j,r_i^j)_{j=0}^K$ satisfying

\begin{equation*}
    \begin{array}{cc}
         0=\theta_i^0< \dots < \theta_i^K=\pi/3, \\
         r_i^j \in [0,1]\;\forall j=0,\dots,K.
    \end{array}
\end{equation*}

In our convergence experiments, we set $K=7$, $r_i^j \sim U[0,1]$ and $(\theta_i^j)_{j=1}^{K-1} =  \text{Sort}((u\sim U[0,\pi/3])_{j=1}^{K-1})$. We sample $\mathcal{C}_0$, $\mathcal{C}_1$, $\mathcal{C}_2$ independently and then reject these functions if the corresponding sets $N_0, N_1, N_2$ do not satisfy the properties of Theorem \ref{theo:Starshaped}. Namely, properties $2$ and $4$ are satisfied by construction, we check that Property $1$ is satisfied by verifying that $N_0\cup N_1 \cup N_2$ contains at least one element.

Property 3 is verified by checking that $N_0\cap N_1 \cap N_2$ covers the entire triangle in $\mathcal{R}^2$, and also that
\begin{equation*}
    \begin{array}{c}
\mathcal{C}_0(0) + \mathcal{C}_{1}(\pi/3)>1,\\
\mathcal{C}_1(0) + \mathcal{C}_{2}(\pi/3)>1,\\
\mathcal{C}_2(0) + \mathcal{C}_{0}(\pi/3)>1.\\
    \end{array}
\end{equation*}

\subsection{MMPF Implementation Details}
\label{sec:ImplementationDetails}

Here we present implementation details to estimate the Minimax Pareto Fair classifier from data. As mentioned in Section \ref{sec:OptimizationMethods}, the APStar algorithm (Algorithm \ref{alg:main_algo}) requires an optimizer to solve the linear weighting problem (Problem \ref{eq:linearweightproblem}). We propose two options, one minimizes it using stochastic gradient descent (SGD), we call this approach joint estimation, and it is described in Algorithm \ref{alg:jointestimation}. Note that each batch samples the sensitive attributes uniformly in order to reduce the variance of the conditional risk estimators for every group in every batch. 
The second approach was also presented in Section \ref{sec:OptimizationMethods} and is called plug-in. Here each conditional distribution ($p(Y|X,A), p(A|X)$) is estimated independently from the data, and the optimal model for a given weighting vector $\bm{\mu}$ is computed using the expression derived in Theorem \ref{theo:briercrosspareto}, Algorithm \ref{alg:pluginestimation} describes this approach.

In both the joint and plug-in approach, the APStar algorithm makes decision based on the risk vector evaluated on the validation datasets. This is done to empirically improve generalization and disallow excessive overfitting.

\begin{algorithm}[tb!]
   \caption{Joint Estimation}
   \label{alg:jointestimation}
\begin{algorithmic}
   \STATE {\bfseries Input:} Train $D^{tr}=\{(x^{i},y^{i},a^{i})\}_{i=1}^{N_{tr}}$, Validation $D^{val}=\{(x^{i},y^{i},a^{i})\}_{i=1}^{N_{val}}$, Network: $h_{\theta^o}$, Weights: $\bm{\mu}$, Loss $\ell(\cdot,\cdot)$, Learning Rate: $\text{lr}$, Decay rate = $\gamma$, Epochs: $n_E$, Batch Size: $n_B$, Maximum Patience = $n_P$
   \STATE  $h^*\leftarrow h_{\theta^o}$, epochs$ \leftarrow 0$, patience $ \leftarrow 0$
   \REPEAT
   \STATE $t \leftarrow 0$, $h_{\theta} \leftarrow h^*$
   \WHILE{$t < \frac{N_{tr}}{n_B}$} 
       \STATE $\{a_i\}_{i=1}^{n_B} \sim U[1,..,|\mathcal{A}|]$; $\{x_i,y_i|a_i\}_{i=1}^{n_B} \sim D^{tr}$
       \STATE $\bm{r}(h_{\theta}) \leftarrow \Big\{\frac{\sum_{i=1}^{n_B}\mathbf{1}[a_i = a]\ell(h_{\theta}(x_i),y_i)}{ \sum_{i=1}^{n_B}\mathbf{1}[a_i = a]}  \Big\}_{a \in \mathcal{A}}$
       \STATE $\theta \leftarrow \theta - lr \nabla_{\theta}\langle\bm{\mu},\bm{r}(h_{\theta})\rangle$ 
       \STATE $t \leftarrow t+1$ 
   \ENDWHILE
   \STATE epochs += 1 \textit{$\#$ epoch ended; evaluate on validation}
   \STATE $\bm{r}^{val}(h_{\theta}) \leftarrow \Big\{\frac{\sum_{i \in D^{val}}\mathbf{1}[a_i = a]\ell(h_{\theta}(x_i),y_i)}{ \sum_{i \in D^{val}}\mathbf{1}[a_i = a]}  \Big\}_{a \in \mathcal{A}}$
   \IF{$\langle\bm{\mu},\bm{r}^{val}(h_{\theta})\rangle \le \langle\bm{\mu},\bm{r}^{val}(h^*)\rangle$ }
   \STATE $h^* \leftarrow h_{\theta}$; patience $ \leftarrow 0$
   \ELSE
   \STATE $lr \leftarrow \gamma lr$; patience += 1
   \ENDIF
  \UNTIL{epochs$\ge n_E$ $\vee$ patience $\ge n_P$}
  \STATE \textbf{return}  $h^*,\bm{r}^{val}(h^*)$
\end{algorithmic}
\end{algorithm}

\begin{algorithm}[tb!]
   \caption{Plug-in Estimation}
   \label{alg:pluginestimation}
\begin{algorithmic}
   \STATE {\bfseries Input:}  Validation $D^{val}=\{(x^{i},y^{i},a^{i})\}_{i=1}^{N_{val}}$,  Networks estimating $p(Y|X,a)$ : $\{p^Y_{\theta_a}\}_{a = 1}^{|\mathcal{A}|}$ and $p(A|X)$: $p_{\phi}^A$, Priors $\{p_a\}_{a=1}^{|\mathcal{A}|}$, Weights: $\bm{\mu}$, Loss $\ell(\cdot,\cdot)$.
   \STATE $h^*(x) \leftarrow \frac{\sum_{a=1}^{|\mathcal{A}|}p_{\theta_a}^Y(x)p_{\phi}^A(x)\frac{{\mu}_a}{p_a}}{\sum_{a=1}^{|\mathcal{A}|}p_{\phi}^A(x)\frac{{\mu}_a}{p_a}}$
   \STATE $\bm{r}^{val}(h^*) \leftarrow \Big\{\frac{\sum_{i \in D^{val}}\mathbf{1}[a_i = a]\ell(h^*(x_i),y_i)}{ \sum_{i \in D^{val}}\mathbf{1}[a_i = a]}  \Big\}_{a \in \mathcal{A}}$
  \STATE \textbf{return}  $h^*,\bm{r}^{val}(h^*)$
\end{algorithmic}
\end{algorithm}

\subsection{Synthetic Data Experiments}
\label{sec:SyntheticAppendix}

We tested our approach on synthetic data where the observations are drawn from the following distributions:
\begin{equation}
    \begin{array}{ll}
        A \sim U[1,...,|\mathcal{A}|],\\
         X|A=a \sim N(m_a,1),\\
         Y|X=x,A=a \sim Ber\big(f_a(x)\big), \\
         f_a(x) =  \rho^{l}_a \mathbf{1}[x\le t_a] + \rho^{h}_a \mathbf{1}[x> t_a] . \\
         
    \end{array}
\end{equation}
Note that $f_a(x)$ is a piecewise-constant function. We used Brier Score as our loss function, then the Bayes-optimal classifier for each group is $f_a(x)$, while the optimal classifier for the linear weighting problem can be computed numerically using the expression derived in Theorem \ref{theo:briercrosspareto}.

For the synthetic experiments presented in this section we chose $|A| =3$, $\{m_0,m_1,m_2\}$ = $\{-0.5,0,0.5\}$, $\{t_0,t_1,t_2\}$ = $\{-0.25,0,0.25\}$, $\rho^{l}_{0,1,2} = 0.1$, $\rho^{h}_{0,1} = 0.9$, and $\rho^{h}_2 = 0.8$. Figure \ref{fig:SyntheticPareto}.a shows the conditional distributions $p(X|a)$ and $p(Y|X,a)$ obtained with these parameters, along with the minimax Pareto fair classifier.

We evaluate performance of the APStar algorithm performs when $h^{\bm{\mu}}, \bm{r}(\bm{\mu})$ are computed using the closed form formula in Theorem \ref{theo:briercrosspareto}, expectations are computed via numerical integration. This enables us to evaluate the performance of the algorithm in the infinite sample and perfect $h$ optimization regime. We recover the optimal minimax Pareto fair weights $\bm{\mu}^*$ via grid search, since a closed form solution for these weights is not available. Figure \ref{fig:SyntheticPareto}.b and Figure \ref{fig:SyntheticPareto}.c show how the risk vector $\bm{r}(\bm{\mu})$ and linear weights $\bm{\mu}$ approach the minimax optimal $\bm{r}^*,\bm{\mu}^*$ as a function of iterations of the APStar algorithm. Figure \ref{fig:AlgorithmSynthetic} shown in Section \ref{sec:OptimizationMethods} was also generated with these parameters.

Figures \ref{fig:SyntheticPareto}.d and \ref{fig:SyntheticPareto}.e show the performance of the algorithm as a function of training samples, the classifier is implemented using an NN, and is minimized using SGD. Table  \ref{table:NetworkArchitectures} in Section  \ref{sec:AppendixArchitectures} provides the architecture and optimization details. We observe that the optimal classifier is non-linear (see Figure \ref{fig:SyntheticPareto}.a) which motivates the use of NNs for estimation. Note that relative errors decrease with the number of samples. In both cases, the algorithm is able to effectively converge to the minimax Pareto fair risk.

The Pareto curve shown in figure \ref{fig:ParetoCurve2d} (Section \ref{sec:ProblemStatement}) was generated with parameters $|A|=2$, $\{m_0,m_1\}$ = $\{-0.5,0.5\}$, $\{t_0,t_1\}$ = $\{-0.25,0.25\}$, $\{\rho^{l}_{0},\rho^{l}_{1} \}= \{0.3,0.05\}$, $\{\rho^{h}_{0},\rho^{h}_{1} \}= \{0.7,0.95\}$ and Brier Score risk.

\begin{figure}[ht]
 \centering
 \includegraphics[width=1\linewidth]{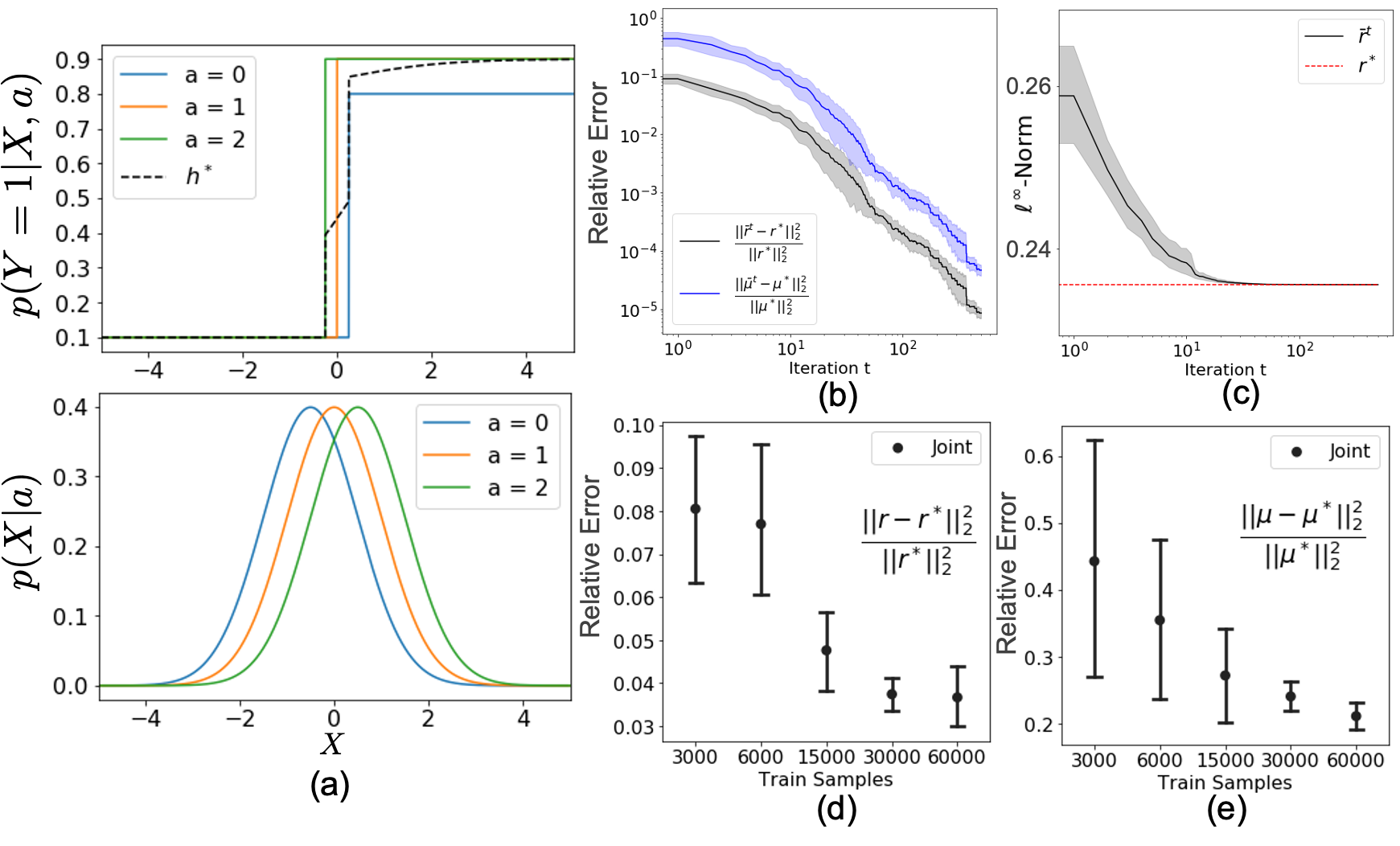}
 \caption{\footnotesize Synthetic data experiment. (a) Conditional distributions $p(X|a)$ and $p(Y=1|X,a)$, the minimax Pareto fair classifier $h^*$ is also shown. (b) The APStar algorithm converges to the optimal risk and weight vectors in a scenario where access to ground truth joint distributions is provided. (c) shows how the minimax risk is reduced in this scenario as well. (d) and (e) show minimax convergence of risk and weight vectors respectively as a function of samples when the classifier $h^{\bm{\mu}}$ is estimated with a fully connected neural network. Relative error quickly decays when more samples are acquired.}
\label{fig:SyntheticPareto}
\end{figure}

\subsection{Plug in vs Joint Estimation Analisis}
\label{sec:PlugvJointAppendix}

We empirically compare the performance of the plug-in and joint estimation approaches presented in Section \ref{sec:AnalysisPOSols}. The main advantage of plug-in estimation is that once the conditional classifiers are calculated, evaluating a new weight vector does not require any optimization; in contrast, joint estimation requires a full optimization run for each new weight vector. The main advantage of joint estimation is that it requires a single model, and makes use of all samples to train it; this can be beneficial when samples are scarce, and can be motivated by the transfer learning literature. In our problem setting, we can consider sensitive groups as different domains or tasks, and our goal is to find a model that has the best minimax performance. If the conditional distributions $p(Y|X,a)$ and $p(X|a)$ match for every $a \in \mathcal{A}$ ($Y \perp A|X$ and $X \perp A$), we are in the optimal transfer learning scenario where all groups benefit from each others' samples to estimate the target. In this situation, the joint approach would be expected to perform better on the test set than the plug-in approach. From this ideal case we can identify the following two deviations:
\begin{itemize}
    \item Case I: $Y\perp A|X$ and $d(p(X|a),p(X|a'))\ge\epsilon$, $\forall a,a' \in \mathcal{A}, \epsilon>0$;
    \item Case II: $X\perp A$ and $d(p(Y|X,a),p(Y|X,a'))\ge\epsilon$, $\forall a,a' \in \mathcal{A}, \epsilon>0$;
\end{itemize}
where $d(.,.)$ is some distance or divergence between distributions.

Case I keeps $Y\perp A|X$, but $p(X|a)$ differ across $a$ values. It is reasonable to assume that as $\epsilon$ increases, the difference in the group conditional risks on test data for joint and plug-in estimation will be low. Note that in this scenario, if the hypothesis class is unbounded, there is no trade-off between group risks as shown in Lemma \ref{lemma:limitingCases}, hence the Pareto front is the Utopia point; we expect the joint estimation approach to outperform plug-in estimation at the same number of weight updates.

Case II keeps $X \perp A$, but $p(Y|X,a)$ differ across $a$ values. In this scenario, there may be a trade-offs between group risks. We argue that in a finite sample scenario, as $\epsilon$ increases, it is not clear if joint estimation will be better or worse than plug-in since the former may be affected by negative transfer.  

\subsubsection*{Experiments}
We empirically examine these cases by choosing $|\mathcal{A}| = 2$ and simulating synthetic data from the following distribution:
\begin{equation}
    \label{eq:synthetic1}
    \begin{array}{ll}
         X|0 \sim \mathcal{N}(0,1),& X|1\sim\mathcal{N}(m_1,1),m_1\ge0,\\
         Y|X,0 \sim Ber(f_{0}(X)),& Y|X,1\sim Ber(f_{1}(X)),\\ 
         A \sim Ber(\frac{1}{2}), \\
    \end{array}
\end{equation}
where
\begin{equation}
\label{eq:synthetic2}
    \begin{array}{l}
        f_{0}(X) =(0.6 + 0.2 \mathbf{1}[x \ge 0]) sign[sin(2 \pi X)+1]\\
        \hspace{0.6in} + 0.2 - 0.1 \mathbf{1}[x \ge 0],\\
        f_{1}(X) = (1-\lambda)f_{0}(X) + \lambda[1-round(f_{0}(X))], \\
        \hspace{0.6in}\lambda \in [0,1].
    \end{array}
\end{equation}
Note that parameter $m_1$ is used to change the separation between $p(X|0)$ and $p(X|1)$; we measure this using $KL$-divergence $D_{KL}(p(X|0),p(X|1)) = \frac{m^2_{1}}{2}$. 
The parameter $\lambda$ controls the difference between $p(Y|X,1)$ and $p(Y|X,0)$, we measure this using $E_{X}[D_{KL}(p(Y|X,0),p(Y|X,1))]$, which can be numerically approximated. For Case I we choose $\lambda = 0$ and $m_0 \in \{0,0.5,1,1.5,2\}$; for Case II we choose $m_0=0$ and  $\lambda \in \{0,0.2,0.5,0.8\}$. 

Figure \ref{fig:simulationplug} shows three examples of synthetic data generated with these distributions were the input variable $X$, sensitive variable $A$, and target $Y$ exhibit various dependencies. Figures \ref{fig:simulationplug}.a and \ref{fig:simulationplug}.b show examples of conditional distributions $p(X|A)$ and $p(Y=1|X,A)$ for Cases I and II. Figure \ref{fig:simulationplug}.c shows an example were both $X \not\perp A$ and $Y\not\perp A$. Additionally, we plot the MMPF classifier $h^*(X)$. When $Y \perp A|X$ we have that $h^*(X) = p(Y|X)$ and when $X \perp A$, $h^*(X)$ is a linear combination of $p(Y|X,0)$ and $p(Y|X,1)$ with the minimax weight coefficients. These are direct consequences of Theorem \ref{theo:briercrosspareto}.  In Figure \ref{fig:simulationplug}.c $h^*(X)$ shows a more complex interaction since both $p(X|a)$ and $p(Y|X,a)$ differ between $a$ values. Asymptotically, since $p(a=1|X=\infty)=1, p(a=0|X=-\infty)=1 $, we recover $h^*(\infty) = p(Y|X,1)$ and $h^*(-\infty) = p(Y|X,0)$.

\begin{figure}[ht!]
 \centering
 \includegraphics[width=1\linewidth]{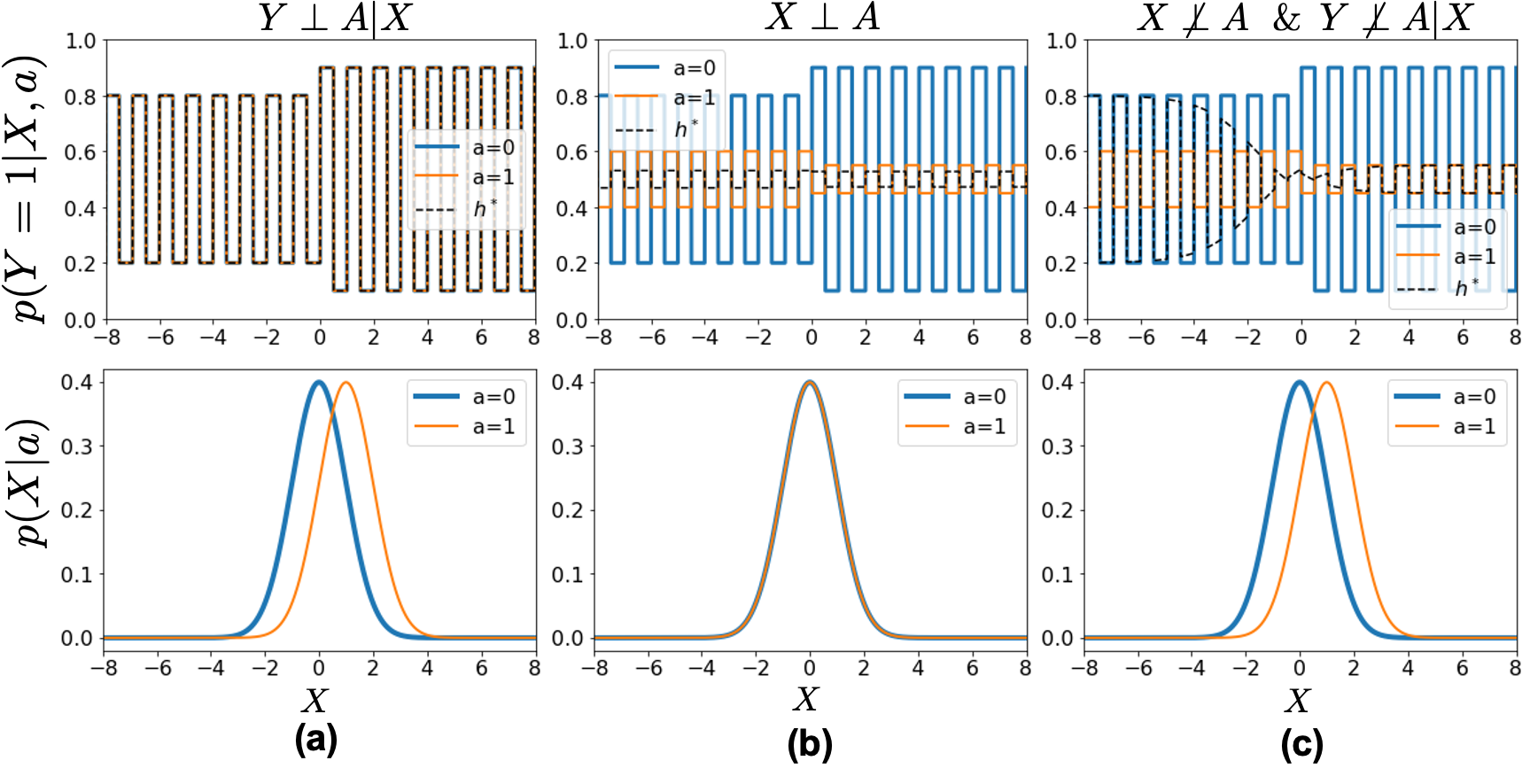}
 \caption{\footnotesize (a), (b), and (c) show conditional distributions for $p(Y=1|X,a)$ (top) and $p(X|a)$ (bottom), simulated according to Eq.\ref{eq:synthetic1} and Eq.\ref{eq:synthetic2} for three situations. (a) corresponds to case I ($Y\perp A|X$), (b) to case II ($X\perp A$) and (c) to $Y\not\perp A|X$  and $X\not\perp A$. All plots of $p(Y=1|X,a)$ include the optimal minimax Pareto fair classifier $h^*(X)$.} 
\label{fig:simulationplug}
\end{figure}

Figure \ref{fig:resultsplug} compares the performance of the plug-in and joint estimation approach. On both approaches, we limit the number of weight updates of the APStar algorithm to $15$. Figure \ref{fig:resultsplug}.a compares relative differences in the risk vector as a function of the divergence between $p(X|0)$ and $p(X|1)$ for $4k$ and $9k$ train samples and Case I ($Y\perp A)$. Figure \ref{fig:resultsplug}.b compares accuracies under the same conditions. We observe that in this scenario, the benefit of joint estimation is evident, especially with a small number of samples. Note that this gap is reduced as $D_{KL}(p(X|0)||p(X|1))$ and the number of samples increases; as expected, both methods perform better when more samples are available.

Figures \ref{fig:resultsplug}.c and \ref{fig:resultsplug}.d show the same comparisons for Case II. In this scenario there is no clear difference between plug-in and joint estimation, though the former appears to be marginally better; both methods improve performance with additional samples.

\begin{figure}[ht!]
 \centering
 \includegraphics[width=1\linewidth]{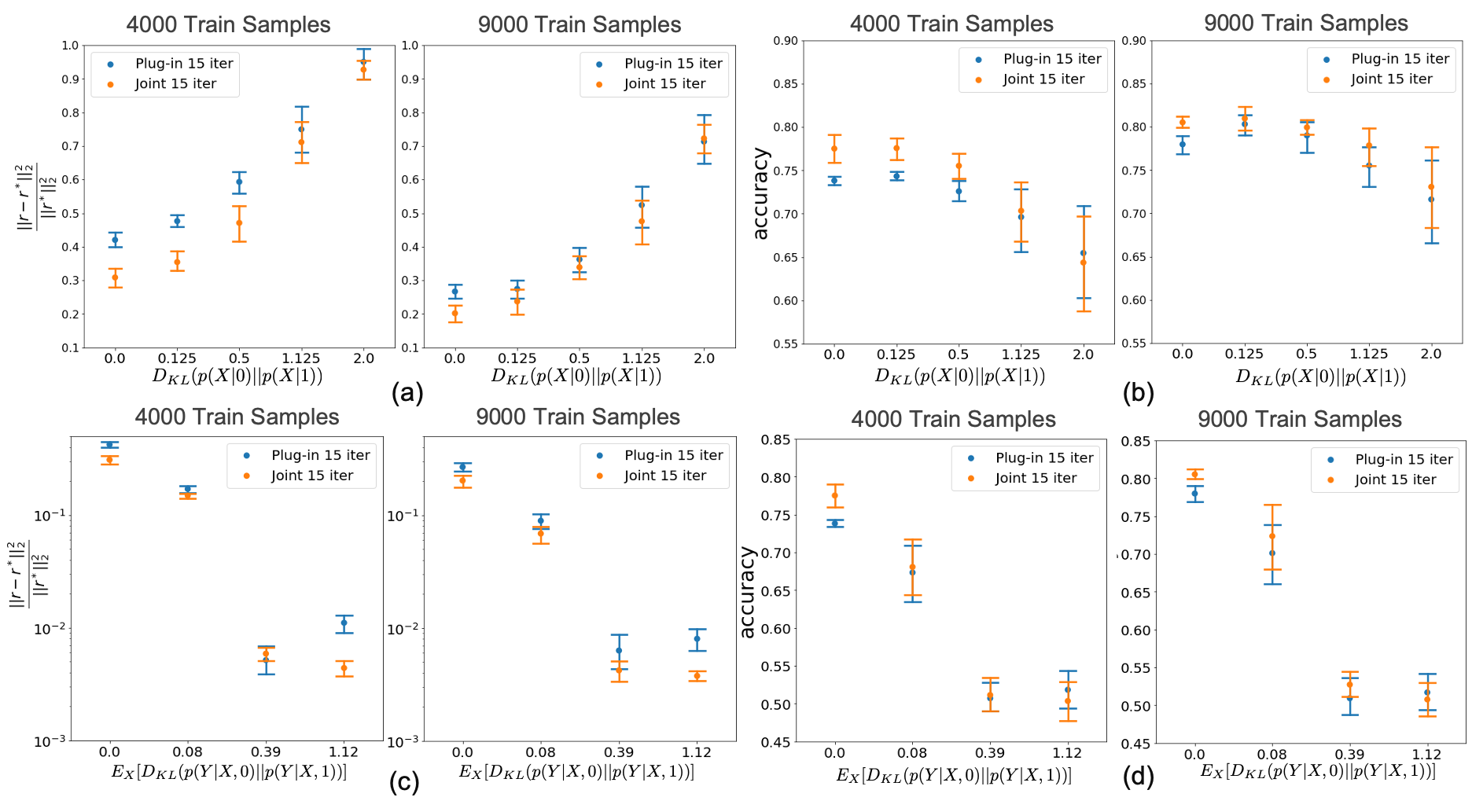}
 \caption{\footnotesize  (a) Relative error of the estimated Minimax Pareto Fair risk as a function of divergence between $p(X|0)$ and $p(X|1)$ for Case I ($Y\perp A|X$) at $4k$ and $9k$ training samples. (b) shows the corresponding accuracy comparison. (c) and (d) mirror (a) and (b) for Case II ($X\perp A$). }
\label{fig:resultsplug}
\end{figure}

\subsubsection*{Unbalanced Classification: $Y=A$}
In balanced risk minimization, we have $Y=A$. If the risk function is either Cross Entropy or Brier Score, we can apply Theorem \ref{theo:linearweightconvex} and recover 
\[ h^{\bm{\mu}}(x)  = \bigg\{\frac{\frac{\mu_a}{p(a)} p(a|x)}{\sum_{a'\in\mathcal{A}}\frac{\mu_a'}{p(a')} p(a'|x)}\bigg\}_{a=1}^{|\mathcal{A}|}.\]

This particular scenario is noteworthy because the plug-in approach only needs to estimate $p(a|x) = p(y|x)$, all Pareto optimal classifiers can be easily derived from $p(y|x)$ by simply re-weighting each component of the output probability vector. This re-weighting requires no expensive minimization procedure, and enables extensive iterations of the APStar algorithm to find the optimal weight vector $\bm{\mu}$. For these types of scenarios, it is advantageous to estimate $p(y|x)$ by using a Naive or Balanced classifier, and then derive all optimal classifiers via a simple weighting of the output vector, the optimal weights can still be found using our APStar algorithm.

\subsubsection*{Observation summary}

We summarize our observations from these experiments and discussions.
\begin{itemize}
    \item Joint estimation may benefit from transfer learning and seems to be no worse than plug-in estimation even when target conditional distributions do not match. Note that a certain amount of negative transfer may be required by the minimax Pareto Fair classifier, which may negate the advantages of plug-in estimation when target conditional distributions differ.
    
    \item Plug-in estimation requires multiple ($|\mathcal{A}|+1$) models, while joint estimation only requires one; this makes this approach impractical in some scenarios. However, the former approach allows for cheap iterations of the APStar algorithm.
    
    \item In the balanced classification problem ($Y=A$), plug-in estimation requires the same number of models as joint estimation, but is cheaper to evaluate.
\end{itemize}

\label{sec:RepurposeAppendix}
\subsection{Methods}
\label{sec:MethodsAppendix}
We compare the performance of the following methods:

\paragraph{Kamishima.} \cite{kamishima2012fairness} uses logistic regression as a baseline classifier, and requires numerical input (observations), and binary target variable. Fairness is controlled via a regularization term with a tuning parameter $\eta$ that controls the trade-off between fairness and overall accuracy. $\eta$ is optimized via grid search with $\eta \in (0,300)$ as in the original paper. We report results on the hyperparameter configuration that produces the best minimax cross-entropy across sensitive groups.
\paragraph{Feldman.} \cite{feldman2015certifying} provides a preprocessing algorithm to sanitize input observations. It modifies each input attribute so that the marginal distribution of each coordinate is independent of the sensitive attribute. The degree to which these marginal distributions match is controlled by a $\lambda$ parameter between 0 and 1. It can handle numerical and categorical observations, as well as non-binary sensitive attributes, and arbitrary target variables. Following \cite{friedler2019comparative}, we train a linear logistic regressor on top of the sanitized attributes. $\lambda$ is optimized via grid search with increments of $0.05$. We report results on the hyperparameter configuration that produces the best minimax cross-entropy across sensitive groups.
\paragraph{Zafar.} \cite{zafar2015fairness} Addresses disparate mistreatment via a convex relaxation. Specifically, in the implementation provided in \cite{friedler2019comparative}, they train a logistic regression classifier with a fairness constraint that minimizes the covariance between the sensitive attribute and the classifier decision boundary. This algorithm can handle categorical sensitive attributes and binary target variables, and numerical observations. The maximum admissible covariance is handled by a hyperparameter $c$, tuned by logarithmic grid search with values between $0.001$ and $1$. We report results on the hyperparameter configuration that produces the best minimax cross-entropy across sensitive groups.
\paragraph{Hardt.} \cite{hardt2016equality} proposes a post-processing algorithm that takes in an arbitrary predictor and the sensitive attribute as input, and produces a new, fair predictor that satisfies equalized odds. This algorithm can handle binary target variables, an arbitrary number of sensitive attributes, and any baseline predictor, but requires test-time access to sensitive attributes. it does not contain any tuning parameter. We apply this method on top of both the Naive Classifier and our Pareto Fair classifier.
\paragraph{Naive Classifier (Naive).} Standard classifier, trained to minimize an expected risk $h =  \argmin\limits_{h \in \mathcal{H}} E_{X,A,Y}[\ell(h(X),\oh{Y})]$. The baseline classifier class $\mathcal{H}$ is implemented as a neural network and varies by experiment as described in Section \ref{sec:AppendixArchitectures}, the loss function also varies by experiment and is also described in Section \ref{sec:AppendixArchitectures}. Optimization is done via stochastic gradient descent.
\paragraph{Balanced Classifier (Balanced).} Baseline classifier designed to address undersampling of minority classes, trained to mimimize a class-rebalanced expected risk $h =  \argmin\limits_{h \in \mathcal{H}} E_{A\sim U[1,\dots, |\mathcal{A}|], (X,Y)\sim P(X,Y|A)}[\ell(h(X),\oh{Y}]$. Like the Naive classifier, it is implemented as a neural network and optimized via stochastic gradient descent. The sole difference with the Naive classifier is that, during training, samples are drawn from the new input distribution $A\sim U[1,\dots, |\mathcal{A}|]; \; X,Y|A \sim P(X,Y|A)$, which is achieved by re-weighted sampling of the original training dataset.

\paragraph{Minimax Pareto Fair (MMPF, joint and plug-in).} Our proposed methodology, finds a Pareto-optimal model $h^*$ such that it has minimax Pareto risk $\textbf{r}^*$. It achieves this by searching for weighting coefficients $\mathbf{\mu}^*$ such that $\textbf{r}^*$ is the solution to the corresponding linear weighted problem (see Eq.\ref{eq:linearweightproblem}). This method alternates between minimizing a linearly-weighted loss function (Eq.\ref{eq:linearweightproblem}), and updating the weighting coefficients according to Algorithm \ref{alg:main_algo}.
The baseline classifier class $\mathcal{H}$ is implemented as a neural network and varies by experiment as described in Section \ref{sec:AppendixArchitectures}, the loss function also varies by experiment and is also described in Section \ref{sec:AppendixArchitectures}. The minimization of the weighted loss function at every step of the APStar algorithm is implemented via stochastic gradient descent (joint estimation) or by direct application of the closed form optimal classifier derived in Theorem \ref{theo:briercrosspareto} (plug-in estimation). The latter approach requires estimating $p(a|x)$, $p(y|x,a)$ and $p(a)$.

\subsection{Evaluation Metrics}
\label{sec:Metrics}
Here we describe the metrics used to evaluate the performance of all tested methods. We are given a set of test samples $\mathcal{D}_t = \{(x_i,y_i)\}_{i=1}^N$ where $x_i \in \mathcal{X}$ is a realization of our model input and $y_i \in \mathcal{Y}$ the corresponding objective. We assume that $\mathcal{Y}$ is a finite alphabet, as in a classification problem; we will represent the one-hot encoding of $y_i$ as $\oh{y_i}$. Given a trained model $h: \mathcal{X} \rightarrow [0,1]^{|\mathcal{Y}|}$ the predicted output for an input $x_i$ is a vector $\bm{h}(x_i) : \normone{\bm{h}(x_i)}=1$ (e.g., output of a softmax layer). The predicted class is $\hat{y_i}=\argmax_j h_j(x_i)$ and its associated confidence is $\hat{p_i}=\max_{j} h_j(x_i)$. Ideally $\hat{y_i}$ should be the same as $y_i$. Using these definitions, we compute the following metrics.

\textbf{Accuracy (Acc):} $\frac{1}{N}\sum_{i=1}^N \textbf{1}(y_i = \hat{y}_i)$. Fraction of correct classifications in dataset.


\textbf{Brier Score (BS):}  $\frac{1}{N}\sum_{i=1}^N ||\oh{y_i} - \vec{p}_i||^2$ where $\oh{y_i}$ is the one-hot representation of the categorical ground truth value $y_i$. This quantity is also known as Mean Square Error (MSE).

\textbf{Cross Entropy (CE):} $-\frac{1}{N}\sum_{i=1}^N\langle\oh{y_i},\log\bm{h}(x_i)\rangle$ also known as negative log-likelihood (NLL) of the multinomial distribution.

\textbf{Expected Calibration Error (ECE):} 
$\frac{1}{N}\sum_{m=1}^M \big|  \sum_{i \in B_m} [\textbf{1}(y_i = \hat{y}_i) - \hat{p_i}] \big|$ where $M$ is the number of bins to divide the interval $[0,1]$ such that $B_m = \{i \in \{1,..,N\} : \hat{p_i} \in (\frac{m-1}{M}, \frac{m}{M}] \}$ are the group of samples that our model assigns a confidence ($\hat{p_i}$) in the interval $(\frac{m-1}{M}, \frac{m}{M}]$. Measures how closely the predicted probabilities match the true base rates.

\textbf{Maximum Calibration Error (MCE):} $max_{m \in \{1,...,M\}}\big| \frac{1}{|B_m|} \sum_{i \in B_m} [\textbf{1}(y_i = \hat{y}_i) - \hat{p_i}] \big|$. Measures worst-case miscalibration errors.

{\color{black} These metrics are computed independently for each sensitive subgroup on the test set and reported in Section \ref{sec:SupplementaryResults}.}

\subsection{Details on Experiments on Real Data}
\label{sec:realDatasets}

The following is a description of the data and experiments for each of the real datasets. The information present here is summarized in Table \ref{table:Datasets}.

\begin{table}[ht]
{
\caption{\footnotesize Basic characteristics of real datasets}
\label{table:Datasets}
\scriptsize
\begin{tabular}{@{}lccccc@{}}
\toprule
\makecell{Dataset}& \makecell{Objective} & \makecell{Sensitive \\ Attribute} & \makecell{Train/\\Val/\\Test} & Splits \\ \midrule
\makecell{Adult\\ \cite{Dua:2019}} & \makecell{2 categories:\\Income} & \makecell{4 categories:\\Gender (F/M),\\ Ethnicity(W/NW)} & 60/20/20 &\makecell{5} \\
\makecell{German\\ \cite{Dua:20192}} & \makecell{2 categories:\\Credit} & \makecell{2 categories:\\Gender (F/M)} & 60/20/20 &\makecell{5} \\
\makecell{MIMIC-III\\ \cite{johnson2016mimic}} & \makecell{2 categories:\\Mortality (A/D)} & \makecell{8 categories:\\Mortality(A/D), \\Age (A/S)\\ Ethnicity (W/NW)} &  60/20/20 & \makecell{5} \\
\makecell{HAM10000 \\ \cite{tschandl2018ham10000}} & \makecell{7 categories:\\Type of lesion} & \makecell{7 categories:\\Type of lesion } &  60/20/20 & \makecell{5} \\ \bottomrule
\end{tabular}}
\end{table}

\paragraph{MIMIC-III.} This dataset consist of clinical records collected from adult ICU patients at the Beth Israel Deaconess Medical Center (MIMIC-III dataset) \cite{johnson2016mimic}. The goal is predicting patient mortality from clinical notes. We follow the pre-processing methodology outlined in \cite{chen2018my}, where we analyze clinical notes acquired during the first 48 hours of ICU admission; discharge notes were excluded, as where ICU stays under 48 hours. Tf-idf statistics on the $10,000$ most frequent words in clinical notes are taken as input features.

We identify 8 sensitive groups as the combination of age (under/over 55 years old), ethnicity as determined by the majority group (white/nonwhite); and outcome (alive/deceased). Here we will use the term adult to refer to people under 55 years old and senior otherwise. This dataset shows large sample disparities since 56.7\% corresponds to the overall majority group (alive-senior-white) and only 0.4\% to the overall minority group (deceased-adult-nonwhite).

We used a fully connected neural network as described in Table \ref{table:NetworkArchitectures} as the baseline classifier for our proposed MMPF framework. We compare our results against both the Naive and Naive Balanced algorithms using the same neural network architecture, and use crossentropy (CE) as our training loss. We also evaluate the performance of Zafar, Feldman and Kamishima applied on the feature embeddings learned by the Naive classifier (results over the original input features failed to converge on the provided implementations).

We report the performance across a 5-fold split of the data, we used a 60/20/20 train-validation-test partition as described on Table \ref{table:Datasets} and report results over the test set. We denote the overall sensitive attribute as the combination of outcome (A:alive/D:deceased), age (A:adult/S:senior) and ethnicity (W:white, NW:nonwhite) with shorthand notation of the form D/A/W to denote, for example, deceased, white adult. We also note that results on Zafar, Kamishima and Hardt were done over only the sensitive attributes Adult/Senior and White/Nonwhite, outcome was not considered as a sensitive attribute for both methods. This was done because Hardt requires test-time access to sensitive attributes, which would not be possible for the outcome variable, and Zafar attempts to decorrelate sensitive attributes and classification decision boundaries, which is counterproductive when the sensitive attribute includes the correct decision outcome.

\paragraph{HAM10000.} This dataset contains over $10,000$ dermatoscopic images of skin lesions over a diverse population \cite{tschandl2018ham10000}. Lesions are classified in 7 diagnostic categories, and the goal is to learn a model capable of identifying the category from the lesion image. The dataset is highly unbalanced since 67\% of the samples correspond to a melanocytic nevi lesion (nv), and 1.1\% to dermatofibroma (df). 

Here we chose to use the diagnosis class as both the target and sensitive variable, casting balanced risk minimization as a particular use-case for the proposed Pareto fairness framework. 

We load a pre-trained DenseNet121 network \cite{huang2017densely} and train it to classify skin lesions from dermatoscopic images using our Pareto fairness framework. We compared against the Naive and the Balanced training setup. Note that in the Balanced approach we use a batch sampler where images from each class have the same probability, this can be seen as a naive oversampling technique. Table \ref{table:NetworkArchitectures} shows implementation details.

We report the performance across 5-fold split of the data, we used a  60/20/20 train-validation-test partition, and  report results over the test set. For each group we follow the original notation: Actinic keratoses and intraepithelial carcinoma / Bowen's disease (akiec), basal cell carcinoma (bcc), benign keratosis-like lesions (bkl), dermatofibroma (df), melanoma (mel), melanocytic nevi (nv) and vascular lesions (vasc).

\paragraph{Adult.} The Adult UCI dataset \cite{Dua:2019} is based on the 1994 U.S. Census and contains data on $32,561$ adults. The data contains 105 binarized observations representing education status, age, ethnicity, gender, and marital status, and a target variable indicating income status (binary attribute representing over or under $\$50,000$). Following \cite{friedler2019comparative}, we take ethnicity and gender as our target sensitive attributes, defining four subgroups (White/Other and Male/Female). We also present results considering just the gender as sensitive attribute (Male/Female). To compare our MMPF framework evenly against the other methods, we limit our hypothesis class to linear logistic regression (MMPF LR). Additionally, we also show results for a Neural Network model (MMPF); a model class that satisfies the convexity property.

\paragraph{German.} The German credit dataset \cite{Dua:2019} contains 20 observations collected across $1000$ individuals, and a binary target variable assessing the individual's credit score as good or bad. We consider gender (Male/Female) as the sensitive attribute, which is not included in the data but can be inferred. As in the Adult dataset, we limit our hypothesis class to linear logistic regression to compare evenly across methodologies.


\subsection{Neural Network Architectures and Parameters}
\label{sec:AppendixArchitectures}

\begin{table}[h!]
{
\scriptsize
\caption{\footnotesize Summary of network architectures and losses. All networks have a softmax output layer. ADAM was used as the training optimizer, with the specified learning rates (lr) and batch size $n_B$. Logistic Regression was trained using the implementation provided in Sklearn \cite{scikit-learn}. }

\label{table:NetworkArchitectures}
\begin{tabular}{@{}lllll@{}}
\toprule
\makecell{Dataset} & \makecell{Network\\ Body} & Gate & \makecell{Loss\\ type} & \makecell{Parameters\\ training} \\ \midrule
\makecell{Synthetic} & \makecell{Dense ResNet \\ (512x512)x2} & ELU & BS & \makecell{$n_B$=512\\ lr=1e-3} \\
\makecell{Adult \\ German} & \makecell{Logistic \\ Regression (LR)} & - & CE &\makecell{-} \\
\makecell{Adult} & \makecell{Dense ResNet \\ (512x512)x2} & - & CE &\makecell{$n_B$=32\\ lr=5e-4} \\
\makecell{MIMIC-III} & \makecell{FullyConnected \\ 2048x2048} & ELU & CE/BS & \makecell{$n_B$=512\\ lr=1e-6/5e-6} \\
\makecell{HAM10000} & \makecell{DenseNet121 \\ \cite{huang2017densely}} & ReLU &  BS & \makecell{$n_B$=32\\ lr=5e-6} \\ \bottomrule
\end{tabular}}
\end{table}

Table \ref{table:NetworkArchitectures} summarizes network architectures and loss functions for all experiments in this paper (Section \ref{sec:ExperimentsAndResults} and supplementary material). Note that all networks have a standard dense softmax as their final layer. The training optimizer is ADAM \cite{kingma2014adam}, loss functions were either crossentropy (CE) or Brier Score (BS), also known as categorical mean square error (MSE).

For joint estimation,  the weights $\bm{\mu}$ were initialized uniformly, we selected maximum patience $n_P= 20$, a decay rate $\gamma = 0.25$ and a maximum of $500$ epochs. All experiments terminated from lack of generalization improvement rather than maximum number of epochs. For the APStar algorithm we picked $\alpha = 0.5$, we allowed a maximum number of 20 iterations. The regularization parameter in the Sklearn implementation of logistic regression was set to $C=1e6$ following \cite{friedler2019comparative}.

In the plug-in approach, each conditional probability was estimated with the architectures and parameters specified in Table \ref{table:NetworkArchitectures} maximum patience $n_P= 20$, a decay rate $\gamma = 0.25$ and a maximum of $500$ epochs. Here we allowed the APStar algorithm to have a maximum of $500$ iterations since each iteration does not require optimization, only risk evaluation.

\subsection{Supplementary Results}
\label{sec:SupplementaryResults}
Here we present expanded results on all datasets. Accuracy (Acc), Brier Score (BS), Cross-Entropy (CE), Expected Calibration Error (ECE) and Maximum Calibration Error (MCE) are displayed per sensitive group. Mean and standard deviations are reported when avaialble across 5 splits. Disparity between best and worst groups is computed per split, and the mean and standard deviation of this value is reported. Note that this way of computing disparity may lead to seemingly large disparity values, since the worst and best performing group per split may differ.


\begin{table*}
\caption{HAM10000 dataset. We underline the worst group metric per method, and bold the one with the best minimax performance. Smallest disparity is also bolded.}
\label{table:HAMOverallTablefull}
\centering
\small

Acc comparison\\
\begin{tabular}{lllllllll}
\toprule
Group  &  akiec &  bcc &  bkl &  df &  mel &  nv &  vasc  & Disparity \\
\midrule
Ratio  &  3.3\% &  5.1\% &  11\% &  1.1\% &  11.1\% &  67\% &  1.4\% & 65.9\%  \\
Naive  &  39.1$\pm$5.7 \%&  58.4$\pm$7.0 \%&  51.5$\pm$4.3 \%&  \underline{2.6$\pm$3.5\%} &  43.7$\pm$4.3 \%&  93.7$\pm$1.1 \%&  66.9$\pm$6.6 \%&  91.1$\pm$3.9\% \\
Balanced &67.9$\pm$6.9 \% &  73.4$\pm$4.2 \% &  58.2$\pm$9.9 \% &  75.7$\pm$7.6 \% &  \underline{58.1$\pm$5.8 \%} &  73.5$\pm$1.6 \% &  83.9$\pm$8.3 \% &  32.5$\pm$4.6\%  \\
MMPF P & 64.2$\pm$5.1 \% &  66.2$\pm$5.8 \% &  \textbf{\underline{63.9$\pm$6.8} }\% &  69.6$\pm$13.5\%  &  67.5$\pm$3.9\%  &  64.1$\pm$1.0\%  &  71.2$\pm$1.0\%  &  \textbf{19.8$\pm$6.6\%}  \\
\bottomrule
\end{tabular}

BS comparison\\
\begin{tabular}{lllllllll}
\toprule
Group  &  akiec &  bcc &  bkl &  df &  mel &  nv &  vasc  & Disparity \\
\midrule
Ratio  &  3.3\% &  5.1\% &  11\% &  1.1\% &  11.1\% &  67\% &  1.4\% & 65.9\%  \\
Naive  &  0.816$\pm$0.082 &  0.586$\pm$0.083 &  0.675$\pm$0.068 &  \underline{1.384$\pm$0.043} &  0.808$\pm$0.043 &  0.093$\pm$0.015 &  0.48$\pm$0.102 &  1.291$\pm$0.042   \\
Balanced &  0.459$\pm$0.089 &  0.37$\pm$0.048 &  \underline{0.579$\pm$0.1} &  0.392$\pm$0.106 &  0.565$\pm$0.069 &  0.361$\pm$0.02 &  0.211$\pm$0.106 &  0.45$\pm$0.066 \\
MMPF P &   0.494$\pm$0.078 &  0.463$\pm$0.065 &  0.49$\pm$0.074 &  0.447$\pm$0.131 &  0.447$\pm$0.042 &  \textbf{\underline{0.5$\pm$0.015}} &  0.38$\pm$0.127 &  \textbf{0.228$\pm$0.058}  \\
\bottomrule
\end{tabular}

CE comparison\\
\begin{tabular}{lllllllll}
\toprule
Group  &  akiec &  bcc &  bkl &  df &  mel &  nv &  vasc  & Disparity \\
\midrule

Ratio  &  3.3\% &  5.1\% &  11\% &  1.1\% &  11.1\% &  67\% &  1.4\% & 65.9\%  \\
 Naive  &  1.924$\pm$0.321 &  1.19$\pm$0.188 &  1.405$\pm$0.142 &  \underline{4.178$\pm$0.209} &  1.589$\pm$0.104 &  0.195$\pm$0.029 &  1.069$\pm$0.282 &  3.983$\pm$0.19  \\
Balanced & 0.944$\pm$0.172 &  0.715$\pm$0.096 &  \underline{1.199$\pm$0.211} &  0.898$\pm$0.295 &  1.122$\pm$0.128 &  0.787$\pm$0.038 &  0.456$\pm$0.281 &  0.95$\pm$0.218  \\
MMPF P & 1.011$\pm$0.177 &  0.913$\pm$0.158 &  0.949$\pm$0.135 &  1.047$\pm$0.401 &  0.85$\pm$0.068 &  \textbf{\underline{1.128$\pm$0.033}} &  0.804$\pm$0.242 &  \textbf{0.615$\pm$0.156} \\
\bottomrule
\end{tabular}

ECE comparison\\
\begin{tabular}{lllllllll}
\toprule
Group  &  akiec &  bcc &  bkl &  df &  mel &  nv &  vasc  & Disparity \\
\midrule
Ratio  &  3.3\% &  5.1\% &  11\% &  1.1\% &  11.1\% &  67\% &  1.4\% & 65.9\%  \\
 Naive  & 0.211$\pm$0.019 &  0.125$\pm$0.026 &  0.189$\pm$0.032 &  \underline{0.598$\pm$0.047} &  0.287$\pm$0.03 &  0.03$\pm$0.001 &  0.156$\pm$0.049 &  0.568$\pm$0.047 \\
Balanced & 0.139$\pm$0.046 &  0.078$\pm$0.013 &  0.135$\pm$0.063 & \underline{\textbf{0.183$\pm$0.035}} &  0.119$\pm$0.035 &  0.066$\pm$0.014 &  0.119$\pm$0.028 &  \textbf{0.143$\pm$0.023}  \\
MMPF P &  0.133$\pm$0.028 &  0.113$\pm$0.029 &  0.107$\pm$0.036 &  \underline{0.213$\pm$0.049} &  0.082$\pm$0.023 &  0.135$\pm$0.01 &  0.133$\pm$0.035 &  0.151$\pm$0.034 \\

\bottomrule
\end{tabular}

MCE comparison\\
\begin{tabular}{lllllllll}
\toprule
Group  &  akiec &  bcc &  bkl &  df &  mel &  nv &  vasc  & Disparity \\
\midrule
Ratio  &  3.3\% &  5.1\% &  11\% &  1.1\% &  11.1\% &  67\% &  1.4\% & 65.9\%  \\
 Naive  & 0.616$\pm$0.262 &  0.353$\pm$0.124 &  0.534$\pm$0.151 &  \underline{0.962$\pm$0.018} &  0.555$\pm$0.091 &  0.315$\pm$0.219 &  0.521$\pm$0.156 &  0.744$\pm$0.082  \\
Balanced &  0.505$\pm$0.157 &  0.383$\pm$0.177 &  0.49$\pm$0.22 &  0.548$\pm$0.145 &  0.272$\pm$0.024 &  0.227$\pm$0.06 &  \underline{0.636$\pm$0.126} &  0.474$\pm$0.089 \\
MMPF P & 0.369$\pm$0.134 &  0.362$\pm$0.156 &  0.362$\pm$0.172 &  \textbf{\underline{0.59$\pm$0.098}} &  0.266$\pm$0.032 &  0.285$\pm$0.037 &  0.556$\pm$0.107 &  \textbf{0.444$\pm$0.058}  \\

\bottomrule
\end{tabular}

\end{table*}


\begin{table*}
\caption{MIMIC dataset. We underline the worst group metric per method, and bold the one with the best minimax performance. Smallest disparity is also bolded. Standard deviations are computed across 5 splits.}
\label{table:MIMICOverallTablefull}
\centering
\scriptsize

Acc comparison\\
\begin{tabular}{llllllllll}

\toprule
Group &  A/A/NW &  A/A/W &  A/S/NW &  A/S/W & D/A/NW &  D/A/W &  D/S/NW &  D/S/W & Disparity \\
\midrule
Ratio  & 5.7\% &  13.3\% &  12.9\% &  56.7\% &  0.4\% & 0.9\% & 1.8\% & 8.3\% &  56.3\% \\

Naive CE &  98.7$\pm$0.8\% &  98.8$\pm$0.5\% &  97.6$\pm$0.5\% &  98.0$\pm$0.3\% &  26.0$\pm$10.0\% &  34.5$\pm$3.9\% &  \underline{20.6$\pm$2.1\%} &  22.6$\pm$2.6\% &  79.4$\pm$1.7\% \\
Naive BS  &  99.0$\pm$0.4\% &  98.8$\pm$0.4\% &  97.8$\pm$0.5\% &  98.1$\pm$0.3\% & 26.7$\pm$9.3\% &  34.6$\pm$4.4\% &  \underline{19.0$\pm$2.0\%} &  21.2$\pm$1.9\% &  80.5$\pm$1.3\% \\
Balanced CE  &  85.8$\pm$1.9\% &  86.2$\pm$1.1\% &  76.4$\pm$1.8\% &  79.2$\pm$0.4\% & 76.1$\pm$8.5\% &  80.1$\pm$3.3\% & \underline{66.9$\pm$2.4\%} &  67.3$\pm$2.1\% &  22.4$\pm$2.8\% \\
Balanced BS &  87.9$\pm$1.2\% &  87.5$\pm$1.4\% &  77.5$\pm$2.1\% &  79.3$\pm$0.5\% & 74.4$\pm$7.1\% &  78.8$\pm$3.4\% &  \underline{66.8$\pm$2.2\%} &  68.0$\pm$2.1\% &  22.6$\pm$2.3\% \\
Zafar &  91.9$\pm$1.5\% &  93.9$\pm$1.4\% &  91.6$\pm$0.8\% &  93.2$\pm$0.5\% & 49.2$\pm$9.4\% &  41.7$\pm$9.2\% &  33.2$\pm$4.1\% &  \underline{32.0$\pm$2.4\%} &  62.9$\pm$3.6\% \\
Feldman &  97.4$\pm$1.9\% &  97.7$\pm$1.9\% &  95.2$\pm$3.4\% &  95.6$\pm$3.3\% &  33.0$\pm$16.4\% &  35.4$\pm$8.0\% &  \underline{28.7$\pm$2.4\%} &  31.9$\pm$3.6\% &  72.1$\pm$5.5\% \\
Kamishima &  98.5$\pm$1.2\% &  98.4$\pm$0.7\% &  96.8$\pm$0.9\% &  96.7$\pm$0.7\% & 26.7$\pm$9.8\% &  37.5$\pm$5.0\% &  \underline{25.1$\pm$5.1\%} &  29.6$\pm$2.8\% &  76.4$\pm$5.2\% \\
MMPF CE  &  83.3$\pm$2.1\% &  83.1$\pm$1.0\% &  \underline{71.3$\pm$1.4\%} &  74.0$\pm$1.2\% & 81.6$\pm$6.5\% &  83.2$\pm$3.9\% &  73.2$\pm$3.0\% &  74.7$\pm$2.9\% &  \textbf{16.2$\pm$2.8\%} \\

MMPF BS &  86.0$\pm$0.9\% &  84.8$\pm$1.3\% &  \textbf{\underline{72.6$\pm$1.7\%}} &  74.1$\pm$0.5\% & 78.9$\pm$9.5\% &  81.7$\pm$3.8\% &  73.1$\pm$2.0\% &  75.2$\pm$2.1\% &  17.1$\pm$3.5\% \\

MMPF CE P  &  82.4$\pm$3.2\% &  81.8$\pm$1.2\% &  \underline{70.7$\pm$1.7\%} &  74.4$\pm$0.8\% &  80.6$\pm$7.7\% &  84.5$\pm$3.0\% &  73.0$\pm$3.5\% &  72.4$\pm$3.1\% &  17.4$\pm$2.5\% \\
MMPF BS P &  81.8$\pm$5.6\% &  81.8$\pm$2.6\% &  \underline{70.7$\pm$2.1\%} &  74.9$\pm$1.0\% &  78.0$\pm$8.6\% &  81.9$\pm$2.5\% &  72.6$\pm$2.8\% &  72.8$\pm$3.5\% &  17.8$\pm$3.8\% \\
\midrule
Naive BS+H &  59.8$\pm$1.8\% &  59.9$\pm$1.9\% &  59.7$\pm$2.1\% &  59.8$\pm$2.1\% & 53.2$\pm$6.4\% &  51.2$\pm$3.2\% &  \underline{50.3$\pm$2.0\%} &  50.5$\pm$0.9\% &  12.3$\pm$3.8\% \\
Balanced BS+H  &  76.7$\pm$1.3\% &  77.1$\pm$1.1\% &  76.9$\pm$2.4\% &  76.7$\pm$1.4\% &  67.3$\pm$10.3\% &  66.7$\pm$2.9\% &  \underline{65.6$\pm$2.8\%} &  66.6$\pm$3.2\% &  19.1$\pm$1.8\% \\
Zafar+H &  64.0$\pm$2.2\% &  64.2$\pm$1.4\% &  64.3$\pm$2.3\% &  64.2$\pm$1.9\% & 53.4$\pm$9.2\% &  52.9$\pm$4.5\% &  \underline{51.5$\pm$2.8\%} &  52.1$\pm$1.6\% &  17.8$\pm$3.1\%  \\
Feldman+H &  61.1$\pm$2.5\% &  61.2$\pm$2.4\% &  61.2$\pm$2.4\% &  61.3$\pm$2.5\% &  55.0$\pm$11.4\% &  \underline{49.7$\pm$1.9\%} &  52.1$\pm$2.9\% &  50.8$\pm$3.1\% & 16.9$\pm$7.4\%\\
Kamishima+H &  60.0$\pm$0.8\% &  60.2$\pm$1.5\% &  60.2$\pm$1.3\% &  60.2$\pm$1.4\% & 53.9$\pm$6.7\% &  52.4$\pm$1.7\% &  \underline{51.6$\pm$2.2\%} &  52.3$\pm$1.5\% &  11.9$\pm$1.9\%\\
MMPF BS+H &  72.6$\pm$2.3\% &  72.7$\pm$1.1\% &  72.1$\pm$1.7\% &  72.5$\pm$1.3\% & 72.1$\pm$8.6\% &  72.1$\pm$2.8\% &  \textbf{72.0$\pm$3.7\%} &  72.2$\pm$2.6\% &  \textbf{11.4$\pm$3.5\%} \\
\bottomrule
\end{tabular}

BS comparison\\
\begin{tabular}{llllllllll}
\toprule
Group &  A/A/NW &  A/A/W &  A/S/NW &  A/S/W &  D/A/NW &  D/A/W &  D/S/NW &  D/S/W & Disparity \\
\midrule
Ratio  & 5.7\% &  13.3\% &  12.9\% &  56.7\% & 0.4\% & 0.9\% & 1.8\% & 8.3\% &  56.3\% \\
Naive CE &  0.029$\pm$0.004 &  0.029$\pm$0.005 &  0.054$\pm$0.007 &  0.048$\pm$0.003 &  0.995$\pm$0.141 &  0.896$\pm$0.072 &  \underline{1.086$\pm$0.014} &  1.053$\pm$0.022 &  1.064$\pm$0.019 \\
Naive BS  &  0.034$\pm$0.004 &  0.035$\pm$0.005 &  0.059$\pm$0.007 &  0.053$\pm$0.003 &  0.961$\pm$0.134 &  0.836$\pm$0.068 &  \underline{1.051$\pm$0.013} &  1.022$\pm$0.016 &  1.025$\pm$0.021 \\
Balanced CE  &  0.2$\pm$0.025 &  0.198$\pm$0.014 &  0.313$\pm$0.017 &  0.284$\pm$0.005 &  0.369$\pm$0.067 &  0.292$\pm$0.043 &  0.417$\pm$0.018 &  \underline{0.421$\pm$0.026} &  0.246$\pm$0.039 \\
Balanced BS & 0.19$\pm$0.018 &  0.188$\pm$0.013 &  0.307$\pm$0.019 &  0.283$\pm$0.004 &  0.386$\pm$0.062 & 0.31$\pm$0.039 &  \underline{0.418$\pm$0.014} &  0.413$\pm$0.026 & 0.251$\pm$0.04 \\
Zafar & 0.157$\pm$0.03 &  0.118$\pm$0.028 &  0.162$\pm$0.015 &  0.131$\pm$0.009 &  1.009$\pm$0.203 &  1.143$\pm$0.177 &  1.323$\pm$0.082 & \underline{1.344$\pm$0.05} & 1.25$\pm$0.072 \\
Feldman &  0.038$\pm$0.026 & 0.035$\pm$0.03 &  0.072$\pm$0.051 &  0.068$\pm$0.052 &  1.172$\pm$0.303 &  1.116$\pm$0.104 &  1.258$\pm$0.077 &  1.213$\pm$0.067 &  \underline{1.288$\pm$0.084} \\
Kamishima &  0.027$\pm$0.011 &  0.026$\pm$0.008 &  0.063$\pm$0.011 & 0.06$\pm$0.008 &  1.053$\pm$0.163 &  0.936$\pm$0.095 &  1.062$\pm$0.041 &  0.993$\pm$0.026 &  \underline{1.108$\pm$0.078} \\
MMPF CE  &  0.237$\pm$0.031 &  0.236$\pm$0.019 &  \underline{0.372$\pm$0.014} &  0.342$\pm$0.011 &  0.303$\pm$0.063 &  0.233$\pm$0.037 &  0.346$\pm$0.016 & 0.344$\pm$0.03 &  0.167$\pm$0.028 \\
MMPF BS &  0.212$\pm$0.015 &  0.217$\pm$0.017 &  \textbf{\underline{0.354$\pm$0.016}} &  0.337$\pm$0.006 &  0.331$\pm$0.057 &  0.254$\pm$0.035 &  0.352$\pm$0.017 &  0.338$\pm$0.028 & 0.17$\pm$0.027 \\
MMPF CE P  &  0.274$\pm$0.033 &  0.269$\pm$0.017 &  \underline{0.373$\pm$0.013} &  0.343$\pm$0.011 &  0.325$\pm$0.061 &  0.271$\pm$0.025 &  0.363$\pm$0.012 & 0.37$\pm$0.027 & \textbf{0.144$\pm$0.017} \\
MMPF BS P & 0.272$\pm$0.06 & 0.262$\pm$0.03 &  \underline{0.368$\pm$0.014} & 0.332$\pm$0.01 &  0.323$\pm$0.079 &  0.256$\pm$0.032 &  0.357$\pm$0.022 &  0.365$\pm$0.032 &  0.168$\pm$0.042 \\

\bottomrule
\end{tabular}

CE comparison\\
\begin{tabular}{llllllllll}
\toprule
Group &  A/A/NW &  A/A/W &  A/S/NW &  A/S/W &  D/A/NW &  D/A/W &  D/S/NW &  D/S/W & Disparity \\
\midrule
Ratio  & 5.7\% &  13.3\% &  12.9\% &  56.7\% & 0.4\% & 0.9\% & 1.8\% & 8.3\% &  56.3\% \\
Naive CE &  0.073$\pm$0.003 &  0.072$\pm$0.005 &  0.122$\pm$0.009 &  0.112$\pm$0.003 &  1.536$\pm$0.175 & 1.314$\pm$0.11 &  \underline{1.561$\pm$0.036} & 1.53$\pm$0.038 &  1.549$\pm$0.027 \\
Naive BS  &  0.093$\pm$0.006 &  0.093$\pm$0.009 & 0.139$\pm$0.01 &  0.129$\pm$0.004 & 1.407$\pm$0.15 &  1.179$\pm$0.103 &  \underline{1.458$\pm$0.041} &  1.425$\pm$0.029 &  1.394$\pm$0.049 \\
Balanced CE  &  0.332$\pm$0.036 & 0.324$\pm$0.02 &  0.473$\pm$0.021 &  0.435$\pm$0.008 &  0.575$\pm$0.092 &  0.458$\pm$0.051 &  0.605$\pm$0.026 &  \underline{0.613$\pm$0.037} &  0.329$\pm$0.052 \\
Balanced BS &  0.322$\pm$0.026 &  0.317$\pm$0.017 &  0.467$\pm$0.024 &  0.437$\pm$0.007 &  0.593$\pm$0.079 &  0.481$\pm$0.044 &  \underline{0.604$\pm$0.022} &  0.601$\pm$0.034 &  0.332$\pm$0.043 \\
Zafar & 1.71$\pm$0.468 &  1.199$\pm$0.298 &  1.834$\pm$0.261 &  1.363$\pm$0.127 &  14.611$\pm$4.841 &  17.197$\pm$4.31 &  21.829$\pm$2.593 &  \underline{22.653$\pm$1.918} &  21.729$\pm$2.164 \\
Feldman &  0.072$\pm$0.047 &  0.064$\pm$0.054 &  0.138$\pm$0.1 &  0.127$\pm$0.095 & 3.924$\pm$0.962 &  3.475$\pm$0.612 & 3.802$\pm$0.842 &  3.854$\pm$0.74 &  \underline{4.211$\pm$0.75} \\
Kamishima &  0.057$\pm$0.015 &  0.054$\pm$0.011 &  0.128$\pm$0.011 &  0.121$\pm$0.011 & \underline{1.838$\pm$0.176} &  1.561$\pm$0.167 & 1.585$\pm$0.077 & 1.521$\pm$0.049 & 1.808$\pm$0.155 \\
MMPF CE  &  0.378$\pm$0.041 &  0.373$\pm$0.026 &  \underline{0.547$\pm$0.019} &  0.508$\pm$0.012 &  0.501$\pm$0.101 &  0.377$\pm$0.045 &  0.517$\pm$0.023 & 0.517$\pm$0.04 &  0.232$\pm$0.037 \\
MMPF BS &  0.349$\pm$0.022 &  0.352$\pm$0.023 &  0.524$\pm$0.021 &  0.503$\pm$0.008 & \underline{0.532$\pm$0.07} & 0.407$\pm$0.04 &  0.525$\pm$0.025 &  0.509$\pm$0.037 &  0.231$\pm$0.036 \\
MMPF CE P  &  0.438$\pm$0.039 & 0.430$\pm$0.024 & 0.550$\pm$0.014 &  0.516$\pm$0.015 &  0.505$\pm$0.078 &  0.438$\pm$0.029 &  0.542$\pm$0.016 & \underline{0.551$\pm$0.03} & \textbf{0.17$\pm$0.023} \\
MMPF BS P &  0.431$\pm$0.072 &  0.416$\pm$0.038 &  0.541$\pm$0.015 &  0.498$\pm$0.015 & 0.51$\pm$0.104 &  0.412$\pm$0.042 &  0.534$\pm$0.026 &  \textbf{\underline{0.544$\pm$0.038}} &  0.212$\pm$0.053 \\

\bottomrule
\end{tabular}

ECE comparison\\
\begin{tabular}{llllllllll}
\toprule
Group &  A/A/NW &  A/A/W &  A/S/NW &  A/S/W &  D/A/NW &  D/A/W &  D/S/NW &  D/S/W & Disparity \\
\midrule
Ratio  & 5.7\% &  13.3\% &  12.9\% &  56.7\% & 0.4\% & 0.9\% & 1.8\% & 8.3\% &  56.3\% \\
Naive CE &  0.049$\pm$0.006 &  0.046$\pm$0.003 &  0.071$\pm$0.007 &  0.068$\pm$0.004 &  0.519$\pm$0.105 &  0.434$\pm$0.065 &  \underline{0.574$\pm$0.014} &  0.542$\pm$0.025 &  0.537$\pm$0.011 \\
Naive BS  &  0.068$\pm$0.003 &  0.065$\pm$0.006 &  0.088$\pm$0.004 &  0.087$\pm$0.004 &  0.502$\pm$0.106 &  0.386$\pm$0.051 &  \underline{0.576$\pm$0.014} & 0.54$\pm$0.018 &  0.519$\pm$0.005 \\
Balanced CE  &  0.066$\pm$0.015 &  0.049$\pm$0.011 &  0.039$\pm$0.013 &  0.029$\pm$0.008 &  \underline{0.173$\pm$0.048} & 0.1$\pm$0.03 &  0.073$\pm$0.029 &  0.057$\pm$0.022 &  0.148$\pm$0.042 \\
Balanced BS &  0.083$\pm$0.004 &  0.073$\pm$0.011 &  0.042$\pm$0.011 &  0.043$\pm$0.009 &  \underline{0.151$\pm$0.043} &  0.079$\pm$0.041 &  0.069$\pm$0.023 &  0.044$\pm$0.012 &  \textbf{0.126$\pm$0.044} \\
Zafar &  0.079$\pm$0.015 &  0.059$\pm$0.014 &  0.081$\pm$0.008 &  0.065$\pm$0.004 &  0.511$\pm$0.099 &  0.576$\pm$0.089 &  0.666$\pm$0.039 &  \underline{0.671$\pm$0.025} &  0.627$\pm$0.037 \\
Feldman &  0.015$\pm$0.009 &  0.012$\pm$0.011 &  0.023$\pm$0.016 &  0.019$\pm$0.015 &  0.624$\pm$0.152 &  0.565$\pm$0.071 &  0.628$\pm$0.044 &  0.587$\pm$0.034 & \underline{0.668$\pm$0.06} \\
Kamishima &  0.029$\pm$0.005 &  0.021$\pm$0.004 &  0.055$\pm$0.011 & 0.05$\pm$0.007 &  0.543$\pm$0.112 & 0.44$\pm$0.069 &  0.547$\pm$0.037 &  0.486$\pm$0.021 & \underline{0.568$\pm$0.06} \\
MMPF CE  &  0.041$\pm$0.015 &  0.029$\pm$0.015 &  0.055$\pm$0.018 & 0.03$\pm$0.008 &  \underline{0.152$\pm$0.057} & 0.08$\pm$0.021 & 0.076$\pm$0.03 &  0.039$\pm$0.013 & 0.143$\pm$0.04 \\
MMPF BS &  0.064$\pm$0.009 &  0.047$\pm$0.011 &  0.037$\pm$0.011 &  0.021$\pm$0.006 &  \underline{0.178$\pm$0.056} &  0.091$\pm$0.034 &  0.073$\pm$0.016 &  0.038$\pm$0.013 &  0.167$\pm$0.044 \\
MMPF CE P  &  0.105$\pm$0.024 &  0.086$\pm$0.016 &  0.042$\pm$0.016 &  0.049$\pm$0.015 &  \underline{0.177$\pm$0.063} &  0.152$\pm$0.016 &  0.102$\pm$0.018 &  0.059$\pm$0.027 &  0.159$\pm$0.044 \\
MMPF BS P &  0.077$\pm$0.033 & 0.052$\pm$0.02 &  0.041$\pm$0.021 &  0.025$\pm$0.007 &  \textbf{\underline{0.144$\pm$0.037}} &  0.122$\pm$0.027 &  0.065$\pm$0.017 &  0.044$\pm$0.026 &  0.127$\pm$0.027 \\

\bottomrule
\end{tabular}

MCE comparison\\
\begin{tabular}{llllllllll}
\toprule
Group &  A/A/NW &  A/A/W &  A/S/NW &  A/S/W &  D/A/NW &  D/A/W &  D/S/NW &  D/S/W & Disparity \\
\midrule
Ratio  & 5.7\% &  13.3\% &  12.9\% &  56.7\% & 0.4\% & 0.9\% & 1.8\% & 8.3\% &  56.3\% \\
Naive CE &  0.277$\pm$0.094 &  0.231$\pm$0.034 &  0.207$\pm$0.027 &  0.206$\pm$0.011 &  0.835$\pm$0.119 &  \underline{0.878$\pm$0.071} & 0.839$\pm$0.05 &  0.874$\pm$0.046 &  0.729$\pm$0.031 \\
Naive BS  &  0.315$\pm$0.075 &  0.262$\pm$0.046 &  0.247$\pm$0.039 & 0.23$\pm$0.022 &  0.806$\pm$0.122 &  \underline{0.869$\pm$0.063} &  0.822$\pm$0.057 &  0.861$\pm$0.058 &  0.707$\pm$0.028 \\
Balanced CE  &  0.138$\pm$0.024 &  0.087$\pm$0.016 &  0.074$\pm$0.025 &  0.049$\pm$0.013 &  \underline{0.383$\pm$0.072} &  0.225$\pm$0.167 &  0.191$\pm$0.077 &  0.119$\pm$0.043 &  0.393$\pm$0.062 \\

Balanced BS &  0.155$\pm$0.039 &  0.114$\pm$0.034 &  0.077$\pm$0.018 &  0.072$\pm$0.016 & \underline{0.395$\pm$0.29} & 0.14$\pm$0.045 &  0.159$\pm$0.085 &  0.109$\pm$0.047 & 0.36$\pm$0.269 \\
Zafar &  0.483$\pm$0.227 & 0.647$\pm$0.15 &  0.591$\pm$0.192 &  0.445$\pm$0.086 &  0.619$\pm$0.122 &  0.678$\pm$0.164 &  \underline{0.696$\pm$0.059} &  0.696$\pm$0.034 &  0.495$\pm$0.154 \\
Feldman &  0.272$\pm$0.094 &  0.283$\pm$0.222 &  0.245$\pm$0.108 &  0.219$\pm$0.129 &  \underline{0.864$\pm$0.068} &  0.705$\pm$0.052 & 0.736$\pm$0.05 & 0.734$\pm$0.06 &  0.738$\pm$0.142 \\
Kamishima &  0.416$\pm$0.056 &  0.158$\pm$0.026 &  0.191$\pm$0.019 &  0.154$\pm$0.025 &  \underline{0.823$\pm$0.133} &  0.747$\pm$0.162 & 0.787$\pm$0.08 & 0.785$\pm$0.04 & 0.73$\pm$0.067 \\
MMPF CE  &  0.102$\pm$0.043 &  0.059$\pm$0.022 &  0.097$\pm$0.028 & 0.068$\pm$0.02 &  \underline{0.346$\pm$0.119} & 0.24$\pm$0.156 & 0.147$\pm$0.05 &  0.073$\pm$0.027 &  0.352$\pm$0.114 \\

MMPF BS &  0.111$\pm$0.024 &  0.076$\pm$0.024 &  0.076$\pm$0.031 &  0.041$\pm$0.012 &  \underline{0.347$\pm$0.072} & 0.25$\pm$0.133 &  0.149$\pm$0.039 &  0.063$\pm$0.013 &  0.343$\pm$0.069 \\

MMPF CE P  &  0.163$\pm$0.044 &  0.128$\pm$0.014 & 0.078$\pm$0.03 &  0.077$\pm$0.014 &  \underline{0.489$\pm$0.216} &  0.249$\pm$0.053 &  0.177$\pm$0.045 &  0.126$\pm$0.035 &  0.424$\pm$0.219 \\
MMPF BS P &  0.132$\pm$0.047 &  0.091$\pm$0.028 &  0.078$\pm$0.026 &  0.044$\pm$0.012 &  \textbf{\underline{0.327$\pm$0.101}} &  0.253$\pm$0.152 &  0.12$\pm$0.03 &  0.083$\pm$0.033 &  \textbf{0.317$\pm$0.125} \\

\bottomrule
\end{tabular}

\label{table:MIMICperformance}
\end{table*}


\begin{table*}
\caption{Adult ethnicity and gender dataset. We underline the worst group metric per method, and bold the one with the best minimax performance. Smallest disparity is also bolded. Standard deviations are computed across 5 splits.}
\label{table:ADULT4OverallTablefull}
\centering
\scriptsize

Acc comparison\\
\begin{tabular}{llllll}
\toprule
type &  Female Other &  Male Other &  Female White &  Male White & disc \\
\midrule
Ratio  & 6.0\% & 7.7\% &  26.1\% &  60.3\% &  54.3\% \\
Naive LR &  94.7$\pm$0.9\% &  84.0$\pm$1.0\% &  91.8$\pm$0.4\% &  \underline{80.6$\pm$0.5\%} &  14.1$\pm$1.0\% \\
Balanced LR &  95.0$\pm$1.0\% &  84.5$\pm$0.7\% &  91.9$\pm$0.4\% &  \underline{80.5$\pm$0.5\%} &  14.5$\pm$1.0\% \\
Zafar  &  95.1$\pm$0.9\% &  84.1$\pm$1.4\% &  92.0$\pm$0.2\% &  \underline{80.6$\pm$0.5\%}&  14.5$\pm$0.9\% \\
Feldman &  95.1$\pm$1.0\% &  83.7$\pm$1.3\% &  91.8$\pm$0.4\% &  \underline{80.4$\pm$0.3\%} &  14.7$\pm$0.9\% \\

Kamishima  &  95.3$\pm$1.0\% &  83.8$\pm$0.6\% &  91.9$\pm$0.4\% &  \underline{80.0$\pm$1.2\%} &  15.2$\pm$1.8\% \\

MMPF LR &  94.6$\pm$0.7\% &  84.7$\pm$1.0\% &  91.3$\pm$0.3\% &  \underline{80.6$\pm$0.5\%} &  14.0$\pm$1.0\% \\
MMPF LR P &  94.6$\pm$0.7\% &  84.0$\pm$1.0\% &  91.4$\pm$0.6\% &  \underline{80.7$\pm$0.5\%} &  13.9$\pm$1.0\% \\
MMPF &  94.6$\pm$1.2\% &  84.4$\pm$0.9\% &  91.5$\pm$0.5\% &  \underline{80.9$\pm$0.6\%} &  13.6$\pm$1.5\% \\
MMPF P &  94.5$\pm$1.1\% &  84.3$\pm$1.5\% &  90.7$\pm$0.5\% &  \textbf{81.0$\pm$0.8\%} &  \textbf{13.4$\pm$1.5\%} \\
\midrule
Naive LR+H &  76.3$\pm$1.0\% &  72.6$\pm$2.2\% &  75.6$\pm$1.6\% &  \underline{71.7$\pm$2.0\%} &  5.6$\pm$0.8\% \\
Balanced LR+H &  76.5$\pm$1.2\% &  72.3$\pm$2.2\% &  75.7$\pm$2.0\% &  \underline{71.6$\pm$2.1\%} &  5.7$\pm$0.6\%\\
Zafar+H &  73.8$\pm$2.5\% &  70.4$\pm$2.3\% &  73.3$\pm$2.5\% &  \underline{69.7$\pm$3.0\%} &  5.0$\pm$0.7\% \\
Feldman+H &  74.8$\pm$2.7\% &  71.2$\pm$2.6\% &  74.2$\pm$2.3\% &  \underline{70.7$\pm$2.6\%} &  5.5$\pm$1.1\%\\

Kamishima+H &  72.3$\pm$3.5\% &  68.5$\pm$3.2\% &  71.4$\pm$3.6\% &  \underline{68.0$\pm$4.0\%} &  \textbf{4.8$\pm$0.7\%}\\

MMPF LR+H & 79.2$\pm$1.4\% &  75.4$\pm$1.5\% &  78.2$\pm$0.6\% &  \textbf{74.4$\pm$1.8\%} &  5.5$\pm$1.3\%  \\

\bottomrule
\end{tabular}

BS comparison\\
\begin{tabular}{llllll}
\toprule
type &  Female Other &  Male Other &  Female White &  Male White & disc \\
\midrule
Ratio  & 6.0\% & 7.7\% &  26.1\% &  60.3\% &  54.3\% \\
Naive LR &  0.079$\pm$0.009 &  0.207$\pm$0.012 &  0.119$\pm$0.004 &  \underline{0.266$\pm$0.005} &  0.187$\pm$0.011 \\
Balanced LR &  0.078$\pm$0.01 &  0.207$\pm$0.011 &  0.119$\pm$0.004 &  \underline{0.267$\pm$0.006} &  0.189$\pm$0.011 \\
Zafar  & 0.076$\pm$0.01 &  0.208$\pm$0.015 &  0.118$\pm$0.003 &  \underline{0.265$\pm$0.005} & 0.189$\pm$0.01 \\
Feldman &  0.082$\pm$0.011 &  0.213$\pm$0.014 &  0.122$\pm$0.003 &  \underline{0.269$\pm$0.004} &  0.187$\pm$0.011 \\

Kamishima  &  0.081$\pm$0.01 &  0.216$\pm$0.008 &  0.119$\pm$0.005 &  \underline{0.271$\pm$0.016} &  0.191$\pm$0.021 \\

MMPF LR &  0.084$\pm$0.009 &  0.206$\pm$0.013 &  0.126$\pm$0.004 &  \underline{0.264$\pm$0.005} &  0.18$\pm$0.01 \\
MMPF LR P &  0.085$\pm$0.007 &  0.209$\pm$0.014 &  0.127$\pm$0.006 &  \underline{0.264$\pm$0.005} &  0.18$\pm$0.007 \\
MMPF &  0.084$\pm$0.012 &  0.209$\pm$0.009 &  0.126$\pm$0.005 &  \textbf{0.261$\pm$0.005} &  0.177$\pm$0.014 \\
MMPF P &  0.086$\pm$0.013 &  0.209$\pm$0.014 &  0.137$\pm$0.007 &  \underline{0.262$\pm$0.006} &  \textbf{0.176$\pm$0.015} \\

\bottomrule
\end{tabular}

CE comparison\\
\begin{tabular}{llllll}
\toprule
type &  Female Other &  Male Other &  Female White &  Male White & disc \\
\midrule
Ratio  & 6.0\% & 7.7\% &  26.1\% &  60.3\% &  54.3\% \\
Naive LR & 0.14$\pm$0.013 &  0.321$\pm$0.017 &  0.204$\pm$0.005 &  \underline{0.408$\pm$0.008} &  0.268$\pm$0.016 \\
Balanced LR &  0.138$\pm$0.013 &  0.322$\pm$0.018 &  0.203$\pm$0.004 &  \underline{0.411$\pm$0.008} &  0.273$\pm$0.015 \\
Zafar  &  0.143$\pm$0.024 &  0.336$\pm$0.027 &  0.204$\pm$0.003 & \underline{0.409$\pm$0.01} & 0.266$\pm$0.03 \\
Feldman &  0.149$\pm$0.015 &  0.332$\pm$0.019 & 0.21$\pm$0.005 &  \underline{0.412$\pm$0.006} &  0.262$\pm$0.016 \\

Kamishima  &  0.146$\pm$0.014 &  0.337$\pm$0.013 &  0.202$\pm$0.006 &  \underline{0.414$\pm$0.023} &  0.269$\pm$0.026 \\

MMPF LR &  0.153$\pm$0.012 &  0.322$\pm$0.018 &  0.218$\pm$0.005 &  \textbf{0.404$\pm$0.007} &  \textbf{0.251$\pm$0.015} \\
MMPF LR P &  0.153$\pm$0.01 &  0.324$\pm$0.019 &  0.218$\pm$0.008 &  \underline{0.405$\pm$0.006} &  0.251$\pm$0.01 \\
MMPF &  0.141$\pm$0.017 &  0.326$\pm$0.015 &  0.219$\pm$0.009 &  \underline{0.404$\pm$0.009} &  0.263$\pm$0.022 \\
MMPF P &  0.151$\pm$0.021 &  0.331$\pm$0.029 &  0.245$\pm$0.007 & \underline{0.41$\pm$0.014} & 0.258$\pm$0.03 \\

\bottomrule
\end{tabular}

ECE comparison\\
\begin{tabular}{llllll}
\toprule
type &  Female Other &  Male Other &  Female White &  Male White & disc \\
\midrule
Ratio  & 6.0\% & 7.7\% &  26.1\% &  60.3\% &  54.3\% \\
Naive LR & 0.02$\pm$0.006 & \underline{0.03$\pm$0.003} &  0.019$\pm$0.003 &  0.014$\pm$0.004 &  0.018$\pm$0.003 \\
Balanced LR &  0.017$\pm$0.004 &  \textbf{0.025$\pm$0.005} &  0.014$\pm$0.004 & 0.02$\pm$0.004 &  \textbf{0.014$\pm$0.006} \\
Zafar  &  0.013$\pm$0.003 &  \underline{0.032$\pm$0.008} &  0.015$\pm$0.003 &  0.012$\pm$0.004 &  0.021$\pm$0.011 \\
Feldman &  0.029$\pm$0.004 &  \underline{0.034$\pm$0.009} &  0.024$\pm$0.001 &  0.011$\pm$0.003 &  0.026$\pm$0.009 \\

Kamishima  &  0.014$\pm$0.003 &  \underline{0.028$\pm$0.005} &  0.009$\pm$0.003 &  0.013$\pm$0.004 &  0.019$\pm$0.006 \\

MMPF LR &  \underline{0.036$\pm$0.007} &  0.028$\pm$0.008 &  0.036$\pm$0.006 &  0.01$\pm$0.003 &  0.031$\pm$0.005 \\
MMPF LR P &  0.029$\pm$0.007 &  \underline{0.033$\pm$0.009} &  0.024$\pm$0.009 &  0.013$\pm$0.002 &  0.025$\pm$0.003 \\
MMPF &  0.022$\pm$0.004 &  \underline{0.031$\pm$0.003} &  0.015$\pm$0.004 &  0.016$\pm$0.004 &  0.018$\pm$0.007 \\
MMPF P &  0.028$\pm$0.005 &  \underline{0.033$\pm$0.007} &  0.018$\pm$0.003 &  0.015$\pm$0.004 &  0.021$\pm$0.006 \\

\bottomrule
\end{tabular}

MCE comparison\\
\begin{tabular}{llllll}
\toprule
type &  Female Other &  Male Other &  Female White &  Male White & disc \\
\midrule

Ratio  & 6.0\% & 7.7\% &  26.1\% &  60.3\% &  54.3\% \\
Naive LR &  \underline{0.146$\pm$0.079} &  0.111$\pm$0.032 &  0.082$\pm$0.028 &  0.036$\pm$0.014 &  0.132$\pm$0.06 \\
Balanced LR &  \underline{0.215$\pm$0.075} &  0.086$\pm$0.017 &  0.062$\pm$0.022 &  0.059$\pm$0.019 &  0.161$\pm$0.088 \\
Zafar  &  \underline{0.174$\pm$0.046} & 0.118$\pm$0.03 &  0.061$\pm$0.038 &  0.033$\pm$0.013 &  0.145$\pm$0.044 \\
Feldman &  \underline{0.226$\pm$0.124} &  0.107$\pm$0.028 &  0.083$\pm$0.018 &  0.027$\pm$0.008 &  0.202$\pm$0.115 \\

Kamishima  &  \underline{0.283$\pm$0.126} &  0.103$\pm$0.027 &  0.081$\pm$0.024 &  0.031$\pm$0.008 & 0.252$\pm$0.12 \\

MMPF LR &  \textbf{0.109$\pm$0.033} &  0.077$\pm$0.019 &  0.085$\pm$0.02 &  0.03$\pm$0.008 &  \textbf{0.088$\pm$0.031} \\
MMPF LR P &  \underline{0.22$\pm$0.078} &  0.118$\pm$0.063 &  0.059$\pm$0.021 &  0.038$\pm$0.005 &  0.184$\pm$0.079 \\
MMPF &  \underline{0.202$\pm$0.051} & 0.076$\pm$0.02 & 0.107$\pm$0.03 &  0.043$\pm$0.013 &  0.166$\pm$0.046 \\
MMPF P & \underline{0.17$\pm$0.078} &  0.087$\pm$0.024 & 0.085$\pm$0.03 &  0.038$\pm$0.017 &  0.151$\pm$0.067 \\

\bottomrule
\end{tabular}
\end{table*}

\begin{table} [ht]
\centering
\caption{Adult gender dataset. We underline the worst group metric per method, and bold the one with the best minimax performance. Smallest disparity is also bolded. Standard deviations are computed across 5 splits.}
\label{table:AdultGenderFullMetrics}
\scriptsize

Adult gender Acc\\
\begin{tabular}{llll}
\toprule
{} &  \makecell{Female} & \makecell{Male} & Disparity \\
\midrule 
Ratio  & 32.1\% &  67.9\% &  35.9\%\\
Naive LR  &  92.3$\pm$0.4 &  \underline{80.5$\pm$0.4} &  11.9$\pm$0.7 \\
Balanced LR  &  92.3$\pm$0.3 & \underline{80.3$\pm$0.7} &  12.0$\pm$0.7 \\

Zafar &  92.5$\pm$0.3 &  \underline{80.9$\pm$0.3} &  11.6$\pm$0.4 \\
Feldman &  92.3$\pm$0.3 &  \underline{80.7$\pm$0.2} &  11.6$\pm$0.1 \\
Kamishima  &  92.6$\pm$0.4 &  \underline{80.9$\pm$0.4} &  11.7$\pm$0.7 \\

MMPF LR & 91.9$\pm$0.4 &  \underline{81.0$\pm$0.4} &  10.9$\pm$0.7\\
MMPF &  92.1$\pm$0.3 &  \underline{81.3$\pm$0.3} &  10.8$\pm$0.5 \\
MMPF LR P & 92.0$\pm$0.4 &  \underline{81.0$\pm$0.5} &  11.0$\pm$0.6 \\
MMPF P &  91.7$\pm$0.3 &  \textbf{81.5$\pm$0.5} &  \textbf{10.1$\pm$0.5} \\
\midrule
Feldman+H &  \underline{72.3$\pm$2.5\%} &  76.5$\pm$2.7\% &  4.1$\pm$0.8\% \\
Kamishima+H &  \underline{73.5$\pm$2.0\%} &  77.8$\pm$1.2\% &  4.3$\pm$1.4\% \\
Zafar+H &  \underline{73.3$\pm$2.7\%} &  77.3$\pm$2.5\% &  4.0$\pm$1.1\%  \\
Naive LR+H &  \underline{74.2$\pm$2.7\%} &  78.6$\pm$2.1\% &  4.4$\pm$0.7\% \\
Balanced LR+H &  \underline{73.8$\pm$3.1\%} &  77.7$\pm$2.9\% &  3.9$\pm$0.9\% \\
MMPF LR+H &  \textbf{76.0$\pm$2.2\%} &  79.8$\pm$1.6\% &  \textbf{3.8$\pm$1.3\%} \\

\bottomrule
\end{tabular}

Adult gender CE\\
\begin{tabular}{llll}
\toprule
{} &  \makecell{Female} & \makecell{Male} & Disparity \\
\midrule 
Ratio  & 32.1\% &  67.9\% &  35.9\% \\
Naive LR  &  .204$\pm$.009 &  \underline{.411$\pm$.006} &  .207$\pm$.007\\
Balanced LR  &  .204$\pm$.011 & \underline{.416$\pm$.011} &  .211$\pm$.005 \\

Zafar &  .202$\pm$.018 &  \underline{.398$\pm$.006} &  .195$\pm$.023 \\
Feldman &  .201$\pm$.004 &  \underline{.403$\pm$.004} &  .203$\pm$.006 \\
Kamishima  &  .189$\pm$.006 &  \textbf{.395$\pm$.004} &  .206$\pm$.007 \\

MMPF LR &  .204$\pm$.008 &  \underline{.395$\pm$.006} & .19$\pm$.011 \\
MMPF & .21$\pm$.019 &  \underline{.403$\pm$.025} &  .193$\pm$.013 \\
MMPF LR P & .208$\pm$.008 &  \underline{.395$\pm$.005} & .187$\pm$.01 \\
MMPF P &  .227$\pm$.019 &  \underline{.403$\pm$.023} &  \textbf{.176$\pm$.014} \\
\bottomrule
\end{tabular}

Adult gender BS\\
\begin{tabular}{llll}
\toprule
{} & \makecell{Female} & \makecell{Male} & Disparity \\
\midrule 
Ratio  & 32.1\% &  67.9\% &  35.9\% \\
Naive LR  &  .116$\pm$.004 &  \underline{.268$\pm$.004}  &  .152$\pm$.005  \\
Balanced LR  &  .117$\pm$.005  & \underline{.272$\pm$.007}  &  .155$\pm$.005  \\

Zafar & .11$\pm$.004 &  \underline{.258$\pm$.003} &  .147$\pm$.005 \\
Feldman &  .115$\pm$.003 &  \underline{.263$\pm$.002} &  .148$\pm$.003 \\
Kamishima  & .11$\pm$.005 &  \underline{.258$\pm$.003} &  .147$\pm$.005 \\

MMPF LR & .117$\pm$.006 &  \underline{.257$\pm$.004} & .14$\pm$.007 \\
MMPF &  .117$\pm$.004 &  \textbf{.255$\pm$.004} &  .138$\pm$.007 \\
MMPF LR P &  .119$\pm$.005 &  \underline{.258$\pm$.004} &  .138$\pm$.007 \\
MMPF P &  .127$\pm$.003 & \underline{.256$\pm$.005} &  \textbf{.129$\pm$.003} \\
\bottomrule
\end{tabular}

Adult gender ECE\\
\begin{tabular}{llll}
\toprule
{} &  \makecell{Female} & \makecell{Male} & Disparity \\
\midrule 
Ratio  & 32.1\% &  67.9\% &  35.9\% \\
Naive LR  &  \underline{.026$\pm$.008} &  .013$\pm$.004 &  .013$\pm$.004 \\
Balanced LR  & \underline{.023$\pm$.007} & .014$\pm$.005 &  .01$\pm$.006 \\

Zafar & \textbf{.01$\pm$.002} & .01$\pm$.003 &  \textbf{.003$\pm$.001} \\
Feldman &  \underline{.026$\pm$.003} & .01$\pm$.005 &  .016$\pm$.006 \\
Kamishima  &  \underline{.012$\pm$.003} &  .012$\pm$.002 &  .003$\pm$.003 \\

MMPF LR &  \underline{.032$\pm$.005} &  .011$\pm$.003 &  .021$\pm$.005 \\
MMPF &  .009$\pm$.002 &  \underline{.015$\pm$.003} &  .006$\pm$.004 \\
MMPF LR P &  \underline{.028$\pm$.002} &  .009$\pm$.002 &  .019$\pm$.002 \\
MMPF P & \underline{.02$\pm$.006} &  .015$\pm$.001 &  .006$\pm$.005 \\

\bottomrule
\end{tabular}

Adult gender MCE\\
\begin{tabular}{llll}
\toprule
{} &  \makecell{Female} & \makecell{Male} & Disparity \\
\midrule 
Ratio  & 32.1\% &  67.9\% &  35.9\% \\
Naive LR  &  \underline{.064$\pm$.012} &  .027$\pm$.013 &  .037$\pm$.019\\
Balanced LR  & \underline{.065$\pm$.034} &  .031$\pm$.009 &  .034$\pm$.027 \\

Zafar &  \underline{.058$\pm$.013} & .032$\pm$.01 &  .027$\pm$.008 \\
Feldman &  \underline{.071$\pm$.017} &  .024$\pm$.013 &  .047$\pm$.012 \\
Kamishima  &  \underline{.073$\pm$.008} &  .031$\pm$.009 &  .042$\pm$.002 \\

MMPF LR &  \underline{.072$\pm$.017} &  .033$\pm$.006 &  .039$\pm$.021 \\
MMPF &  \textbf{.057$\pm$.021} &  .031$\pm$.003 &  \textbf{.026$\pm$.019} \\
MMPF LR P &  \underline{.064$\pm$.004} & .03$\pm$.004 &  .034$\pm$.004 \\
MMPF P &  \underline{.085$\pm$.023} & .047$\pm$.01 &  .039$\pm$.024 \\

\bottomrule
\end{tabular}

\end{table}


\begin{table} [ht]
\centering
\caption{German dataset. We underline the worst group metric per method, and bold the one with the best minimax performance. Smallest disparity is also bolded. Standard deviations are computed across 5 splits.}
\label{table:GermanGenderFullMetrics}
\scriptsize

German Acc\\
\begin{tabular}{llll}
\toprule
{} &  \makecell{Female} & \makecell{Male} & Disparity \\
\midrule 
Ratio  & 29.5\% &  70.5\% & 41.0\% \\
Naive LR  &  \underline{70.7$\pm$7.3} &  71.2$\pm$4.5 &  8.8$\pm$4.7 \\
Balanced LR  &  71.6$\pm$5.9 &  \underline{70.9$\pm$4.1} &  5.8$\pm$3.6 \\

Zafar  & 73.0$\pm$5.6 &  \underline{71.0$\pm$3.5} &  5.8$\pm$3.5 \\

Feldman  &  73.5$\pm$8.6 &  \textbf{71.9$\pm$4.3} &  7.9$\pm$4.4 \\
Kamishima &  \underline{68.8$\pm$6.8} &  72.7$\pm$2.6 &  6.0$\pm$4.4 \\

MMPF LR & 72.5$\pm$5.5 &  \underline{71.6$\pm$2.8} &  5.0$\pm$2.6 \\
MMPF LR P&  \underline{70.7$\pm$4.5} &  71.5$\pm$3.6 &  \textbf{4.4$\pm$0.5} \\
\midrule
Naive LR+H &  \underline{57.5$\pm$1.7} &  57.8$\pm$1.8 &  5.7$\pm$3.6 \\
Balanced LR+H &  \underline{60.5$\pm$4.2} &  60.9$\pm$4.5 &  4.6$\pm$3.3 \\
Feldman+H &  \underline{61.6$\pm$4.7} &  62.2$\pm$5.0 &  7.1$\pm$3.9 \\
Kamishima+H &  \underline{61.7$\pm$4.0} &  61.3$\pm$4.2 &  4.5$\pm$2.2 \\
Zafar+H &  \underline{59.8$\pm$4.0} &  60.5$\pm$4.9 &  6.6$\pm$4.9 \\
MMPF LR+H &  \textbf{65.7$\pm$4.7} &  65.9$\pm$4.7 &  \textbf{3.6$\pm$1.7} \\

\bottomrule
\end{tabular}

German CE\\
\begin{tabular}{llll}
\toprule
{} &  \makecell{Female} & \makecell{Male} & Disparity \\
\midrule 
Ratio  &  29.5\% &  70.5\% &  41.0\% \\
Naive LR  &  \underline{.607$\pm$.1} &  .559$\pm$.069 &  .127$\pm$.064 \\
Balanced LR  &  \underline{.594$\pm$.082} &  .568$\pm$.068 & .096$\pm$.05 \\
Zafar  &  .567$\pm$.09 &  \underline{.735$\pm$.205} &  .273$\pm$.151 \\
Feldman  & \underline{.564$\pm$.096} &  .551$\pm$.063 &  .091$\pm$.068 \\
Kamishima &  \underline{.62$\pm$.064} &  .545$\pm$.062 &  .075$\pm$.067 \\

MMPF LR&  \underline{.565$\pm$.04} &  .544$\pm$.046 &  \textbf{.048$\pm$.041} \\
MMPF LR P & \textbf{.563$\pm$.043} &  .537$\pm$.051 &  .057$\pm$.034 \\

\bottomrule
\end{tabular}

German BS\\
\begin{tabular}{llll}
\toprule
{} &  \makecell{Female} & \makecell{Male} & Disparity \\
\midrule 
Ratio  &  29.5\% &  70.5\% &  41.0\% \\
Naive LR  &  \underline{.404$\pm$.077} & .379$\pm$.05 &  .094$\pm$.043 \\
Balanced LR  &  \underline{.393$\pm$.069} & .38$\pm$.043 &  .071$\pm$.036 \\

Zafar  &  .379$\pm$.07 &  \underline{.383$\pm$.052} & .072$\pm$.05 \\
Feldman  &  \textbf{.375$\pm$.079} &  .371$\pm$.045 &  .07$\pm$.05 \\
Kamishima &  \underline{.413$\pm$.051} &  .368$\pm$.044 &  .047$\pm$.048 \\

MMPF LR & \underline{.379$\pm$.038} &  .368$\pm$.039 & \textbf{.044$\pm$.03} \\
MMPF LR P & \underline{.381$\pm$.039} &  .363$\pm$.041 &  .051$\pm$.025 \\

\bottomrule
\end{tabular}

German ECE\\
\begin{tabular}{llll}
\toprule
{} &  \makecell{Female} & \makecell{Male} & Disparity \\
\midrule 
Ratio  & 29.5\% &  70.5\% &  41.0\% \\
Naive LR  &  \underline{.136$\pm$.047} &  .099$\pm$.029 &  .051$\pm$.036 \\
Balanced LR  & \underline{.136$\pm$.032} &  .107$\pm$.032 &  \textbf{.038$\pm$.025} \\

Zafar  &  \underline{.117$\pm$.039} &  .097$\pm$.038 &  .047$\pm$.022 \\
Feldman  & \underline{.129$\pm$.045} & .091$\pm$.04 &  .052$\pm$.043 \\
Kamishima & \underline{.125$\pm$.067} &  .082$\pm$.029 & .052$\pm$.04 \\

MMPF LR & \underline{.096$\pm$.036} &  .046$\pm$.019 &  .056$\pm$.046 \\
MMPF LR P & \textbf{.088$\pm$.039} &  .049$\pm$.008 &  .047$\pm$.034 \\

\bottomrule
\end{tabular}

German MCE\\
\begin{tabular}{llll}
\toprule
{} &  \makecell{Female} & \makecell{Male} & Disparity \\
\midrule 
Ratio  & 29.5\% &  70.5\% &  41.0\% \\
Naive LR  & \underline{.308$\pm$.11} & .23$\pm$.023 &  .095$\pm$.091 \\
Balanced LR  & \underline{.322$\pm$.127} &  .216$\pm$.088 &  .106$\pm$.046 \\

Zafar  &  \underline{.255$\pm$.135} & .206$\pm$.08 &  .138$\pm$.072 \\
Feldman  &  \underline{.285$\pm$.096} &  .166$\pm$.048 &  .137$\pm$.083 \\
Kamishima & \underline{.212$\pm$.084} &  .157$\pm$.052 &  \textbf{.062$\pm$.062} \\

MMPF LR & \underline{.18$\pm$.067} & .11$\pm$.046 &  .093$\pm$.075 \\
MMPF LR P &  \textbf{.172$\pm$.073 }& .11$\pm$.029 &  .106$\pm$.046 \\

\bottomrule
\end{tabular}

\end{table}

\end{document}